\documentclass{article}
\usepackage{multicol}

\usepackage[preprint]{neurips_2023}
\usepackage[table,xcdraw]{xcolor}
\usepackage{makecell}
\usepackage[utf8]{inputenc} % allow utf-8 input
\usepackage[T1]{fontenc}    % use 8-bit T1 fonts
\usepackage{hyperref}       % hyperlinks
\usepackage{url}            % simple URL typesetting
\usepackage{booktabs}       % professional-quality tables
\usepackage{amsfonts}       % blackboard math symbols
\usepackage{nicefrac}       % compact symbols for 1/2, etc.
\usepackage{microtype}      % microtypography
\usepackage{xcolor}         % colors
\usepackage{amsmath}
\usepackage{amssymb}
\usepackage{mathtools}
\usepackage{amsthm}
\usepackage{amsfonts}       % blackboard math symbols
\usepackage{amsmath}
\usepackage{amssymb}
\usepackage{amsfonts}
\usepackage{nicefrac}       % compact symbols for 1/2, etc.
\usepackage{microtype}      % microtypography
\usepackage{xcolor}         % colors
\usepackage{graphicx}
\usepackage{enumitem}
\usepackage{multirow}
\usepackage{tikz}
\usepackage{tcolorbox}
\usepackage{wrapfig}
\newtcbox{\alertinline}[1][red]
  {on line, arc = 0pt, outer arc = 0pt,
    colback = #1!20!white, colframe = #1!50!black,
    boxsep = 0pt, left = 1pt, right = 1pt, top = 2pt, bottom = 2pt,
    boxrule = 0pt, bottomrule = 1pt, toprule = 1pt}

\usepackage{tikz}
\usepackage{tcolorbox}

\definecolor{myblue}{RGB}{0 114 199}
\definecolor{mylightblue}{RGB}{77 191 241}
\definecolor{darkgray}{HTML}{878787}
\definecolor{myorange}{RGB}{217 83 25}
\newtcolorbox{myorangebox}{colframe = myorange}
\newtcolorbox{mygraybox}{colframe = gray}
\newtcolorbox{mybluebox}{colframe = myblue}
\newcommand{\yplu}[1]{\textcolor{red}{(2prime: #1)}}
\newcommand{\honam}[1]{\textcolor{blue}{(Honam: #1)}}
\definecolor{fhcolor}{rgb}{0.523, 0.235, 0.625}

\theoremstyle{plain}
\newtheorem{theorem}{Theorem}[section]

\newtheorem{lemma}[theorem]{Lemma}

\theoremstyle{definition}
\newtheorem{definition}[theorem]{Definition}
\newtheorem{assumption}[theorem]{Assumption}
\theoremstyle{remark}
\DeclareMathOperator{\trace}{tr}

\DeclareMathOperator*{\argmin}{argmin}

\newtheorem{remark}{Remark}
\newtheorem{example}{Example}

\usepackage{url}

\title{Benign overfitting in Fixed Dimension via Physics-Informed Learning with Smooth Inductive Bias}

\author{%
 Honam Wong\\
Department of Computer Science and Engineering \\
The Hong Kong University of Science and Technology\\
\texttt{hnwongaf@connect.ust.hk}\\
\And
Wendao Wu\\ %\thanks{Use footnote for providing further information about author (webpage, alternative address)---\emph{not} for acknowledging funding agencies.} \\
  School of Mathematical Sciences\\
  Peking University \\
  \texttt{wuwendao@stu.pku.edu.cn} \\
  \And
Fanghui Liu\\ %\thanks{Use footnote for providing further information about author (webpage, alternative address)---\emph{not} for acknowledging funding agencies.} \\
  Department of Computer Science\\
  University of Warwick\\
  \texttt{fanghui.liu@warwick.ac.uk} \\
\And
  Yiping Lu\\ %\thanks{Use footnote for providing further information about author (webpage, alternative address)---\emph{not} for acknowledging funding agencies.} \\
  Industrial Engineering \& Management Sciences\\
  Northwestern University\\
  \texttt{yiping.lu@northwestern.edu} \\
}

\begin{document}

\maketitle
\vspace{-0.4in}
\tableofcontents
\newpage

\begin{abstract}
Recent advances in machine learning have inspired a surge of research into reconstructing specific quantities of interest from measurements that comply with certain physical laws. These efforts focus on inverse problems that are governed by partial differential equations (PDEs). In this work, we develop an asymptotic Sobolev norm learning curve  for kernel ridge(less) regression when addressing (elliptical) linear inverse problems. Our results show that the PDE operators in the inverse problem can stabilize the variance and even behave benign overfitting for fixed-dimensional problems, exhibiting different behaviors from regression problems. Besides, our investigation also demonstrates the impact of various inductive biases introduced by minimizing different Sobolev norms as a form of implicit regularization. For the regularized least squares estimator, we find that all considered inductive biases can achieve the optimal convergence rate, provided the regularization parameter is appropriately chosen. The convergence rate is actually independent to the choice of (smooth enough) inductive bias for both ridge and ridgeless regression. Surprisingly, our smoothness requirement recovered the condition found in Bayesian setting and extend the conclusion to the minimum norm interpolation estimators.
\if 0
Recent advances in machine learning have inspired a surge of research into reconstructing specific quantities of interest from measurements that comply with certain physical laws. These efforts focus on inverse problems that are governed by partial differential equations (PDEs). In this work, we develop an asymptotic Sobolev norm learning curve for kernel ridge(less) regression when addressing (elliptical) linear inverse problems. Our results show that the PDE operators in the inverse problem can stabilize the variance and even behave benign overfitting for fixed-dimensional problems, exhibiting different behavior from regression problems. Furthermore, our investigation also demonstrates the impact of various inductive biases introduced by minimizing different Sobolev norms as a form of implicit regularization. The final convergence rate is independent to the choice of (smooth enough) inductive bias for both ridge and ridgeless regression. For the regularized least squares estimator, we find that all considered inductive biases can achieve the optimal convergence rate, provided the regularization parameter is appropriately chosen. Surprisingly, our smoothness requirement recovered the condition found in Bayesian setting and extend the conclusion to the minimum norm interpolation estimators.
\fi
\end{abstract}

\section{Introduction}
Inverse problems are widespread across science, medicine, and engineering, with research in this field yielding significant real-world impacts in medical
image reconstruction \cite{ronneberger2015u}, inverse scattering \cite{khoo2017solving} and 3D reconstruction \cite{sitzmann2020implicit}. The recent swift advancements in learning-based image generation present exciting possibilities in the field of inverse problems \cite{raissi2019physics,lu2021machine,nickl2020convergence}. In this paper, we study the statistical limit of machine learning methods for solving (elliptical) inverse problems. To be specific, we consider the problem of reconstructing a function from random sampled observations with statistical noise in measurements. When the observations are the direct observations of the function, the problem is a classical non-parametric function estimation \cite{de2005learning,tsybakov2004optimal}. Nevertheless, the observations may also come from certain physical laws described by a partial differential equation (PDE) \cite{stuart2010inverse,benning2018modern}. Since the linear inverse problem are always ill-posed, where a small noise in the observation can result in much larger errors in the answers. Further analysis \cite{knapik2011bayesian,nickl2020convergence,lu2021machine,lu2022sobolev,nickl2023bayesian,randrianarisoa2023variational} of how the ill-posed inverse problem would change the information-theoretical analysis is always needed. 

Formally, following the setting in \cite{lu2022sobolev}, we aim to reconstruct a function $f^*$ based on independently sampled data set $D=\{(x_i,y_i)\}_{i=1}^n$ from an unknown distribution $P$ on $\mathcal{X}\times \mathcal{Y}$, where $y_i$ is the noisy measurement of $f^*$ through a measurement procedure $\mathcal{A}$. For simplicity, we assume $\mathcal{A}$ is self-adjoint (elliptic) linear operator in this paper. The conditional mean function $u^*(x)=\mathbb{E}_P(Y|X=x)$ is the ground truth function for observation of $f^*$ through the measurement procedure $\mathcal{A}$, \emph{i.e.} $u^*=\mathcal{A}f^*$. Since the inverse problem $\mathcal{A}^{-1}$ is always ill-posed, thus directly inverse can be dangerous. Recently, over-parameterized machine learning models \cite{raissi2019physics,han2018solving,sirignano2018dgm} and interpolated estimators \cite{yang2021inference,chen2021solving} become success solutions to linear inverse problems and they can generalize well under noisy observation, \emph{i.e.}, benign overfitting \cite{bartlett2020benign,frei2022benign,cao2022benign,zhu2023benign}. Despite the success and popularity of adopting learning based approach for solving (linear) inverse problem, the following question still remains poorly answered:

\begin{center}
    \emph{Do over-parameterized models or interpolating estimators generalize effectively (i.e., exhibit benign overfitting) when addressing inverse problems? What are the conditions inherent to inverse problems that facilitate or impede benign overfitting?}
\end{center}
We further study the following question in our paper
\begin{center}
    \emph{How should one select the appropriate inductive bias for solving inverse problems? Additionally, how does inductive bias contribute to resolving linear inverse problems?}
\end{center}
{We provide affirmative answers to both questions. We discovered that the PDE operator in the inverse problem would stabilize the variance and leads to beginning over-fitting even in the fixed dimension setting. We also observed that inductive bias needs focus enough on the low frequency component to achieve best possible convergence rate}. To be specific, we consider a general class of norm, known as Reproducing Kernel Sobolev space (RKSS) \cite{fischer2020sobolev}, to quantize inductive bias in a certain space. The RKSS is a spectral transformed space with polynomial
transformation \cite{fischer2020sobolev,lu2022sobolev,zhai2024spectrally,zhang2023optimality,li2024asymptotic} which is a spectral characterization of Sobolev spaces \cite{fischer2020sobolev,adams2003sobolev}, which is widely used in characterizing the stability of (elliptic) inverse problems. Mathematically, given a non-negative real number $\beta > 0$, the $\beta$-power Sobolev space $\mathcal{H}^{\beta}$ associated with a kernel $K$ is defined as
\begin{equation*}
  \mathcal{H}^{\beta} := \left\{ \sum_{i \geq 1} a_i \lambda_i^{\beta/2} \psi_i: \sum_{i \geq 1} a_i^2 \leq \infty \right\} \subset L^2(\rho_\mathcal{X})\,,  \quad K(s,t)=\sum_{i=1}^\infty \lambda_i \psi_i(s)\psi_i(t)\,,
\end{equation*}
where $\psi_i$ is the eigenfunction of the kernel $K$ defined by the Mercer's spectral decomposition \cite{minh2006mercer,steinwart2008support}, where $\rho_\mathcal{X}$ is the marginal distribution of $P$ respect to $\mathcal{X}$  and $\mathcal{H}^{\beta}$ is equipped with the $\beta$-power norm via $\| \sum_{i \geq 1}{a_i \lambda_i^{\beta/2} \psi_i}\|_{\beta} := (\sum_{i\geq 1} a_i^2)^{1/2}$. 
Here $\beta \in [0,1]$ characterizes how much we are biased towards low frequency functions, see more details in Section~\ref{Appendix:notations}.
%\fh{and also mention a bit about the motivation of why we choose RKSS.}
%\yplu{$\beta$ charaterized how much we are biased towards low frequency functions}
Regarding the learned model, we consider both regularized least square and minimum norm interpolation in this paper for solving the abstract inverse problem:

\begin{multicols}{2}
\begin{myorangebox}
\textbf{Regularized Least Square}\cite{knapik2011bayesian,nickl2020convergence,lu2022sobolev}:
\begin{equation}
 \begin{aligned}
 \hat f_\gamma:= &  \argmin_f \frac{1}{n} \sum_{i=1}^n\|{\color{blue}\mathcal{A}}f(x_i)-y_i\|^2\\&\qquad\qquad\qquad\qquad+\gamma_n\|f\|_{{\color{blue}H^\beta}}
    \end{aligned}
    \end{equation}
\end{myorangebox}

\begin{myorangebox}
\textbf{Minimum Norm Interpolation}\cite{chen2021solving,yang2021inference,sun2023manifold}:
\begin{equation}
 \begin{aligned}
\hat f:=   \argmin_f & \|f\|_{{\color{blue} H^\beta}}\\
\text{s.t.}&{\color{blue}\mathcal{A}}f(x_i)=y_i
    \end{aligned}
    \end{equation}
\vspace{0.18em}
\end{myorangebox}

\end{multicols}
\vspace{-0.1in}
In this paper, we have developed the generalization guarantees of Sobolev norm learning for both (Sobolev norm)-regularized least squares and minimum (Sobolev) norm interpolators in the context of elliptical linear inverse problems. 
Based on the derived results, we investigate the effects of various inductive biases (\emph{i.e.} $\beta$) that arise when minimizing different Sobolev norms. Minimizing these norms imposes an inductive bias from the machine learning algorithms.  In the case of the regularized least squares estimator, we demonstrate that all the smooth enough inductive biases are capable of achieving the optimal convergence rate, assuming the regularization parameter is selected correctly. Additionally, the choice of inductive bias does not influence the convergence rate for interpolators, e.g., the overparameterized/ridgeless estimators. This suggests that with a perfect spectrally transformed kernel, the convergent behavior of regression will not change. The only difference may occur when using empirical data to estimate the kernel, \emph{i.e.} under the semi-supervised learning setting \cite{zhou2008high,zhai2024spectrally}.

\subsection{Related Works}
\paragraph{Physics-informed Machine Learning:} Partial differential equations (PDEs) are widely used in many disciplines of science and engineering and play a prominent role in modeling and forecasting the dynamics of multiphysics and multiscale systems. The recent machine learning revolution transforming the computational sciences by enabling flexible, universal approximations for high-dimensional functions and functionals. This inspires researcher to tackle traditionally intractable high-dimensional partial differential equations via machine learning methods \cite{long2018pde,long2019pde,raissi2019physics,han2018solving,sirignano2018dgm,khoo2017solving,liu2020multi}. Theoretical convergence results for deep learning based PDE solvers has also received considerable attention recently. Specifically,  \cite{lu2021priori,grohs2020deep,marwah2021parametric,wojtowytsch2020some,xu2020finite,shin2020error,bai2021physics} investigated the regularity of PDEs approximated by a neural network and \cite{lu2021priori,luo2020two,duan2021convergence,jiao2021convergence,jiao2021error,jin2022minimax,doumeche2024physics} further provided generalization analyses. \cite{nickl2020convergence,lu2021machine,hutter2019minimax,manole2021plugin,huang2021solving,wang2023expert} provided information theoretical optimal lower and upper bounds for solving PDEs from random samples. However, previous analyses have concentrated on under-parameterized models, which do not accurately characterize large neural networks \cite{raissi2019physics,weinan2018deep} and interpolating estimators \cite{yang2021inference,chen2021solving}. Our analysis addresses this gap in theoretical research and provide the first unified upper bound from regularized least square estimators to beginning overfitted minimum norm interpolators under fixed dimensions.

\paragraph{Learning with kernel:} Supervised least square regression in RKHS has a long history and its generalization ability and mini-max optimality has been thoroughly studied \cite{caponnetto2007optimal,smale2007learning,de2005learning,rosasco2010learning,mendelson2010regularization}. The convergence of least square regression in Sobolev norm has been discussed recently in \cite{fischer2020sobolev,liu2020estimation,zhang2023optimality}. Recently, training neural networks with stochastic gradient descent in certain regimes has been found to be equivalent to kernel regression \cite{daniely2017sgd,lee2017deep,jacot2018neural}. Recently  \cite{lu2022sobolev,randrianarisoa2023variational,doumeche2024physics,randrianarisoa2023variational} use kernel based analysis to theoretically understand physics-informed  machine learning. Our work is different from this line of researches in two perspective. Firstly, we considered the family of spectrally transformed kernels \cite{zhai2024spectrally} to study how different smoothness inductive bias would affect the efficiency  of machine learning estimators. Secondly, We aim to analysis the statistical behavior of interpolators, e.g., overparameterized estimators. Thus we build the first rigorous upper bound for the excess risk of the min-norm interpolator in the fixed dimensional setting from benign overfitting to tempered overfitting in physics-informed machine learning.

%\yplu{our work is different in two perspective,  and interpolated estimator }  %Gradient descent training of neural network in the kernel regime has been found optimal for a wide class of non-parametric functions with both early stopping regularization and ridge regression \cite{nitanda2020optimal,hu2020regularization}.

\begin{table}[]
{\footnotesize
\centering
\begin{tabular}{c|>{\columncolor[HTML]{ECF4FF}}c |>{\columncolor[HTML]{ECF4FF}}c |>{\columncolor[HTML]{ECF4FF}}c |c | >{\columncolor[HTML]{ECF4FF}}c}
\hline\hline
\textbf{Param.} &  {$\lambda > 1$ }&  $r\in (0,1]$                    & $p<0$                          & $\mathcal{H}_{\beta}$  & $\mathcal{H}_{\beta'}$                   \\ 
%don't make scalar/space in bold
\hline\hline
          &\makecell[t]{Eigendecay of\\ Kernel Matrix\\\emph{\scriptsize(Capacity Condition)}}& \makecell[t]{Smoothness of the\\ ground truth solution\\\emph{\scriptsize(Source Condition)}} & \makecell[t]{Order of the\\ Inverse Problem\\\emph{\scriptsize(Capacity Condition on $\mathcal{A}$)}} & \makecell[t]{norm used for\\ regularization\\$\beta \in [0,1]$ } & \makecell[t]{norm used for\\ evaluation\\$\beta' \in [0, \beta]$} \\ \hline
\end{tabular}
}
\caption{The parameters $\lambda$, $r$, $p$, ${\mathcal{H}_{\beta}}$ and ${\mathcal{H}_{\beta'}}$ are used to describe our problem. The blue-shaded blocks, $\lambda,r,p$ and $\beta'$, represent the parameters that are employed to characterize the inverse problem task, which should influence the minimax optimal risk.}
\vspace{-0.3in}
\label{label:parameter}
\end{table}

\subsection{Contribution and Technical Challenges}
\begin{itemize}[itemindent=0.5cm,leftmargin=0.5cm]
    \item Instead of considering regularizing RKHS norm \cite{lu2022sobolev,randrianarisoa2023variational} or interpolation while minimizing RKHS norm \cite{barzilai2023generalization,cheng2024characterizing}, we consider (implicit) regularization using a Kernel Sobolev norm \cite{fischer2020sobolev} or spectrally transformed kernel \cite{zhai2024spectrally}. Under such setting, we aim to study how different inductive bias will change the statistical properties of estimators. To this end, we derived the closed form solution for spectrally transformed kernel \cite{zhai2024spectrally} estimators for linear inverse problem via a generalized Representer theorem for inverse problem \cite{unser2021unifying} and extend previous non-asymptotic benigning overfitting bounds \cite{bartlett2020benign,cheng2024characterizing,barzilai2023generalization} to operator and inverse problem setting.
    \item Our non-asymptotic bound can cover both regularized and minimum norm interpolation estimators for solving (linear) inverse problems. For the regularized case, we recovered the minmax optimal rate for linear inverse problem presented in \cite{lu2022sobolev}. We provide the first rigorous upper bound for the excess risk of the min-norm interpolator in the fixed dimensional setting from benign overiftting to tempered overifting, and catastrophic overiftting in Physics-informed machine learning. \emph{\textbf{Our results show that the PDE operators in inverse problems possess the capability to stabilize variance and remarkably benign overfitting, even for problems with a fixed number of dimensions, a trait that distinguishes them from regression problems.}}
    
    \item  Our target is to examine the effects of various inductive biases that arise from minimizing different Sobolev norms, which serve as a form of inductive bias imposed by the machine learning algorithms. For regularized regression in fixed dimension, traditional research \cite{fischer2020sobolev,lu2022sobolev,guastavino2020convergence} show that proper regularized least square regression can achieve minimax optimal excess risk with \emph{\textbf{smooth enough}} implicit regularization of arbitrary spectral decay. Our bound concrete the similar phenomenon happens in the overparamterization/interpolate estimators where \textbf{\emph{the choice of smooth enough inductive bias also does not affect convergence speed}}. {The smoothness requirement of implicit bias $\beta$ should satisfies $\lambda\beta\ge\frac{\lambda r}{2}-p$, where $r$ is the smoothness of the target function (characterized by the source condition), $\lambda$ is the spectral decay of the kernel operator and $p$ is the order the elliptical inverse problem, see Table~\ref{label:parameter} for details.} Under the function estimation setting, the selection matches the empirical understanding in semi-supervised learning \cite{zhou2008high,zhou2011semi,smola2003kernels,chapelle2002cluster,dong2020cure,zhai2024spectrally} and \emph{\textbf{theoretically surprisingly matches the smoothness threshold deteremined for the Bayesian Inverse problems}} \cite{knapik2011bayesian,szabo2013empirical}.
\end{itemize}

%{\bf Technical challenges:}
%\fh{We need a separate paragraph to talk about the technique challenge when compared to Ohad's paper. The reviewer will definitely ask for this.}

%\honam{Things we can talk about: new definition of concentration coefficient, $\alpha_k$ and $\beta_k$, the assumption of $2p + \beta \lambda > 0$  to make $\tilde{\Sigma}$ has decaying eigenvalue (we possibly require $2p + \beta' \lambda > 0$ for bias session \ref{appendix:ub_bias}) (referring to \ref{assumption:well_behaveness_feature} \ref{def:concerntration}), details in the lemma i.e. how to ensure the bound is still tight (?) @Yiping}
%\fh{I think we can talk about: 1) from finite-dimension to infinite-dimension because of operator \honam{This one is aleady overome by Ohad's paper so not necessary i guess}\fh{you can talk about something about the coupling relationshp between $\mathcal{A}$ and $\Sigma$}; 2) new definition of concetration coefficient, well-behave features, and no need to define effective rank; 3) talk about some significance, e.g., benign overfitting in our case under fixed dimension (no strong data assumptions)}

\section{Preliminaries, Notations, and Assumptions}
\label{Appendix:notations}

In this section, we introduce the necessary notations and preliminaries for reproducing kernel Hilbert space (RKHS), including Mercer's decomposition, the integral operator techniques \cite{smale2007learning,de2005learning,caponnetto2007optimal,fischer2020sobolev,rosasco2010learning} and the relationship between RKHS and the Sobolev space \cite{adams2003sobolev}. 
The required assumptions are also introduced in this section.

We consider a Hilbert space $\mathcal{H}$ with inner product $\left<\cdot,\cdot\right>_{\mathcal{H}}$ is a separable Hilbert space of functions $\mathcal{H}\subset \mathbb{R}^{\mathcal{X}} $. We call this space a Reproducing Kernel Hilbert space if $f(x)=\left<f,K_x\right>_{\mathcal{H}}$ for all $K_x\in\mathcal{H}:t\rightarrow K(x,t),x \in \mathcal{X}$. Now we consider a distribution $\rho$ on $\mathcal{X}\times \mathcal{Y} (\mathcal{Y}\subset \mathbb{R})$ and denote $\rho_X$ as the margin distribution of $\rho$ on $\mathcal{X}$. We further assume $\mathbb{E}[K(x,x)]<\infty$ and $\mathbb{E}[Y^2]<\infty$. We define $g\otimes h=gh^\top$ is an operator from $\mathcal{H}$ to $\mathcal{H}$ defined as
$
g\otimes h:f\rightarrow \left<f,h\right>_{\mathcal{H}}g.
$ The integral operator technique \cite{smale2007learning,caponnetto2007optimal} consider the covariance operator on the Hilbert space $\mathcal{H}$ defined as $\Sigma = \mathbb{E}_{\rho_\mathcal{X}}K_x\otimes K_x$. Then for all $f\in\mathcal{H}$, using the reproducing property, we know that
$$
(\Sigma f)(z)
=\left<K_z,\Sigma f\right>_{\mathcal{H}}=\mathbb{E}[f(X)K(X,z)]=\mathbb{E}[f(X)K_z(X)].
$$

If we consider the mapping $S:\mathcal{H}\rightarrow L_2(\rho_\mathcal{X})$ defined as a parameterization of a vast class of functions in $\mathbb{R}^{\mathcal{X}}$ via $\mathcal{H}$ through the mapping $(Sg)(x)=\left<g,K_x\right>$ %($\Phi(x)=K_x=K(\cdot,x)$).\wendao{$\Phi$ is not defined or mentioned at this point.} 
Its adjoint operator $S^\ast$ then can be defined as $S^\ast:\mathcal{L}_2\rightarrow\mathcal{H}: g\rightarrow \int_\mathcal{X}g(x)K_x \rho_X(dx)$. At the same time $\Sigma=SS^\ast$ is the same as the self-adjoint operator $S^\ast S$.  We further define the empirical operator $\hat{S}_n: \mathcal{H} \to \mathbb{R}^n$ as $\hat{S}_n f := (\langle f, K_{x_1} \rangle, \cdots, \langle f, K_{x_n} \rangle)$ and $\hat{S}^{*}_n: \mathbb{R}^n \to \mathcal{H}$ as $\hat{S}^{*}_n \theta = \sum_{i=1}^{n}{\theta_i K_{x_i}}$, then we know $\hat{S}_n \hat{S}^{*}_{n}: \mathbb{R}^n \to \mathbb{R}^n$ is the Kernel Matrix we denote it as $\hat{K}$, and $\frac{1}{n}\hat{S}^{*}_{n} \hat{S}_n: \mathcal{H} \to \mathcal{H}$ is the empirical covariance operator $\hat{\Sigma}$. %and the self-adjoint operator $\mathcal{L}=SS^\ast:L_2\rightarrow L_2$ can be defined as $(\mathcal{L}f)(z) =(SS^*f)(z)=\int_\mathcal{X} K(x,z)f(x)\rho_{\mathcal{X}}(dx),\forall f\in L_2$.

\begin{comment}
    
\honam{ This is \cite{barzilai2023generalization}'s notation, for reference, may a little modifications:\\
Let \( \mathcal{X} \) be some input space, \( \mu \) an associated measure and 
\( K : \mathcal{X} \times \mathcal{X} \rightarrow \mathbb{R} \) a Mercer kernel, 
meaning that it admits a spectral decomposition of the form
\[
    K(x, x') = \sum_{i=1}^{\infty} \lambda_i \psi_i(x) \psi_i(x'),
\]
where \( \lambda_i \geq 0 \) are the non-negative eigenvalues (not necessarily 
ordered), and the eigenfunctions \( \psi_i \) form an orthonormal basis in 
\( L^2_{\mu}(\mathcal{X}) \). Let \( \rho \in \mathbb{N} \cup \{\infty\} \) denote 
the number of non-zero eigenvalues, and w.l.o.g let \( \phi(x) \) := 
\( (\sqrt{\lambda_i} \psi_i(x))_{i=1}^{\rho} \) be the non-zero features 
(with \( \lambda_i > 0 \)) and \( \psi(x) \) := \( (\psi_i(x))_{i=1}^{\rho} \). 
Since \( \mathbb{E}_x [\psi(x) \psi(x)^T] = I \), the features admit a diagonal 
and invertible (uncentered) covariance operator given by 
\( \Sigma := \mathbb{E}_x [\phi(x) \phi(x)^T] = \text{diag}(\lambda_1, \lambda_2, \ldots) \). 
The features are related to the eigenfunctions by \( \phi(x) = \Sigma^{1/2}\psi(x) \), 
and to the kernel by \( K(x, x') = \langle \phi(x), \phi(x') \rangle \) 
where the dot product is the standard one.}
\end{comment}

Next we consider the eigen-decomposition of the integral operator $\mathcal{L}$ to construct the feature map mapping via Mercer's Theorem. There exists an orthogonal basis $\{\psi_i\}$ of $\mathcal{L}_2(\rho_{\mathcal{X}})$ consisting of eigenfunctions of kernel integral operator $\mathcal{L}$. The kernel function have the following representation %\fh{polish this part to coincide with the introduction}
\begin{equation}\label{eq:kernel}
  K(s,t)=\sum_{i=1}^\infty \lambda_i \psi_i(s)\psi_i(t).  
\end{equation}
where $\psi_i$ are orthogonal basis of $\mathcal{L}_2(\rho_\mathcal{X})$. Then $\psi_i$ is also the eigenvector of the covariance operator $\Sigma$ with eigenvalue $\lambda_i>0$, \emph{i.e.} $\Sigma \psi_i = \lambda_i \psi_i$.

Following the \cite{bartlett2020benign,cheng2024characterizing,barzilai2023generalization,tsigler2023benign}, we conduct the theoretical analysis using spectral decomposition. Thus, in this paper, we define the spectral feature map $\phi: \mathcal{H} \to \mathbb{R}^\infty$ via $\phi f:=(\left<f,\phi_i\right>_\mathcal{H})_{i=1}^\infty$ where $\phi_i=\sqrt{\lambda_i}\psi_i$ which forms an orthogonal basis of the reproducing Kernel Hilbert space. Then $\phi^*: \mathbb{R}^\infty \to \mathcal{H}$ takes $\theta$ to $\sum_{i=1}^{\infty}{\theta_i \phi_i}$. Then $\phi^* \phi = id:\mathcal{H}\rightarrow\mathcal{H}$, $\phi \phi^* = id: \mathbb{R}^{\infty}\rightarrow\mathbb{R}^{\infty}$. $\phi$ is an isometry i.e. for any function $f$ in $\mathcal{H}$ we have $\|f\|_{\mathcal{H}}^2 = \|\phi f\|^2$. Similarly we also define $\psi: \mathcal{H} \to \mathbb{R}^{\infty}$ via $\psi f := (\langle f, \psi_i \rangle_{\mathcal{H}})_{i=1}^{\infty}$, the motivation of defining this is this can simplify our computation in the lemmas, we define $\psi^*: \mathbb{R}^\infty \to \mathcal{H}$ takes $\theta$ to $\sum_{i=1}^{\infty}{\theta_i \psi_i}$.

We then define the operator $\Lambda_{\mathcal{X}}:\mathbb{R}^\infty\rightarrow\mathbb{R}^\infty$ corresponding to $\mathcal{X}$ is the operator such that $\mathcal{X}=\phi^* \Lambda_{\mathcal{X}} \phi$, which implies $\Lambda_{\mathcal{X}\mathcal{Y}}=\Lambda_{\mathcal{X}}\Lambda_{\mathcal{Y}}$. Followed by our notation, we can simplify the relationship between $\phi$ and $\psi$ as $\phi = \Lambda_{\Sigma}^{1/2} \psi$ and $\phi^* = \psi^* \Lambda_{\Sigma}^{1/2}$. %Furthermore, assuming that $\mathcal{A}$ and $\Sigma$ are co-diagonalizable, i.e., they share the same $\phi$, we can express $\mathcal{A} \Sigma$ as $\phi^* \Lambda_{\mathcal{A}} \Lambda_{\Sigma} \phi = \phi^* \Lambda_{\mathcal{A} \Sigma} \phi$, and similarly $\mathcal{A}^2 = \phi^* \Lambda_{\mathcal{A}^2} \phi$.

%  \yplu{We can define an abstract operator $\Lambda_{\mathcal{X}}$ as $\mathcal{X}=\phi^* \Lambda_{\mathcal{X}} \phi$ and then $\Lambda_{\mathcal{X}\mathcal{Y}}=\Lambda_{\mathcal{X}}\Lambda_{\mathcal{Y}}$} \honam{It is matrix instead of operator?}. Then we also have $\phi = \Lambda_{\Sigma}^{1/2} \psi$, and $\phi^* = \psi^* \Lambda_{\Sigma}^{1/2}$.\honam{Remember to add assumption that $\mathcal{A}$ and $\Sigma$ co-diagonalizable i.e. share the same $\phi$}Example: $\mathcal{A} \Sigma = \phi^* \Lambda_{A} \phi \phi^* \Lambda_{\Sigma} \phi  = \phi^*  \Lambda_{\mathcal{A}} \Lambda_{\Sigma} \phi = \phi^* \Lambda_{\mathcal{A} \Sigma} \phi$, $\mathcal{A}^2 = \phi^* \Lambda_{\mathcal{A}^2} \phi$. 

%{ \color{blue} To better understand the notation, we can compare  \cite{barzilai2023generalization}'s notation with here, $\hat{S}_n \phi^*$ corresponds to $\phi(X)$,  where $\phi \hat{S}^{*}_n$ corresponds to $\phi(X)^T$ and $ \hat{S}_n \hat{S}^{*}_n = \hat{S}_n \phi^* \phi \hat{S}^{*}_n$, which corresponds to $\phi(X) \phi(X)^T = \hat{K} \in \mathbb{R}^{n \times n}$, which is the empirical kernel matrix, in later chapter, we will introduce the empirical spectrally transformed matrix $\tilde{K}$ (see Lemma \ref{eq:objective}). } 

%\paragraph{Reproducing Kernel Sobolev Space}
\begin{definition}[Sobolev Norm]
    \label{definition:sobolev_norm}
    For $\beta > 0$, the $\beta$-power Reproducing Kernel Sobolev Space is
    $$\mathcal{H}^{\beta} := \{ \sum_{i \geq 1} a_i \lambda_i^{\beta/2} \psi_i: \sum_{i \geq 1} a_i^2 \leq \infty \} \subset L^2(\rho_\mathcal{X}),$$
    equipped with the $\beta$-power norm via $\| \sum_{i \geq 1}{a_i \lambda_i^{\beta/2} \psi_i}\|_{\beta} := (\sum_{i\geq 1} a_i^2)^{1/2}$.
\end{definition}
As shown in \cite{fischer2020sobolev}, $\mathcal{H}^\beta$ is an interpolation between Reproducing Kernel Hilbert Space and $\mathcal{L}_2$ space. Formally, $\|\mathcal{L}^{\beta/2}f\|_\beta=\|f\|_{L_2}$ where  $\mathcal{L}=SS^\ast$ and $\|f\|_\beta = \|\Sigma^{\frac{1-\beta}{2}}f\|_\mathcal{H}$ for $0\leq \beta \leq 1$. Thus when $\beta=1$, the $\mathcal{H}^\beta$ is the same as Reproducing Kernel Hilbert Space and  when $\beta=0$ the $\mathcal{H}^\beta$ is the same as $\mathcal{L}_2$ space. Reproducing Kernel Sobolev Space is introduced to characterize the misspecification in kernel regression \cite{zhang2023optimality,kanagawa2016convergence,pillaud2018statistical,steinwart2009optimal}. In our paper we use it as spectral charaterization of Sobolev space \cite{adams2003sobolev,wendland2004scattered} which is the most natural function space for PDE analysis.

\begin{assumption}[Assumptions on Kernel and Target Function]
\setlength{\itemsep}{0pt}
\setlength{\parsep}{0pt}
\setlength{\parskip}{0pt}
\label{assumption:kernel} We assume the standard capacity condition on kernel covariance operator with a source condition about the regularity of the target function following \cite{caponnetto2007optimal} and assumption of the inverse problem following \cite{lu2022sobolev}. These conditions are stated explicitly below:
\begin{itemize}[itemindent=0.5cm,leftmargin=0.5cm]
\setlength{\itemsep}{0pt}
\setlength{\parsep}{0pt}
\setlength{\parskip}{0pt}
\item \textbf{(a) Assumptions on boundedness.} 
The kernel feature are bounded almost surely, \emph{i.e.} $|k(x,y)|\le R$ and the observation $y$ is also bounded by $M$ almost surely.
\item \textbf{(b) Capacity condition.} Consider the spectral representation of the kernel covariance operator $\Sigma=\sum_i \lambda_i \psi_i\otimes \psi_i$, we assume polynomial decay of eigenvalues of the covariance matrix $\lambda_i\propto i^{-\lambda}$ for some $\lambda>1$.
%\wendao{$Q\propto \Sigma_i 1/i$ is divergent?}
This assumption satisfies for many useful kernels in the literature such as \cite{minh2006mercer}, neural tangent kernels \cite{bietti2020deep,chen2020deep}.

%\fh{I suggest to use $\beta'$ here as $\beta$ is used for Sobolev norm in the introduction.} \yplu{we can rewrite the true function in $f\in\mathcal{H}_\beta$}
\item \textbf{(c) Source condition.} We also impose an assumption on the smoothness of the true function. There exists $r\in(0,1]$ such that $f^\ast=\mathcal{L}^{r/2}\phi$ for some $\phi\in L^2$. If $f^\ast(x)=\left<\theta_\ast,K_x\right>_{\mathcal{H}}$, the source condition can also be written as $\|\Sigma^{\frac{1-r}{2}}\theta_\ast\|_{\mathcal{H}}<\infty.$ The source condition can be understood as the target function lies in the $r$-power Sobolev space.
\end{itemize}
\end{assumption}
\begin{assumption}[Capacity conditions on $\mathcal{A}$] For theoretical simplicity, following \cite{knapik2011bayesian,cabannes2021overcoming,de2021convergence,lu2022sobolev}, we assume that the self-adjoint operator  $\mathcal{A}$ are co-diagonalizable in the same orthonormal basis $\psi_i$ with $\Sigma$. Thus we can assume $\mathcal{A} = \sum_{i=1}^\infty p_i\psi_i\otimes \psi_i,$, with $p_i\propto i^{-p}$ and $p < 0$, which holds when $\mathcal{A}$ is a differential operator. 
\end{assumption}
\begin{remark} 
    We provide several common examples in physics-informed learning that satisfy the co-diagonalizable assumption. The simplest example is $\mathcal{A} = \text{id}$ which corresponds to the function estimation setting. For solving PDE numerically we oftern take $\mathcal{A} = \Delta^k$, where
    the parameter $p$ in (d) is used to characterize the order of PDE. When the domain is sphere and the data distribution is uniform, the eigenfunctions are spherical harmonics, which are also the eigenfunctions of a wide class of kernels, including dot product kernels and the Neural Tangent Kernel \cite{jacot2018neural}. When the domain is torus, the eigen-functions are Fourier modes, and eigenfunctions of shift-invariant kernel are also Fourier modes by Bochner's theorem \footnote{We consider shift-invariant kernel $K(x,y) = \psi (x - y)$, from Bochner's theorem we have $K(x,y) = \sum_{i=1}^{n} \tilde{\psi} (w) e^{iws} e^{-wt}$. }. 
    We also demonstrate that empirically our findings still hold beyond co-diagonalization assumption  (Appendix~\ref{appendix:exp}).
\end{remark}

%\vspace{-0.0in}

\vspace{-0.05in}

\begin{example}[Schr\"{o}dinger equation on a Hypercube] Consider solving Schr\"{o}dinger equation on a hypercube $-\Delta u + u = f$ on $\mathbb{T}^d=[0,1]_{\text{per}}^{d}$, where $\Delta$ is the Laplacian operator. To solve the Schr\"{o}dinger equation, one observe collocation points $x_i$ uniformly sampled from $\mathbb{T}^d$ with associated function values $y_i = f(x_i) + \varepsilon_i$ ($1 \leq i \leq n$) where $\varepsilon_i$ is a mean-zero i.i.d observational noise. 
\label{schrodinger:example}
\end{example}
%\honam{I also suggest using a Venn diagram to represent relationship of function spaces here}\yplu{\url{https://app.diagrams.net/} here is a nice webcit to do that. I think you can share the link and everybody can edit together. But I'm still not 100\% sure what's the right answer.}\honam{\url{https://drive.google.com/file/d/1LUhwi6AFdt1ocFrtLB8DMtA14-YzSVVD/view?usp=sharing}}

\paragraph{Decomposition of Signals} Following \cite{Bartlett_2020,tsigler2023benign,cheng2024characterizing}, we decompose the risk estimation to the "low dimension" part which concentrates well and "higher dimension" part which performs as regularization. We define the decomposition operations in this paragraph. We first additionally define $\phi_{\leq k}: f \mapsto (\langle f, \phi_i \rangle_{\mathcal{H}})_{i = 1}^{k}$ which maps $\mathcal{H}$ to it's "low dimensional" features in $\mathbb{R}^{k}$, it intuitively means casting $f \in \mathcal{H}$ to its top $k$ features, similarly we can define $\phi_{>k}: f \mapsto (\langle f, \phi_i \rangle_{\mathcal{H}})_{i = k+1}^{\infty}$. We also define  $\phi^{*}_{\leq k}$ takes $\theta \in \mathbb{R}^k$ to $\sum_{i=1}^{k} \theta_i \phi_i$, similarly we can define $\phi^{*}_{>k}$ takes $\theta \in \mathbb{R}^{\infty}$ to $\sum_{i=k+1}^{\infty} \theta_{i-k} \phi_i$. For function $f \in \mathcal{H}$, we also define $f_{\leq k} := \phi_{\leq k}^* \phi_{\leq k} f = \sum_{i=1}^{k}{\langle f, \phi_i \rangle_{\mathcal{H}} \phi_i}$ which intuitively means only preserving the top $k$ features, for operator $\mathcal{A}: \mathcal{H} \to \mathcal{H}$, we also define $\mathcal{A}_{\leq k}: f \mapsto (\mathcal{A} f)_{\leq k}$. Similarly we could define $f_{>k}$ and $\mathcal{A}_{>k}$. We could show the decomposition $f = f_{\leq k} + f_{>k}$ and $\mathcal{A} = \mathcal{A}_{\leq k} + \mathcal{A}_{>k}$ holds for both signal and operators which is formally proved in  Lemma \ref{lemma:decomposition} in the appendix. 

We use $\| \cdot \|$ to denote standard $l^2$ norm for vectors, and operator norm for operators. We also use standard big-O notation $O(\cdot), o(\cdot), \Omega(\cdot), \tilde{O}(\cdot)$ (ignore logarithimic terms).
\section{Main Theorem: Excess Risk of Kernel Estimator for Inverse Problem}

Using the notations in Section \ref{Appendix:notations}, we can reformulate the data generating process as $y = \hat{S}_n \mathcal{A} f^* + \varepsilon$, where $y\in\mathbb{R}^n$ is the label we observed on the $n$ data points $\{x_i\}_{i=1}^n$, $f^*$ is the ground truth function and $\varepsilon \in \mathbb{R}^n$ is the noise. We first provide closed form solutions to ridge regression via the recently developed generalized representer theorem for inverse problem \cite{unser2021unifying}.
\begin{myorangebox}
\begin{lemma} The least square problem regularized by Reproducing Kernel Sobolev Norm 
\begin{equation}
    \label{eq:objective}
   \hat f_\gamma :=\argmin_{f\in\mathcal{H}^\beta} \frac{1}{n}\|\hat S_n\mathcal{A}f-y\|^2+\gamma_n \|f\|_{\mathcal{H}^\beta}^2.
\end{equation}
has the finite-dimensional representable closed form solution $\hat{f} = \mathcal{A}\Sigma^{\beta - 1}\hat{S}^{*}_n\hat{\theta}_n$  where
\begin{align*}
    \hat{\theta}_n &:= \underbrace{(\hat{S}_n \mathcal{A}^2 \Sigma^{\beta - 1} \hat{S}_n^* + n\gamma_n I}_{\tilde{K}^{\gamma}})^{-1}  y \in\mathbb{R}^n\ .
\end{align*}

%\yplu{$f = \mathcal{A}\mathcal{B}^{-2}\hat{S}^{*}_n\hat{\theta}_n$ where $\hat{\theta}_n = (\hat{S}_n \mathcal{A}^2 \mathcal{B}^{-2} \hat{S}_n^* + \gamma I)^{-1}  y$.}
\end{lemma} 
\end{myorangebox}

\begin{proof} As mentioned in Definition \ref{definition:sobolev_norm}, we have
$\|f\|_{\mathcal{H}^{\beta}}=\|\Sigma^{\frac{1 - \beta}{2}} f\|_{\mathcal{H}}$ thus we can rewrite the objective function (\ref{eq:objective}) as 
$$\hat f_\gamma=\argmin \frac{1}{n} \|\hat{S}_n \mathcal{A} f-y\|^2 + \gamma_n \| \Sigma^{\frac{1 - \beta}{2}} f\|_{\mathcal{H}}\Leftrightarrow \Sigma^{\frac{1 - \beta}{2}}\hat f_\gamma=\argmin \frac{1}{n}\| \hat{S}_n \mathcal{A}\Sigma^{\frac{\beta-1}{2}} g-y\|^2 + \gamma_n \| g\|_{\mathcal{H}}.$$

By representer theorem for inverse problem \cite{unser2021unifying}, the solution of the optimization problem $g_\gamma=\arg\min \frac{1}{n}\| \hat{S}_n \mathcal{A}\Sigma^{\frac{\beta-1}{2}} g-y\|^2 + \gamma_n \| g\|_{\mathcal{H}}$ have the finite dimensional representation that $g_\gamma =\mathcal{A} \Sigma^{\frac{\beta-1}{2}} \hat{S}^{*}_{n} \hat{\theta}_n $ for some $\hat\theta_n\in \mathbb{R}^n$. Then we know the $\hat f_\gamma = \Sigma^{\frac{\beta-1}{2}}g_\gamma = \mathcal{A} \Sigma^{\beta - 1} \hat{S}^{*}_{n} \hat{\theta}_n$, for some $\hat\theta_n \in \mathbb{R}^n$. Plug the finite dimensional representation of $\hat f_\gamma$ to objective function (\ref{eq:objective}) thus we have
 $$
 \hat{\theta}_n=\argmin_{\theta_n\in\mathbb{R}^n} \frac{1}{n}\|\hat{S}_n\mathcal{A}^2 \Sigma^{\beta - 1} \hat{S}^{*}_{n} \hat{\theta}_n-y\|^2+\gamma_n \|\Sigma^{\frac{1-\beta}{2}} \mathcal{A} \Sigma^{{\beta-1}} \hat{S}^{*}_{n}\theta_n\|_{\mathcal{H}}^2.
 $$
Thus we have  $\hat{\theta}_n = (\hat{S}_n \mathcal{A}^2  \Sigma^{\beta-1} \hat{S}_n^*\hat{S}_n \mathcal{A}^2  \Sigma^{\frac{1 - \beta}{2}} \hat{S}_n^* + \gamma_n \hat{S}_n \mathcal{A}^2  \Sigma^{\beta-1} \hat{S}_n^*)^{-1} (\hat{S}_n \mathcal{A}^2 \Sigma^{\beta-1} \hat{S}_n^*) y= (\hat{S}_n \mathcal{A}^2\Sigma^{{\beta-1}} \hat{S}_n^* + n \gamma_n I)^{-1}  y$. (For $\mathcal{A}$ is self-adjoint and co-diagonalizable with $\Sigma$.)

\begin{comment}

{\color{red}
Yiping's derivation
\begin{enumerate}
  \item $\min \|f\|_{H}$ s.t. $Af(x)=y$ -- this means $\hat{S}_n Af = y$
  
  $\Rightarrow \min \|\hat{S}_n A f - y\|_2^2 + \gamma \|f\|_{H}^2$ by Representer Thm. $f^* = A S_n^\ast \hat{\theta}_n$ in $\mathbb{R}^n$
  
  $\Rightarrow \hat{\theta}_n = (\hat{S}_n A A^* \hat{S}_n^*\hat{S}_n A A^* \hat{S}_n^* + \gamma \hat{S}_n A A^* \hat{S}_n^*)^{-1} (\hat{S}_n A A^* \hat{S}_n^*) y$ 
  
  $\Rightarrow \hat{\theta}_n = (\hat{S}_n A A^* \hat{S}_n^* + \gamma I)^{-1} y $

  \item $\min \|f\|_{H^s}$ s.t. $Af(x)=y$

  Another way to consider the problem is   $\min \|B f\|_{H}$ s.t. $A(f(x)) = y$. 
  
  $\Rightarrow \min \|\hat{S}_n A f - y\|_2^2 + \gamma \|f\|_{H}^2$ by Representer Thm. $f^* = A B^{-2} \hat{S}_n\hat{\theta}_n$ 
  
  $\hat{\theta}_n = (\hat{S}_n A AB^{-2} \hat{S}_n^*\hat{S}_n A AB^{-2} \hat{S}_n^* + \gamma \hat{S}_n A AB^{-2} \hat{S}_n^*)^{-1} (\hat{S}_n A AB^{-2} \hat{S}_n^*) y$
  
  $\Rightarrow f = B^* (f^*)^{-1} = \bar{A} B (B^* K)^{-1}$
  
  $\hat{\theta}_n = (\hat{S}_n A A^* B^* \hat{S}_n^* + \gamma \hat{S}_n A A^* B^* \hat{S}_n^*)^{-1} (\hat{S}_n A A^* B^* \hat{S}_n^*) y$
  
  $\Rightarrow \hat{\theta}_n = (\hat{S}_n A AB^{-2} \hat{S}_n^* + n \gamma I)^{-1}  y$ 
  
  $f^\ast = (\hat{S}_n A AB^{-2} \hat{S}_n^* + n \gamma I)^{-1}  y$
\end{enumerate}}
    
\end{comment}
\end{proof}

For the simplicity of presentation, We denote the empirical spectrally transformed kernel $\hat{S}_n \mathcal{A}^2 \Sigma^{\beta - 1} \hat{S}_n^*$ as $\tilde{K}$, and the regularized version $\hat{S}_n \mathcal{A}^2 \Sigma^{\beta - 1} \hat{S}_n^* + n\gamma I$ as $\tilde{K}^{\gamma}$, and we denote the spectrally transformed covariance operator $\tilde{\Sigma}$ as $\mathcal{A}^2 \Sigma^{\beta}$. 

%\honam{Change subsection name}
\subsection{Excess Risk and Eigenspectrum of spectrally transformed kernel $\tilde{K}$}
We evaluate excess risk in a certain Sobolev space $\mathcal{H}^{\beta'}$ with $\beta'\in [0,\beta]$.
The selection of $\beta'$ is independent of certain learning algorithms on source and capacity conditions, but depends on the downstream applications of learned inverse problem solution. We denote $\hat{f}:=\mathcal{A}\Sigma^{\beta - 1}\hat{S}^{*}_n(\hat{S}_n \mathcal{A}^2 \Sigma^{\beta - 1} \hat{S}_n^* + n\gamma I)^{-1}  y$ as $\hat{f}(y)$ to highlight its dependence on $y \in \mathbb{R}^n$. Recall the data generation process, $y = \hat{S}_n \mathcal{A} f^* + \varepsilon$, we consider $\hat{S}_n \mathcal{A} f^*$ and $\varepsilon$ in bias and variance separately. The excess risk $R(\hat{f}(y)) := \| \hat{f} - f^* \|^2_{H^{\beta'}} $ has the following bias-variance decomposition.

%$+ \underbrace{ \| P^{\perp}_{\mathcal{H}^{\beta'}} f^*\|_{L^2(\rho_{\mathcal{X}})}^2}_{\text{Approximation Error}}$ where $P_{\mathcal{H}^{\beta'}}$ is the projection operator that projects function into $\mathcal{H}^{\beta'}$, and $P^\perp_{\mathcal{H}^{\beta'}}$ projects into $(\mathcal{H}^{\beta'})^{\perp}$. \honam{If we assume $f^*$ lies in $\mathcal{H}^{\beta} \subset \mathcal{H}^{\beta'}$ then maybe we don't need the projection}

\begin{equation}
\begin{aligned}
    \| \hat{f} - P_{\mathcal{H}^{\beta'}}f^* \|^2_{H^{\beta'}} = \underbrace{\| \hat{f}(\hat{S}_n \mathcal{A} f^*) - f^* \|_{\mathcal{H}^{\beta'}}^2}_{\text{bias: }B} + \underbrace{\mathbb{E}_{\varepsilon}[\| \hat{f}(\varepsilon)\|^2_{\mathcal{H}^{\beta'}}]}_{\text{variance:} V}. 
\end{aligned}
\end{equation}

% \noindent 
% \yplu{$\Lambda^{\le k}_{\mathcal{X}}$ just as the $k\times k$ submatrix }
% \honam{Maybe define the spectrally transformed covariance operator $\tilde{\Sigma}$ also}.

Following \cite{barzilai2023generalization, Bartlett_2020,cheng2024characterizing}, we split the eigenvalues into two components, $\leq k, >k$ and bound them separately. Therefore, we define $\tilde{K}_{\leq k}$ as $\hat{S}_n \mathcal{A}_{\leq k}^2 \Sigma_{\leq k}^{\beta - 1} \hat{S}^{*}_n$, and $\tilde{K}_{\leq k}^{\gamma}$ as $\tilde{K}_{\leq k} + n\gamma_n I$, similarly we can define $\tilde{K}_{>k}$ and $\tilde{K}_{>k}^{\gamma}$ respectively. We can also have $\tilde{K} = \tilde{K}_{\leq k} + \tilde{K}_{>k}$ (proved in Appendix~\ref{lemma:decomposition}). To bound the excess risk of mininum norm interpolation estimator, we need to show the ”high dimensional” part of the Kernel matrix $\tilde{K}_{>k}$ is similar to $\tilde{\gamma}I$ and thus can behave as a self-regularization. To show this, we present here the concentration bounds of eigenvalues with proof given in Appendix \ref{theorem:eigenspectrum_proof}.
% Then we here present the concentration bounds of the eigenvalues:\yplu{remember to copy motivation of every theorem in \cite{barzilai2023generalization} here}
%\honam{TODO: include motivation of introducing eigenspectrum here}

\begin{theorem}[Eigenspectrum of spectrally transformed kernel $\tilde{K}$]
    \label{theorem:eigenspectrum}
    Suppose Assumption \ref{assumption:beta_regularity_condition} holds, and eigenvalues of $\tilde{\Sigma}$ are given in non-increasing order (i.e. $2p + \beta \lambda > 0$).
    There exists absolute constant $c, C,c_1,c_2>0$ s.t. for any  $k\leq k'\in[n]$ and $\delta>0$, it holds w.p. at least $1-\delta-4\frac{r_k}{k^4}\exp(-\frac{c}{\beta_k}\frac{n}{r_k})-2\exp(-\frac{c}{\beta_k}\max\left(\frac{n}{k},\log(k)\right))$  that
    $$\mu_{k}\left(\frac{1}{n} \tilde{K} \right) \leq c_1 \beta_{k}\left(\left(1+\frac{k \log (k)}{n}\right) \lambda_{k}^{\beta} p_k^2 +\log (k+1) \frac{\operatorname{tr}\left(\tilde{\Sigma}_{>k} \right)}{n}\right),$$
    $$ \mu_{k}\left(\frac{1}{n} \tilde{K} \right) \geq c_2 \mathbb{I}_{k, n} \lambda_{k}^{\beta} p_k^2 +\alpha_{k}\left(1-\frac{1}{\delta} \sqrt{\frac{n^{2}}{\trace(\tilde{\Sigma}_{>k'})^2/\trace(\tilde{\Sigma}^2_{>k'})}}\right) \frac{\operatorname{tr}\left( \tilde{\Sigma}_{>k'} \right)}{n}, $$
    where $\mu_k$ is the $k$-th largest eigenvalue of $\tilde{K}$, $\tilde{\Sigma} := \mathcal{A}^2 \Sigma^{\beta}$, $r_k := \trace(\tilde{\Sigma}_{>k})/(p_{k+1}^2\lambda_{k+1}^\beta)$, and $\mathbb{I}_{k,n}=\begin{cases}
			1, & \text{if } C\beta_kk\log(k)\leq n\\
            0, & \text{otherwise}
		 \end{cases}$.
%    \honam{Need to replace $R_k'$ with something else @Wendao}\wendao{done}
%     \yplu{don't put anything without a proof in this overleaf project!}

%     \yplu{All theorems in this overleaf need to be correct}

% \yplu{you can always use that overleaf project as draft. }
   
\end{theorem}
\begin{remark}
    Informally it can be understood as the spectrally transformed kernel $\tilde{K} \approx \tilde{K}_{\leq k} + \tilde{\gamma} I$ and $\mu_i(\frac{1}{n} \tilde{K}_{\leq k}) \approx \lambda_k^{\beta} p_k^2$. $\mathbb{I}_{k,n}$ here also gives constraint on choice of $k$ in Section \ref{section:applications} where $k$ should be $O(\frac{n}{\log(n)})$.
\end{remark}

% \begin{proof}[Proof Sketch]
    
% \end{proof}

%\fh{cite Ohad's paper, and talk about the relationship on $\rho_{n,k}$ and $\alpha_k$, $\beta_k$}
\subsection{Concentration Coefficients}
We expect that $\tilde{K}_{>k} \approx \tilde{\gamma} I$ which serves as a self-regularization term, inspired by \cite{barzilai2023generalization} we quantify this by introducing the concentration coefficient for spectrally transformed kernel $\tilde{K}$. %\honam{TODO: explain why adding $\tilde{\Sigma}_{>k}$ here}

%\honam{Here we basically re-define concentration coefficient using $\tilde{K}$ and $\tilde{\Sigma}$}

\begin{definition}[Concentration Coefficient $\rho_{n,k}$]\label{def:concerntration} We quantify this by what we call the
concentration coefficient
    $$ \rho_{k,n}:=\frac{\| \tilde{\Sigma}_{>k}\| + \mu_1(\frac{1}{n} \tilde{K}_{>k}) + \gamma_n}{\mu_n(\frac{1}{n} \tilde{K}_{>k}) + \gamma_n}, \quad\quad \text{where} \quad \tilde{\Sigma}=\mathcal{A}^2 \Sigma^{\beta}.$$
\end{definition}

Assumptions on  feature map is essential to obtain various concentration inequalities, typically sub-Gaussian assumptions on feature map is needed to obtain concentration results. However, this does not hold for many
common kernels. Following recent work \cite{barzilai2023generalization}, we only require mild condition on features i.e. $\alpha_k , \beta_k = o(1)$ which is applicable in many common kernels, without imposing sub-Gaussian assumptions, but our bound in the interpolation case can be tighter with the sub-Gaussian assumption in Theorem \ref{theorem:polynomial_bias_variance_interpolate}.
\begin{assumption}[Well-behaved features] \label{assumption:well_behaveness_feature} Given $k\in\mathbb{N}$,  we define $\alpha_k,\beta_k$ as follows.
    \begin{align*}
        &\alpha_k := \inf_x \min \left\{ \frac{\sum_{i>k} p_i^{a } \lambda_i^{b} \psi_i(x)^2}{\sum_{i>k} p_i^{a } \lambda_i^{b}}
        : {\text{finite choices of }} a, b \right  \}, \\
        &\beta_k := \sup_x \max \left \{ \frac{\sum_{i=1}^{k} \psi_i(x)^2}{k}, \frac{\sum_{i>k} p_i^{a } \lambda_i^{b} \psi_i(x)^2}{\sum_{i>k} p_i^{a } \lambda_i^{b}}
        : {\text{finite choices of }} a, b \right\}, \\
        %\sup_x \max \{ \frac{\| \psi K_x\|^2}{2k}\} \\
    \end{align*}
    %\honam{@Wendao, the notation may have some problem since we don't have $\psi_i(x)$? Can you change that?}
    $(a, b)$ is picked in our proof of Lemma \ref{lemma:simultaneous_concentration} in the Appendix. Since $\inf \leq \mathbb{E} \leq \sup$, one always has $0 \leq \alpha_k \leq 1 \leq \beta_k$. We assume that $\alpha_k, \beta_k = \Theta(1)$.
    % \frac{\| \Lambda^{>k}_{\mathcal{A}^{a / 2} \Sigma^{b / 2}} \psi K_x\|^2}{\trace(\mathcal{A}^{a}_{>k} \Sigma^{b}_{>k})} \cdots 
    %where $a, b$ is pre-determined in the lemmas \ref{lemma:concentration_1}.
    % \honam{Here we should assume weaker condition like \cite{barzilai2023generalization}, or we directly assume sub-Gaussianity of features but it is rather strong}
    %\honam{We need to show $\alpha_k$ and $\beta_k$ is $o(1)$, this is proven for many common kernels in their original definition in \cite{barzilai2023generalization} and they show this condition is weaker than the subgaussian condition, can we still show it being $o(1)$ in our definition?}
    %\honam{We should use notation other than $\alpha_k$, $\beta_k$ to avoid collision with $\beta, \beta'$ and $a, b$.}
    % \honam{Remark: $\alpha_k$ and $\beta_k$ is $o(1)$ in many common kernels as said in \cite{barzilai2023generalization}, we should keep the original definition and not put $\Sigma$, $\mathcal{A}$ inside}
\end{assumption}
% \yplu{specify the $a,b$ later, just assume for the parameter you used for your proof.}\honam{Maybe we need to change another notation than $a, b$ to avoid confusion with $\alpha_k$}\yplu{yes}\fh{also $\beta_k$}

\begin{remark}
    For each term in these definitions, the denominator is the expected value of the numerator, so $\alpha_k$ and $\beta_k$ quantify how much the features behave as they are ”supposed to”. Note that $\alpha_k$ and $\beta_k$ are $\Theta(1)$ in many common kernels. 

   We here give several examples that satisfies the assumptions, includes
\begin{itemize}
    \item \emph{Kernels With Bounded Eigenfunctions} If $\psi_i^2(x) < M$ uniformly holds for $\forall i, x$ then Assumption \ref{assumption:beta_regularity_condition} trivially holds that $\beta_k \leq M$ for any $k \in \mathbb{N}$. Analogously, if $\psi_i^2 \geq M'$ then $\alpha_k \geq M'$. This may be weakened to the the training set such that only a high probability lower bound is needed. Kernels satisfies this assumption includes RBF and shift-invariant kernels \cite[Theorem 3.7]{steinwart2006explicit} and Kernels on the Hypercube $\{0,1\}^d$ of form $h\left(\frac{\langle x, x' \rangle}{\|x\|\|x'\|}, \frac{\|x\|^2}{d}, \frac{\|x'\|^2}{d}\right)$ \cite{yang2019fine}.
    \item \emph{Dot-Product Kernels on $\mathcal{S}^d$} Follows the computation in \cite[Appendix G]{barzilai2023generalization}, one can know dot-product Kernels on $\mathcal{S}^d$ satisfies Assumption \ref{assumption:beta_regularity_condition}.
\end{itemize}     
\end{remark}

Similar to \cite{barzilai2023generalization}, we require regularity condition on $\beta_k$ to overcome technical difficulty in extending to infinite dimension in Lemma~\ref{lemma:upper_bound_largest_eigen}:
\begin{assumption}[Regularity assumption on $\beta_k$]
    There exists some sequence of natural numbers $(k_i)^\infty_{i=1}\subset \mathbb{N}$ with $k_i\underset{i\to\infty}{\longrightarrow}\infty \text{\ s.t.\ } \beta_{k_i}\trace(\tilde{\Sigma}_{> k_i})\underset{i\to\infty}{\longrightarrow}0$.
    \label{assumption:beta_regularity_condition}
\end{assumption}
We can know $\tilde{\Sigma}_{>k_i}$ is still transformed trace class, so one always has $\trace(\tilde{\Sigma}_{> k_i})\underset{i\to\infty}{\longrightarrow}0$. As such, Assumption~\ref{assumption:beta_regularity_condition} simply states that for infinitely many choices of $k\in\mathbb{N}$,  $\beta_k$ does not increase too quickly. This is of course satisfied by the previous examples of kernels with $\beta_k=\Theta(1)$.

% \yplu{why don't just assume as $\text{tr}*\tilde(\Sigma)$ exists}
% \honam{\cite{barzilai2023generalization} requires this and also Our Lemma~\ref{lemma:upper_bound_largest_eigen} needs this to overcome the technical difficulty in extending to infinite dim}

\subsection{Main Results}

In this section, we state our main results on the bias and variance of the estimators.

The following theorem is the main result for upper bounds of the bias and variance with the proof details given in Appendix \ref{theorem:variance_proof} for bounding the variance and Appendix \ref{theorem:bias_proof} for bounding the bias.
\begin{theorem}[Bound on Variance]\label{main:var} Let $k\in\mathbb{N}$ and $\rho_{k,n}$ is defined follows Definition \ref{def:concerntration}, then the variance can be bounded by 
\begin{align*}
        V &\leq \sigma_{\varepsilon}^2\rho_{k,n}^2 \cdot  {\Bigg(}\frac{{\trace(\hat{S}_n \psi^*_{\leq k} \Lambda^{\leq k}_{\mathcal{A}^{-2} \Sigma^{-\beta'}} \psi_{\leq k} \hat{S}^*_n)}}{{\mu_k(\psi_{\leq k} \hat{S}^*_n \hat{S}_n \psi^*_{\leq k})^2}}
    + \frac{\overbrace{\trace(\hat{S}_n \psi^*_{>k} \Lambda^{>k}_{\mathcal{A}^2 \Sigma^{-\beta' + 2\beta}} \psi_{>k} \hat{S}^*_n)}^{\text{effective rank}}}{n^2 {\| \tilde{\Sigma}_{>k} \|^2}} {\Bigg)}.\\
    \end{align*}
\end{theorem}
%\yplu{check assumptions}
\begin{remark}
    The variance bound is decomposed into two parts, the $\le k$ part which characterize the variance of learning the "low dimension" components and $\ge k$ part characterizing the variance of learning "high dimension" components.  We did similar analsyis for the bias as follows. 
 \end{remark}

%\yplu{$\trace(\hat{S}_n \psi^*_{>k} \Lambda^{>k}_{\mathcal{A}^2 \Sigma^{-\beta' + 2\beta}} \psi_{>k} \hat{S}^*_n)\approx n\sum_{i=k+1}^{\infty}{i^{-2p + \lambda(\beta' - 2\beta)}}$ to let this term not diverge, one need }

\begin{theorem}[Bound on Bias] \label{main:bias}Let $k\in\mathbb{N}$ and $\rho_{k,n}$ is defined follows Definition \ref{def:concerntration}, then for every  $\delta>0$, with probability $1-\delta - 8\exp(-\frac{c'}{\beta_k^2} \frac{n}{k})$, the bias can be bounded by 
\begin{equation}
\begin{aligned}
       B\lesssim \rho_{k,n}^3 \frac{1}{\delta} {\Big[}   \frac{ 1}{{p_k^2 \lambda_k^{\beta'}} } (  \| \phi_{>k} \mathcal{A}_{>k} f_{>k}\|^2_{\Lambda^{>k}_{\Sigma}}) + \left(\gamma_n + \frac{\beta_k \trace(\tilde{\Sigma}_{>k})}{n}\right)^2 \frac{ \|\phi_{\leq k} f_{\leq k}^*\|^2_{\Lambda^{\leq k}_{\mathcal{A}^{-2} \Sigma^{1 - 2\beta}}}}{p_k^2 \lambda_k^{\beta'}} + \ \|\phi_{>k} f^*_{>k}\|^2_{\Lambda^{>k}_{\Sigma^{1-\beta'}}} 
        % \frac{\| \Lambda^{>k}_{\mathcal{A}^4 \Sigma^{2\beta}} \|}{\mu_n(\frac{1}{n} \tilde{K}^{\gamma}_{>k})^2}
        % &+  \  \| \Lambda^{>k}_{\mathcal{A}^{-4} \Sigma^{1 - \beta' - 2\beta}}\| (  \| \phi_{>k} \mathcal{A}_{>k} f_{>k} \|^2_{\Lambda_{\Sigma}^{>k}}) ( p_{k+1}^2 \lambda_{k+1}^{2\beta - 1}) \\
        % &+  (\gamma_n + \frac{\beta_k \trace(\mathcal{A}^2_{>k} \Sigma_{>k}^{\beta})}{n})^2  \| \Lambda^{>k}_{\mathcal{A}^{-2} \Sigma^{-\beta'}}\| \ \|\phi_{\leq k} f^*_{\leq k}\|_{\Lambda^{\leq k}_{\mathcal{A}^{-2} \Sigma^{1 - 2\beta}} }^2
        {\Big]}.
    \end{aligned}
%     \begin{aligned}
%         B&\lesssim \rho_{k,n}^3{\Bigg (}
%         \frac{ \|\phi_{\leq k} f_{\leq k}^*\|^2_{\Lambda^{\leq k}_{\mathcal{A}^{-2} \Sigma^{1 - 2\beta}}}}{\mu_n((\tilde{K}_{>k}^{\gamma})^{-1} )^2p_k^2\lambda_k^{\beta'}}+\frac{\| \Lambda^{>k}_{\Sigma^{-\beta'+\beta}}\|  \| \phi_{\leq k} f^*_{\leq k}\|_{\Lambda^{\leq k}_{\mathcal{A}^{-2} \Sigma^{1 - 2\beta}}}}{\mu_n((\tilde{K}^{\gamma}_{>k})^{-1})^2}\\
%         &+\frac{ \| \phi_{>k} \mathcal{A}_{>k} f_{>k}\|^2_{\Lambda^{>k}_{\Sigma}}}{\delta{p_k^2 \lambda_k^{\beta'}}}+\ \|\phi_{>k} f^*_{>k}\|^2_{\Lambda^{>k}_{\Sigma^{1-\beta'}}} + \frac{\| \Lambda^{>k}_{\mathcal{A}^{-2} \Sigma^{-\beta'}}\| \ \|\phi_{\leq k} f^*_{\leq k}\|_{\Lambda^{\leq k}_{\mathcal{A}^{-2} \Sigma^{1 - 2\beta}} }^2}{\mu_1[(\tilde{K}_{>k}^
% \gamma)^{-1}]^2} {\Bigg)}
%     \end{aligned}
\end{equation}
\end{theorem}

\section{Applications}
\label{section:applications}

Our main results can provide bounds for both the regularized \cite{yang2021inference,lu2022sobolev} and unregularized cases \cite{chen2021solving} with the same tools. In this section, we present the implication of our results for both regularized regression and minimum norm interpolation estimators.

\subsection{Regularized Regression}
In this section, we demonstrate the implication of our derive bounds for the classical setup where the regularization $\gamma_n$ is relatively large. We consider regularized least square estimator with regularization strength $\gamma_n=\Theta(n^{-\gamma})$. By selecting $k$ as $\lceil n^{\frac{\gamma}{2p + \beta \lambda}} \rceil$ in Theorem  \ref{main:var} and Theorem \ref{main:bias}, we obtain $\rho_{k,n} = o(1)$ and get a bound that matches \cite{lu2022sobolev}, which indicates the corectness and tightness of our results.

\begin{mygraybox}
    \begin{theorem}[Bias and Variance for Regularized Regression]\label{theorem:bias_variance_regularized} Let the kernel and target function satisfies Assumption \ref{assumption:kernel}, \ref{assumption:well_behaveness_feature} and \ref{assumption:beta_regularity_condition}, $\gamma_n=\Theta(n^{-\gamma})$, and suppose $2p + \lambda \beta > \gamma > 0, 2p + \lambda r>0$, and $r > \beta'$, then for any $\delta > 0$, it holds w.p. at least $1 - \delta - O(\frac{1}{n})$ that 
    %with high probability\yplu{???}\fh{The probability is the same as Ohad, $1 - \delta - 16 xxxx$}\fh{also under well-behaved features? The condition $\beta_k \log k xxx < n$ is needed.}}
    
    \begin{align*}
        V \leq\sigma_{\varepsilon}^2 O(n^{ \max \{ \frac{\gamma (1 + 2p + \lambda \beta')}{2p + \lambda \beta}, 0 \} - 1}),B \leq \frac{1}{\delta} \cdot O(n^{\frac{\gamma}{2p + \beta \lambda}(\max \{ {\color{myorange}\lambda (\beta'-r)},{ \color{myblue}-2p+\lambda (\beta' 
- 2\beta)} \})}).\\
    \end{align*}
\end{theorem}
\end{mygraybox}

%\yplu{here we don't need control of $\alpha_k$?}
%\honam{I think we need to explain the conditions here i.e. $\gamma < 2p + \lambda \beta$ and $2p + \lambda \beta' > 0$ and $r > \beta'$.}
\begin{remark} Once proper regularization norm is selected, \emph{i.e.  $\lambda \beta \geq \frac{\lambda r}{2} - p$},  with optimally selected ${\gamma }=\frac{2p + \lambda \beta}{(2p + \lambda+2r)}$ which balance the variance $n^{\frac{\gamma (1 + 2p + \lambda \beta')}{2p + \lambda \beta} - 1}$  and the bias $n^{\frac{\gamma(\lambda(\beta'-r))}{2p+\beta\lambda}}$, our bound  can achieve final bound: $n^{\frac{\lambda(\beta'-r)}{2p+\lambda r+1}}$ matches with the convergence rate build in the literature \cite{knapik2011bayesian,lu2022sobolev}
\end{remark}

%\yplu{your $r$ with the source condition $r$,should have the following replacemetn $1-2r=\lambda(1-r)$}

\subsection{Min-norm Interpolation from benign overfitting to tempered overfitting}

We now shift our attention to the overparameterized interpolating estimators. Recently, \cite{mallinar2022benign} distinguished between three regimes: one where the risk explodes to infinity (called catastrophic overfitting), another where the risk remains bounded (called tempered overfitting), and a third regime involving consistent estimators whose risk goes to zero (called benign overfitting).  These two regimes are significantly different. In the tempered overfitting regime, when the noise is small, estimator can still achieve a low risk despite overfitting. This means that the bias goes to zero, and the variance cannot diverge too quickly. Recent work \cite{rakhlin2019consistency,cui2021generalization,barzilai2023generalization,cheng2024characterizing} showed that  minimum (kernel) norm interpolators are nearly tempered over-fitting. However, as shown in Theorem \ref{theorem:polynomial_bias_variance_interpolate}, \emph{\textbf{the PDE operator in the inverse problem can stabilize the variance term and make the min-norm interpolation estimators benign over-fitting even in fixed-dimension setting.}}

%\yplu{check why $2p+\lambda \beta'>0$}
\begin{mygraybox}
    \begin{theorem}[Bias and Variance for Interpolators] \label{theorem:polynomial_bias_variance_interpolate}
     Let the kernel and target function satisfies Assumption \ref{assumption:kernel}, \ref{assumption:well_behaveness_feature} and \ref{assumption:beta_regularity_condition}, and suppose $2p + \lambda \min\{r,\beta\}> 0$ and $r > \beta'$, then for any $\delta > 0$ it holds w.p. at least $1 - \delta - O(\frac{1}{\log(n)})$ that
    %\fh{assumptions and probability}
    \begin{align*}
        V \leq \sigma_{\varepsilon}^2 \rho_{k,n}^2\tilde{O}(n^{\max\{2p+\lambda \beta',-1\}}), B \leq  \frac{\rho_{k,n}^3}{\delta}\tilde{O}(n^{\max  \{ {\color{myorange}\lambda (\beta'-r)}, {\color{myblue}-2p + \lambda(\beta' - 2\beta)} \}\}}).
    \end{align*}
\end{theorem}

\end{mygraybox}

\begin{comment}
$\rho_{k,n} = \tilde{O}(n^{2p + \beta \lambda - 1})$
\begin{theorem}[Bias and Variance for Interpolators]
    \label{theorem:polynomial_bias_variance}
    \begin{align*}
        V \leq \sigma_{\varepsilon}^2 \tilde{O}(n^{6p + \lambda (2\beta + \beta') - 2}), B \leq  \tilde{O}(n^{\max  \{ -2 - 2r + 6p - \lambda (1 - \beta' - 3\beta), -3 + 4p + \lambda (\beta' + \beta) \}\}})
    \end{align*}
\end{theorem}
\end{comment}

%\honam{I think we need to explain the conditions here i.e. $2p + \lambda \beta > 0$ and $2p + \lambda \beta' > 0$.}
\begin{remark}

For well-behaved sub-Gaussian features, the concentration coefficients $\rho_{k,n} = \Theta(1)$ \cite{barzilai2023generalization} and in the worst case $\rho_{k,n}$ can become $\tilde{O}(n^{2p + \beta \lambda - 1})$ which is shown in the appendix. Our bound can recover the results in \cite{barzilai2023generalization} by setting $p=0,\beta=1,\beta'=0$ and recover the results in \cite{cui2021generalization} when $\sigma_\epsilon=0,\beta'=0$ and $\rho_{k,n} = 1$. \emph{Since the $p$ considered for PDE inverse problems is a negative number (See Assumption \ref{assumption:kernel}), our bound showed that the structure of PDE inverse problem made benign over-fitting possible even in the fixed dimesional setting. This result differs the behavior of regression with inverse problem when large over-parameterized model is applied.}  The more negative $p$  leads to smaller bound over the variance which indicates Sobolev training is more stable to noise, matches with empirical evidence \cite{son2021sobolev,yu2021gradient,lu2022sobolev}. 
    %\honam{Here we require $2p + \beta \lambda - 1 \geq 0$ since $\rho_{k,n} \geq 1$ by definition}
%\fh{In our setting we need $2p + \beta \lambda > 0$ with $p<0$ and $\beta' \in [0,\beta]$, so talk about like this
%\begin{itemize}
%    \item benign overfitting: $2p + \beta' \lambda\leq 0$
%    \item tempered overfitting: similar to Ohad's paper
%\end{itemize}
%}

\end{remark}

\subsection{Implication of Our Results}

\paragraph{Selection of Inductive Bias:} As demonstrated in Theorem \ref{theorem:bias_variance_regularized}
and Theorem \ref{theorem:polynomial_bias_variance_interpolate}, variance is independent of the inductive bias (i.e., $\beta$) and the only dependency is appeared in bounding the bias. At the same time, the upper bound for the bias is a maximum of the {\color{myorange}orange} part and the {\color{myblue}blue} part. The {\color{myorange}orange} part is independent of the inductive bias and only depend on the inverse problem (i.e., $r$ and $\lambda$) and evaluation metric (i.e., $\beta'$), while the {\color{myblue}blue} part is the only part depending on the inductive bias used in the regularization. With properly selected inductive bias $\beta$, one can achieve the best possible convergence rate which only depends on the {\color{myorange}orange} part.  When the inductive bias does not focus much on the low frequency eigenfunctions (i.e., $\lambda\beta\le \frac{\lambda r}{2}-p$), that means, regularized with kernel which is not smooth enough, the rate is dominated by the {\color{myblue}blue} part and is potential sub-optimal. Under the function estimation setting, the selection matches the empirical understanding in semi-supervised learning \cite{zhou2008high,zhou2011semi,smola2003kernels,chapelle2002cluster,dong2020cure,zhai2024spectrally} and \emph{\textbf{theoretically surprisingly matches the smoothness requirement determined in the Bayesian inverse problem literature.}}\cite{knapik2011bayesian,szabo2013empirical}.

\paragraph{Takeaway to Practitioners:} Our theory demonstrated that to attain optimal performance in physics-informed machine learning, incorporating sufficiently smooth inductive biases is necessary. For PINNs applied to higher-order PDEs, one needs smoother activation functions. This is because the value of $p$ for higher-order PDEs is a negative number with a larger absolute value, thus making the term $\frac{\lambda r}{2}-p$ larger. A larger value of $\frac{\lambda r}{2}-p$ necessitates the use of smoother activation functions \cite{bietti2020deep,chen2020deep} to ensure the solution satisfies the required smoothness conditions imposed by the higher-order PDE. Another implication of the theory is {\bf the variance stabilization effects} as mentioned before brought about by the PDE operator in the inverse problem. Higher-order PDEs would benefit from more substantial stabilization effects. This motivates the idea that Sobolev training \cite{son2021sobolev,yu2021gradient} may not only aid optimization \cite{lu2022sobolev} but also contribute to improved generalization error for overparameterized models. However, as previously demonstrated, utilizing a neural network with smoother activations is necessary to leverage these benefits.

\section{Experiments}
%\vspace{-0.1in}
\begin{figure*}[htbp]
    \centering
    \vspace{-0.1in}
    \includegraphics[width=0.3\linewidth]{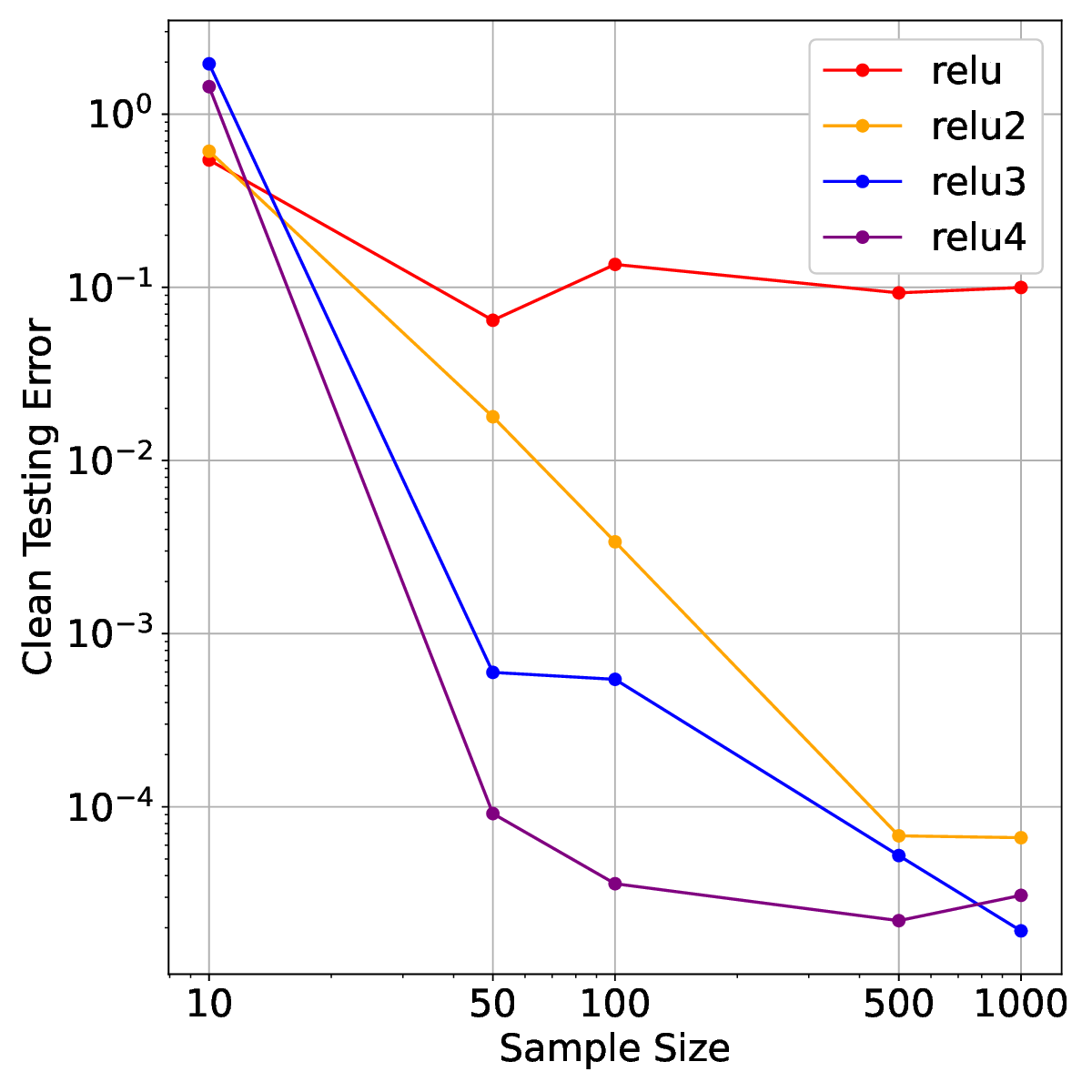}\includegraphics[width=0.3\linewidth]{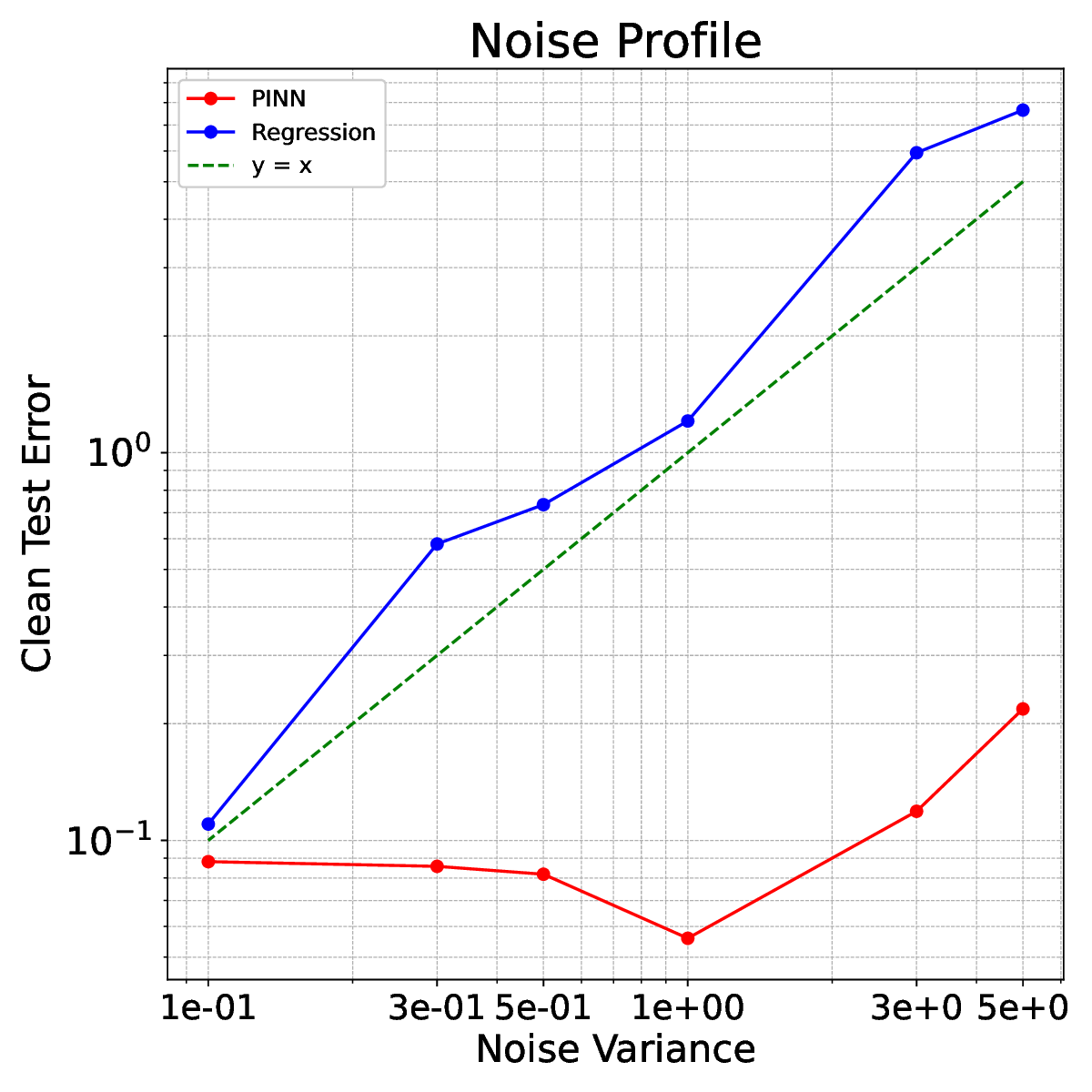}\includegraphics[width=0.3\linewidth]{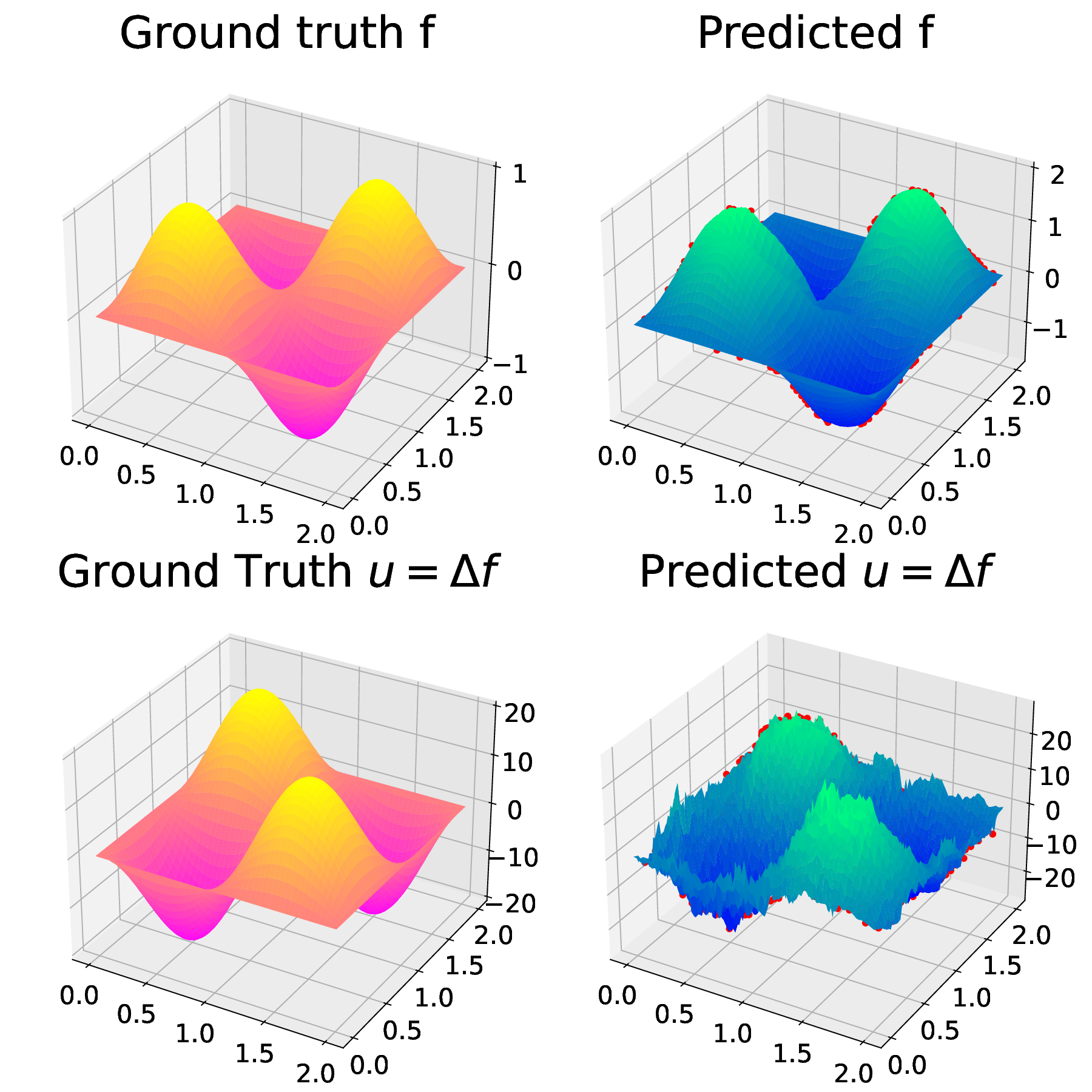}
    \vspace{-0.15in}
    \caption{We verified our finding beyond kernel estimators. For all the plotted figure, we learn two dimensional Poisson equation. \textbf{(Left)} We examine the impact of smooth inductive bias on convergence. Our findings demonstrate that when the activation function is sufficiently smooth, the inductive bias has a limited effect on improving convergence, which aligns with our theoretical predictions. \textbf{(Middle)} Noise profile of Physics-informed interpolator and regression Interpolator. The physics-informed interpolator exhibits benign overfitting, unlike the regression interpolator. \textbf{(Right)}  Visualization of the ground truth and the learned solutions for $f$ and $u = \Delta f$. The learned solution for $f$ effectively smooths out the high-frequency components in the error of $\Delta f$. }
    \label{fig:convergence_rate}
\end{figure*}
%\vspace{-0.1in}
We verify our theoretical findings beyond kernel estimators, we conducted experiments on Poisson equation and other PDE problems (refer to Appendix~\ref{appendix:exp}). To be specific, we consider the Poisson equation $u = \Delta f$ in $\Omega = [0,2]^2$ with value zero in $\partial \Omega$, where the ground truth $f(x_1,x_2) = \sin(\pi x_1) \sin(\pi x_2)$. The data points $\{(x_i, y_i)\}_{i=1}^{n}$ are uniformly sampled from $\Omega$, and $y_i = \Delta f(x_i) + \varepsilon$ with $\varepsilon \sim \mathcal{N}(0, \sigma^2)$. Our experiments illustrate our theory from the following three aspects.

\paragraph{Effect of Smoothness of the Inductive Bias}
We examine convergence rate of using activation function ReLU, $\text{ReLU}^2$, $\text{ReLU}^3$ by fixing noise level variance $\sigma^2 = 0.1$, and vary number of samples as 50, 100, 500, 1000 and plot the test error against number of samples. The result in Figure \ref{fig:convergence_rate}(Left) verifies our finding that when the inductive bias is not smooth enough, the convergence will benefit from smoother activation function. However, by comparing convergence rate of $\text{ReLU}^3$ and $\text{ReLU}^4$ in Figure \ref{fig:convergence_rate}(Left), when the activation function is smooth enough, the convergence behavior would not be affected too much.

\paragraph{Benign Overfitting of Physics-Informed Interpolator} Following \cite{benning2018modern}, we verify the benign overfitting behavior by plotting the noise profiles of the Physics-Informed interpolator and comparing with the regressor in Figure \ref{fig:convergence_rate}(Middle). A noise profile characterizes the sensitivity of a learning procedure to noise in the training set, specifically how the asymptotic risk varies with the variance of additive Gaussian noise. The figure demonstrates that the physics-informed interpolator is robust to noise while the regression interpolator exhibits tempered overfitting, as expected from \cite{mallinar2022benign}.
This supports our theory that Physics-Informed interpolator can still generalize well over noisy data, i.e., benign overfitting.

\paragraph{The Noise Stabilization Effect} We also plotted the final output of the neural network in Figure \ref{fig:convergence_rate}.  The intuition behind our theory of benign overfitting in inverse problems differs from that of standard regression because we predict \( \Delta^{-1} u \) rather than \( u \) in the regression setting. The operator \( \Delta^{-1} \) functions as a kernel smoothing mechanism, where the Green's function serves as the kernel. This smoothing process attenuates high-frequency components, which are the dominant contributors to the prediction error, and thus effectively alleviates their impact. For general PDEs governing physical laws, most behave like differential operators, where the forward problem amplifies high-frequency components. Consequently, solving the inverse problem tends to attenuate these high-frequency components, resulting in a similar noise stabilization effect.

\section{Conclusions}

In conclusion, we study the behavior of kernel ridge and ridgeless regression methods for linear inverse problems governed by elliptic partial differential equations (PDEs). Our asymptotic analysis reveals that the PDE operator can stabilize the variance and even lead to benign overfitting in fixed-dimensional problems, exhibiting distinct behavior compared to regression problems. Another key focus of our investigation was the impact of different inductive biases introduced by minimizing various Sobolev norms as a form of (implicit) regularization. Interestingly, we found that the final convergence rate is independent of the choice of smooth enough inductive bias for both ridge and ridgeless regression methods.
For the regularized least-squares estimator, our results demonstrate that all considered inductive biases can achieve the minimax optimal convergence rate, provided the regularization parameter is appropriately chosen. Notably, our analysis recovered the smoothness condition found by Empirical Bayes in the function regression setting and extended it to the minimum norm interpolation and inverse problem settings.

\bibliographystyle{unsrt}
\bibliography{kernel}

\newpage
\tableofcontents

\newpage
\appendix
\addcontentsline{toc}{section}{Appendices}
\section{Additional Notations and Some useful lemmas}
For brevity, we denote simplified notation for $\leq k$ and $>k$, for function $f \in \mathcal{H}$, we define $f_{\leq k} := \phi_{\leq k}^* \phi_{\leq k} f$, for operator $\mathcal{A}: \mathcal{H} \to \mathcal{H}$, we also define $\mathcal{A}_{\leq k}: f_{\leq k} \mapsto \phi_{\leq k}^* \phi_{\leq k} \mathcal{A} f_{\leq k}$.

\noindent We denote $\mu_n(M)$ as the $n$-th largest eigenvalue of some matrix $M$. We also define $id_{\leq k}$ and $id_{>k}$. We denote $[n]$ as integers between 1 and $n$.

$\phi_{\leq k} \hat{S}_n^*$ is the map from $\mathbb{R}^n \to \mathbb{R}^k$, therefore, we can consider it as $k \times n$ matrix, where each column is the top $k$ features of the data points. $\hat{S}^*_n \phi_{\leq k}$ is the map from $\mathbb{R}^k \to \mathbb{R}^n$, therefore, we can consider it $n \times k$ matrix, and $(\phi_{\leq k} \hat{S}_n^*)^T = \hat{S}^*_n \phi_{\leq k}$. Similar reasoning holds for $>k$ case.

Note that for simplicity, we always convert to using $\psi$ for convenient computation, by using the following: $\phi_{\leq k} = \Lambda_{\Sigma^{1/2}}^{\leq k} \psi_{\leq k}$ and $\phi^{*}_{\leq k} = \psi^{*}_{\leq k} \Lambda_{\Sigma^{1/2}}^{\leq k}$, also similar for $>k$. This is because $\mathbb{E}([\hat{S}_n \psi^*_{>k}]_{ji}^2) = 1$ by Lemma \ref{lemma:psi_expectation_one}. 

Next we deliver several useful lemmas.

%\yplu{a small comment $\phi_{>k}f$ should be the same as $\phi_{>k}f_{>k}$?} \honam{Yes, write like this is just for readability, we maybe should mention this in the paper also to avoid confusing the reviewers} \honam{Maybe add some explanations here.}

The following lemma justifies our $<k$ and $\geq k$ decomposition.
\begin{lemma}[Decomposition lemma] 
    \label{lemma:decomposition}
    The following holds:
    \begin{enumerate}
        \item For any function $f \in \mathcal{H}$, $f = f_{\leq k} + f_{>k}$;
        \item For any operator $\mathcal{A}: \mathcal{H} \to \mathcal{H}$, $\mathcal{A} = \mathcal{A}_{\leq k} + \mathcal{A}_{>k}$;
        \item For the spectrally transformed kernel matrix $\tilde{K}$, $\tilde{K} = \tilde{K}_{\leq k} + \tilde{K}_{>k}$.
    \end{enumerate}
\end{lemma}
\begin{proof}
    We first prove (1), 
    \begin{align*}
        f_{\leq k} + f_{>k} 
       & = \ \phi^*_{\leq k} \begin{pmatrix}
            \langle f, \phi_1 \rangle_{\mathcal{H}}\\
            \langle f, \phi_2 \rangle_{\mathcal{H}}\\
            \cdots\\
            \langle f, \phi_k \rangle_{\mathcal{H}}
        \end{pmatrix} + \phi^*_{> k} \begin{pmatrix}
            \langle f, \phi_{k+1} \rangle_{\mathcal{H}}\\
            \langle f, \phi_{k+2} \rangle_{\mathcal{H}}\\
            \cdots\\
            \end{pmatrix} = \sum_{i=1}^{k} \langle f, \phi_i \rangle_{\mathcal{H}} \phi_i + \sum_{i=k+1}^{\infty} \langle f, \phi_i \rangle_{\mathcal{H}} \phi_i \\
        &=\sum_{i=1}^{\infty} \langle f, \phi_i \rangle_{\mathcal{H}} \phi_i = \ f.
    \end{align*}
    Then we move on to (2), for any $f \in \mathcal{H}$, we have
    \begin{align*}
        &(\mathcal{A}_{\leq k} + \mathcal{A}_{>k}) f 
        = \ (\mathcal{A} f)_{\leq k} + (\mathcal{A} f)_{>k} 
        = \ \mathcal{A} f. \ \ \ \text{\color{gray} (By (1))} \qquad    
    \end{align*}
Finally we prove the statement (3) , this is because
    \begin{align*}
        \tilde{K} = \hat{S}_n \mathcal{A}^2 \Sigma^{\beta - 1}\hat{S}_n^*= \hat{S}_n (\mathcal{A}_{\leq k}^2 \Sigma_{\leq k}^{\beta - 1} + \mathcal{A}_{> k}^2 \Sigma_{> k}^{\beta - 1}) \hat{S}_n^* = \hat{S}_n \mathcal{A}_{\leq k}^2 \Sigma_{\leq k}^{\beta - 1} \hat{S}_n^* + \hat{S}_n  \mathcal{A}_{> k}^2 \Sigma_{> k}^{\beta - 1} \hat{S}_n^* = \tilde{K}_{\leq k} + \tilde{K}_{>k}.
    \end{align*}
\end{proof}

\noindent In the following lemma modified from \cite{barzilai2023generalization}, we give a lemma which is useful for bounding $\hat{f}(y)_{\leq k}$'s norm in bounding bias and variance in \ref{lemma:concentration_variance}, \ref{lemma:intermediate_ub_bias}.

\begin{lemma}
\label{lemma:separate_search}
Denote $\hat f(y):=\mathcal{A} \Sigma^{\beta - 1} \hat{S}_n^* (\tilde{K}^{\gamma})^{-1}  y$ (highlight its dependence on $y$), we have  
$$
\underbrace{\phi_{\leq k} \hat{f}(y)_{\leq k}}_{k\times 1} + \underbrace{\phi_{\leq k} \mathcal{A}_{\leq k} \Sigma^{\beta - 1}_{\leq k} \hat{S}_n^*}_{k\times n} \underbrace{(\tilde{K}_{>k}^{\gamma})^{-1}}_{n\times n} \underbrace{\hat{S}_n \mathcal{A}_{\leq k} \hat{f}(y)_{\leq k}}_{n\times 1} = \underbrace{\phi_{\leq k}\mathcal{A}_{\leq k} \Sigma^{\beta - 1}_{\leq k} \hat{S}_n^* }_{k\times n} \underbrace{(\tilde{K}^{\gamma}_{>k})^{-1}}_{n\times n} \underbrace{y}_{n\times 1},
$$
where $\tilde{K}_{> k}^{\gamma}$ is the regularized version of spectrally transformed matrix, defined as $\hat{S}_n \mathcal{A}_{> k}^2 \Sigma_{> k}^{\beta - 1} \hat{S}^{*}_n + n\gamma_n I$.

% {\color{red} Previous:
% $$
% \hat f_{\le k}(y)+(\Sigma_{AB^{-2}})_{\le k}\phi_{\le k}\hat S_n^\ast A_k^{-1}\hat S_n^\ast\phi_{\le k}^\ast (\Sigma_A)_{\le k} \hat f_{\le k}(y) = (\Sigma_{AB^{-2}})_{\le k}\phi_{\le k}\hat S_n^\ast A_k^{-1}y
% $$    
% where $A_k=\hat S_n^\ast \phi_{\ge k}^\ast (\Sigma_{A^2B^{-2}})_{\ge k} \phi_{\ge k} \hat S_n$
% }
\end{lemma}
\begin{proof}
First we discuss the ridgeless case i.e. $\gamma_n = 0$,
where $\hat{f}$ is the minimum norm solution, then $\hat{f}_{>k}$ is also the minimum norm solution to $\hat{S}_n \mathcal{A}_{>k} \hat{f}_{>k} = y - \hat{S}_n \mathcal{A}_{\leq k} \hat{f}_{\leq k}$, then similar to \ref{eq:objective} we can write
$$
\hat{f}_{>k} = \mathcal{A} \Sigma^{\beta - 1} \hat{S}_n^* (\hat{S}_n \mathcal{A}_{>k}^2 \Sigma_{>k}^{\beta - 1} \hat{S}_n^*)^{-1} (y - \hat{S}_n \mathcal{A}_{\leq k}\hat{f}_{\leq k}).
$$
Therefore,
$$
\phi_{>k} \hat{f}_{>k} = \Lambda^{>k}_{\mathcal{A} \Sigma^{\beta - 1}} \phi_{>k} \hat{S}_n^* (\hat{S}_n \mathcal{A}^2_{>k} \Sigma^{\beta - 1}_{>k} \hat{S}_n^*)^{-1} (y - \hat{S}_n  \phi^* _{\leq k} \Lambda^{\leq k}_{\mathcal{A}} \phi_{\leq k}\hat{f}_{\leq k}).
$$
As such, we obtain min norm interpolator is the the minimizer of following 
\begin{align*}
    \phi \hat{f}(y) =& \arg\min_{\hat{f}_{\leq k}} v(\phi_{\leq k} \hat{f}_{\leq k}) \\
    :=& [( \phi_{\leq k} \hat{f}_{\leq k})^T, (y - \hat{S}_n  \phi^* _{\leq k} \Lambda^{\leq k}_{\mathcal{A}} \phi_{\leq k}\hat{f}_{\leq k})^T  (\hat{S}_n \mathcal{A}^2_{>k} \Sigma^{\beta - 1}_{>k} \hat{S}_n^*)^{-1} (\phi_{>k} \hat{S}_n^*)^T \Lambda^{>k}_{\mathcal{A} \Sigma^{\beta - 1}}].
\end{align*}
The vector $\phi \hat{f}(y)$ gives minimum norm if for any additional vector $\eta_{\leq k} \in \mathbb{R}^k$ we have $v(\phi_{\leq k} \hat{f}_{\leq k}(y)) \perp v(\phi_{\leq k} \hat{f}_{\leq k}(y) + \eta_{\leq k}) - v(\phi_{\leq k} \hat{f}_{\leq k}(y))$ in $\mathcal{H}^{\beta}$ norm.

We first write out the second vector
$$
v(\phi_{\leq k} \hat{f}_{\leq k}(y) + \eta_{\leq k}) - v(\phi_{\leq k} \hat{f}_{\leq k}(y)) = [\eta_{\leq k}^T, -\eta_{\leq k}^T \Lambda^{\leq k}_{\mathcal{A}} (\hat{S}_n  \phi^* _{\leq k})^T (\hat{S}_n \mathcal{A}_{> k}^2 \Sigma_{> k}^{\beta - 1} \hat{S}_n^*)^{-1} (\phi_{>k} \hat{S}_n^*)^T \Lambda^{>k}_{\mathcal{A} \Sigma^{\beta - 1}}].
$$
Then we compute the inner product w.r.t. $\mathcal{H}^{\beta}$ norm, by \ref{lemma:equiv_sobolev_matrix_norm} we have:
\begin{align*}
&\eta_{\leq k}^T \Lambda^{\leq k}_{\Sigma^{1 - \beta}} (\phi_{\leq k} \hat{f}_{\leq k}) \\
-& \eta_{\leq k}^T \Lambda^{\leq k}_{\mathcal{A}} (\hat{S}_n  \phi^* _{\leq k})^T \underbrace{(\hat{S}_n \mathcal{A}_{>k}^2 \Sigma^{\beta - 1}_{>k} \hat{S}_n^*)^{-1}}_{(1)}  \underbrace{(\phi_{>k} \hat{S}_n^*)^T \Lambda^{>k}_{\mathcal{A} \Sigma^{\beta - 1}} \Lambda^{> k}_{\Sigma^{1 - \beta}} \Lambda^{>k}_{\mathcal{A} \Sigma^{\beta - 1}} (\phi_{>k} \hat{S}_n^*)}_{(2)} \\
& (\hat{S}_n \mathcal{A}^2_{>k} \Sigma^{\beta - 1}_{>k} \hat{S}_n^*)^{-1} (y - \hat{S}_n  \phi^* _{\leq k} \Lambda^{\leq k}_{\mathcal{A}} \phi_{\leq k}\hat{f}_{\leq k}) = 0.
\end{align*}
Note that (1) and (2) cancel out, and since the equality above holds for any $\eta_{\leq k}$, we have:
$$
 \Lambda^{\leq k}_{\Sigma^{1 - \beta}} (\phi_{\leq k} \hat{f}_{\leq k}) - \Lambda^{\leq k}_{\mathcal{A}} (\hat{S}_n  \phi^* _{\leq k})^T (\hat{S}_n \mathcal{A}_{>k}^2 \Sigma^{\beta - 1}_{>k} \hat{S}_n^*)^{-1} (y - \hat{S}_n  \phi^* _{\leq k} \Lambda^{\leq k}_{\mathcal{A}} \phi_{\leq k}\hat{f}_{\leq k}) = 0.
$$
Therefore,
$$
\phi_{\leq k} \hat{f}_{\leq k} - \Lambda^{\leq k}_{\mathcal{A} \Sigma^{\beta - 1}} \phi_{\leq k} \hat{S}_n^* (\tilde{K}^{\gamma}_{>k})^{-1} (y - \hat{S}_n \mathcal{A} \hat{f}_{\leq k}) = 0.
$$
With some simple algebraic manipulation we can obtain the required identity
$$
\phi_{\leq k} \hat{f}_{\leq k} + \phi_{\leq k} \mathcal{A}_{\leq k} \Sigma_{\leq k}^{\beta - 1} \hat{S}_n^* (\tilde{K}^{\gamma}_{>k})^{-1} \hat{S}_n \mathcal{A} \hat{f}_{\leq k}  = \phi_{\leq k} \mathcal{A}_{\leq k} \Sigma_{\leq k}^{\beta - 1}  \hat{S}_n^* (\tilde{K}^{\gamma}_{>k})^{-1} y.
$$
This finishes our discussion on ridgeless case.

For the regularized case i.e. $\gamma_n > 0$, first we prove 
$$
 \hat{f}(y)_{\leq k} +  \mathcal{A}_{\leq k} \Sigma^{\beta - 1}_{\leq k} \hat{S}_n^* (\tilde{K}_{>k}^{\gamma})^{-1} \hat{S}_n \mathcal{A}_{\leq k} \hat{f}(y)_{\leq k} = \mathcal{A}_{\leq k} \Sigma^{\beta - 1}_{\leq k} \hat{S}_n^* (\tilde{K}^{\gamma}_{>k})^{-1} y.
$$
We know by \ref{lemma:decomposition} $\tilde{K}^{\gamma} = \tilde{K} + n\gamma I = (\tilde{K}_{>k} + n\gamma I) + \tilde{K}_{\leq k} = \tilde{K}_{>k}^{\gamma} + \tilde{K}_{\leq k}$, we split $\tilde{K}^{\gamma}$ into two parts: $\tilde{K}_{>k}^{\gamma}$ and $\tilde{K}_{\leq k}$.
Accordingly, $ \hat{f}(y)_{\leq k}$ can be represented as
    \begin{align*}
        \hat{f}(y)_{\leq k} &= \phi_{\leq k}^* \phi_{\leq k} \hat{f}(y) = \phi_{\leq k}^* \phi_{\leq k} \mathcal{A} \Sigma^{\beta - 1} \hat{S}_n^* (\tilde{K}^{\gamma})^{-1}  y \\
        &= \mathcal{A}_{\leq k} \Sigma_{\leq k}^{\beta - 1} \hat{S}_n^* (\tilde{K}_{>k}^{\gamma} +  \tilde{K}_{\leq k})^{-1}  y  \,.
    \end{align*}
    %\yplu{I know the last line is right, but we still need to explain the reason}
    Therefore, taking it back to LHS, we have
    \begin{align*}
        & \hat{f}(y)_{\leq k} + \mathcal{A}_{\leq k} \Sigma_{\leq k}^{\beta - 1} \hat{S}_n^* (\tilde{K}_{>k}^{\gamma})^{-1} \hat{S}_n \mathcal{A}_{\leq k} \hat{f}(y)_{\leq k} \text{ (LHS)}\\
        =& \ \mathcal{A}_{\leq k} \Sigma_{\leq k}^{\beta - 1} \hat{S}_n^* (\tilde{K}_{>k}^{\gamma} +  \tilde{K}_{\leq k})^{-1}  y \\
        +& \ \mathcal{A}_{\leq k} \Sigma_{\leq k}^{\beta - 1} \hat{S}_n^* 
        (\tilde{K}_{>k}^{\gamma})^{-1}\underbrace{ \hat{S}_n \mathcal{A}_{\leq k} \mathcal{A}_{\leq k} \Sigma_{\leq k}^{\beta - 1} \hat{S}_n^*}_{\text{equals to } \tilde{K}_{\leq k}} (\tilde{K}_{>k}^{\gamma} +  \tilde{K}_{\leq k})^{-1}  y \qquad \text{  (Expand $\hat{f}(y)_{\leq k}$)}\\
        =& \ \mathcal{A}_{\leq k} \Sigma^{\beta - 1}_{\leq k} \hat{S}^{*}_{n} (\tilde{K}_{>k}^{\gamma})^{-1} (\tilde{K}_{>k}^{\gamma} + \tilde{K}_{\leq k})  (\tilde{K}_{>k}^{\gamma} + \tilde{K}_{\leq k})^{-1} y \\
        =& \ \mathcal{A}_{\leq k} \Sigma^{\beta - 1}_{\leq k} \hat{S}_n^* (\tilde{K}^{\gamma}_{>k})^{-1}y \text{   (RHS)}\,.
    \end{align*}
    We project LHS and RHS back to $\mathbb{R}^k$ for convenient usage in \ref{lemma:concentration_variance}, \ref{lemma:intermediate_ub_bias}, we project the functions in $\mathcal{H}$ back to $\mathbb{R}^k$ so we use $\phi_k$ in both two sides and we obtain
    $$
    \phi_{\leq k} \hat{f}(y)_{\leq k} + \phi_{\leq k} \mathcal{A}_{\leq k} \Sigma^{\beta - 1}_{\leq k} \hat{S}_n^* (\tilde{K}_{>k}^{\gamma})^{-1} \hat{S}_n \mathcal{A}_{\leq k} \hat{f}(y)_{\leq k} = \phi_{\leq k}\mathcal{A}_{\leq k} \Sigma^{\beta - 1}_{\leq k} \hat{S}_n^* (\tilde{K}^{\gamma}_{>k})^{-1} y,
    $$
    which concludes the proof.
\end{proof}

\noindent This lemma justifies we can switch between using Sobolev norm and matrix norm by using $\phi$.
\begin{lemma}[Equivalence between Sobolev norm and Matrix norm]
    \label{lemma:equiv_sobolev_matrix_norm}
    For any function $f \in \mathcal{H}^{\beta'}$, we have
    $$
    \|f\|^2_{\mathcal{H}^{\beta'}} = \|\phi f\|^2_{\Lambda_{\Sigma^{1 - \beta'}}}.
    $$
    And additionally, $
    \|f_{\leq k}\|^2_{\mathcal{H}^{\beta'}} = \|\phi_{\leq k} f_{\leq k}\|^2_{\Lambda^{\leq k}_{\Sigma^{1 - \beta'}}}
    $, $
    \|f_{> k}\|^2_{\mathcal{H}^{\beta'}} = \|\phi_{> k} f_{> k}\|^2_{\Lambda^{> k}_{\Sigma^{1 - \beta'}}}.
    $
\end{lemma}
\begin{proof}
According to the definition of Sobolev norm, we have
    \begin{align*}
        \text{LHS} &= \| \Sigma^{\frac{1 - \beta'}{2}} f \|^2_{\mathcal{H}} \\
        &= \| \phi \Sigma^{(1 - \beta') / 2} f \|^2  \qquad\text{\color{gray} (by isometry i.e. $\|f\|_{\mathcal{H}} = \|\phi f\|^2$)}\\
        &= \| \Lambda_{\Sigma^{(1 - \beta') / 2}} \phi f \|^2 \qquad \text{ \color{gray}(by $\phi \phi^* = id: \mathbb{R}^{\infty} \to \mathbb{R}^{\infty}$)}\\
        &= \| \phi f\|_{\Lambda_{\Sigma^{1 - \beta'}}}^2 = \text{RHS}.
    \end{align*}

    Then for the $\leq k$ case, we have
    $$\|f_{\leq k}\|_{\mathcal{H}^{\beta'}} = \|\phi f_{\leq k}\|^2_{\Lambda_{\Sigma^{1 - \beta'}}}$$
    Since $(\phi f_{\leq k})_{\leq k} = \phi_{\leq k} f_{\leq k}$, all its $>k$ entries are zero, then 
    %\wendao{$\phi f_{>k}\neq\phi_{>k}f_{>k}$, $\phi f_{\leq k}\neq\phi_{\leq k}f_{\leq k}$, we should explain clearly here if we want to use them without distinction} \honam{They are the same, yes we should explain it probably at the start of the Appendix}
    $$\|\phi f_{\leq k}\|^2_{\Lambda_{\Sigma^{1 - \beta'}}} = (\phi f_{\leq k})^T \Lambda_{\Sigma^{1 - \beta'}} (\phi f_{\leq k}) =  (\phi f_{\leq k})^T \Lambda^{\leq k}_{\Sigma^{1 - \beta'}} (\phi f_{\leq k}) = \|\phi_{\leq k} f_{\leq k}\|^2_{\Lambda^{\leq k}_{\Sigma^{1 - \beta'}}}.$$

    The proof above works similarly for the $>k$ case.
\end{proof}

\begin{lemma}[Separation of $<k$ and $>k$ case]
    \label{lemma:separate_smaller_k_large_eq_k}
    For any function $f \in \mathcal{H}^{\beta'}$, then 
    $$
    \|f\|^2_{\mathcal{H}^{\beta'}} = \|f_{\leq k}\|^2_{\mathcal{H}^{\beta'}} + \|f_{>k}\|^2_{\mathcal{H}^{\beta'}}.
    $$
\end{lemma}
\begin{proof}
    \begin{align*}
        \|f\|^2_{\mathcal{H}^{\beta'}} &= \|\phi \Sigma^{(1 - \beta') / 2} f\|^2 \\ 
        &= \ \sum_{i = 1}^{\infty}{[\phi \Sigma^{(1 - \beta') / 2} f]_i^2} =\sum_{i = 1}^{k}{[\phi \Sigma^{(1 - \beta') / 2} f]_i^2} + \sum_{i = k+1}^{\infty}{[\phi \Sigma^{(1 - \beta') / 2} f]_i^2} \\
        &= \|\phi_{\leq k} \Sigma_{\leq k}^{(1 - \beta') / 2} f_{\leq k}\|^2 + \|\phi_{>k} \Sigma_{>k}^{(1 - \beta') / 2} f_{>k}\|^2 \\
        &= \|f_{\leq k}\|^2_{\mathcal{H}^{\beta'}} + \|f_{>k}\|^2_{\mathcal{H}^{\beta'}}.
    \end{align*}
\end{proof}

%\honam{We can use this lemma as a motivation example when introducing $\psi$ i.e. to simplify computation}
\begin{lemma}
    \label{lemma:psi_expectation_one}
    $
    \mathbb{E}([\hat{S}_n \psi^*_{>k}]_{ji}^2) = 1
    $ holds for any $i > k, j \in [n]$.
\end{lemma}
\begin{proof}
    \begin{align*}
        \mathbb{E}([\hat{S}_n \psi^*_{>k}]_{ji}^2) = \mathbb{E}([\langle \psi_i, K_{x_j} \rangle_{\mathcal{H}}^2]) = \mathbb{E}(\psi_i(x_j)^2) = 1.
    \end{align*}
\end{proof}

Last we present a lemma which is useful in $>k$ case in deriving bias's bound.
\begin{lemma}
    \label{lemma:extension_sherman_morrison_woodbury}
    $$
    (A + UCV)^{-1} U = A^{-1} U (I + CVA^{-1}U)^{-1}.
    $$
\end{lemma}
\begin{proof}
    By Sherman-Morrison-Woodbury formula we have
    $$
    (A+UCV)^{-1} = A^{-1} - A^{-1} U (C^{-1} + VA^{-1} U )^{-1} VA^{-1}
    $$
    Therefore,
    \begin{align*}
        (A+UCV)^{-1} U =& \ A^{-1} U - A^{-1} U (C^{-1} + VA^{-1} U )^{-1} VA^{-1} U \\
        =& \ A^{-1} U (I - (C^{-1} + VA^{-1} U )^{-1} VA^{-1} U) \\
        =& \ A^{-1} U (I - (C^{-1} + VA^{-1} U )^{-1} (C^{-1} + VA^{-1} U) + (C^{-1} + VA^{-1} U )^{-1} C^{-1}) \\
        =& \ A^{-1} U (I - I + (C (C^{-1} + VA^{-1}U))^{-1}) \\
        =& A^{-1} U (I + CVA^{-1}U)^{-1}.
    \end{align*}
    
\end{proof}

\section{Concentration Lemmas}
Here we present several lemmas for bounding several quantities in \ref{appendix:ub_variance}, \ref{appendix:ub_bias}.

\begin{lemma}
    \label{lemma:concentration_1}
    Let $k \in [n]$, $a$ be the power of $\mathcal{A}$, and $b$ be the power of $\Sigma$, we bound the trace of this $n \times n$ matrix, w.p. at least $\displaystyle 1 - 2\exp(-\frac{1}{2\beta_k^2} n)$  we have
    $$
    \frac{1}{2} n\sum_{i>k} p_i^{a} \lambda_i^{b} \leq \trace(\hat{S}_n \psi^*_{>k} \Lambda^{>k}_{\mathcal{A}^{a} \Sigma^{b}} \psi_{>k} \hat{S}^*_n) \leq \frac{3}{2} n\sum_{i>k} p_i^{a} \lambda_i^{b}.
    $$
\end{lemma}
\begin{proof}
    Note that $\Lambda^{>k}_{\mathcal{A}^{a} \Sigma^{b}}$ is a diagonal matrix with entry $p_i^{a} \lambda_i^{b}$ ($i>k$).
    \begin{align*}
        &\trace(\hat{S}_n \psi^*_{>k} \Lambda^{>k}_{\mathcal{A}^{a} \Sigma^{b}} \psi_{>k} \hat{S}^*_n)
        = \sum_{j=1}^{n} \ [(\hat{S}_n \psi^*_{>k})(\Lambda^{>k}_{\mathcal{A}^{a} \Sigma^{b}})(\psi_{>k} \hat{S}^*_n))]_{jj} = \sum_{j=1}^{n} \underbrace{\sum_{i=k+1}^{\infty}{p_i^{a} \lambda_i^{b} [\hat{S}_n \psi_{>k}^*]_{ji}}^2}_{v_j}.
    \end{align*}
    Here we denote the term inside $j$ summation as $v_j$, then by \ref{lemma:psi_expectation_one},  the expectation of the trace is 
    $$n\sum_{i>k} p_i^{a} \lambda_i^{b}.$$ 
    We also know that $v_j$ is lower bounded by $0$ and by def. of $\beta_k$ \ref{assumption:well_behaveness_feature}, it can be upper bounded by 
    \begin{align*}
        v_j =& \sum_{i = k+1}^{\infty}{p_i^{a} \lambda_i^{b} \psi_i(x_j)^2} 
        \leq \ \underbrace{\beta_k \sum_{i = k+1}^{\infty}{p_i^{a} \lambda_i^{b}}}_{\text{denoted as $M$}}.
    \end{align*}

    %\honam{We should stick to the original definition and should not put some of  $\mathcal{A}^{a}$ and $\Sigma^{b}$ inside $\beta_k$}, we can writ it as follows:

    % \yplu{why don't you $\sum_{i = k+1}^{\infty}{p_i^{a} \lambda_i^{b - 1} \lambda_i \psi_i(x_j)^2} 
    %     \leq  \sum_{i=k+1}^{\infty} \ p_{i}^{a} \lambda_{i}^{b} \psi_i(x_j)^2\le  \sum_{i=k+1}^{\infty} \ p_{i}^{a} \lambda_{i}^{b}$ The last term is ensentially just a geometric serires sum.
    % }

    \noindent Then we have $0 \leq v_j \leq M$ for all $j$ and $v_j$ is independent, we can apply the Hoeffding's inequality to bound $\sum_{j=1}^n v_j$:%\yplu{how}\\
    $$
    \mathbb{P}(| \sum_{j=1}^{n} v_j -  n\sum_{i>k} p_i^{a} \lambda_i^{b}| \geq t) \leq 2\exp \left(\frac{-2t^2}{nM^2} \right).
    $$
    We then pick $t := \frac{n}{2} \sum_{i>k}{p_i^{a} \lambda_i^{b}}$, and we get $\displaystyle\frac{-2t^2}{nM^2} = -\frac{1}{2\beta_k^2} n$, and we know the trace value exactly corresponds to $\sum_{j=1}^{n} v_j$.

    Therefore, w.p.at least $\displaystyle 1 - 2\exp(-\frac{1}{2\beta_k^2} n)$,
        $$
    \frac{1}{2} n\sum_{i>k} p_i^{a} \lambda_i^{b} \leq \trace(\hat{S}_n \psi^*_{>k} \Lambda^{>k}_{\mathcal{A}^{a} \Sigma^{b}} \psi_{>k} \hat{S}^*_n) \leq \frac{3}{2} n\sum_{i>k} p_i^{a} \lambda_i^{b}.
    $$
    % \begin{align*}
    %     &\frac{-2t^2}{nM^2}\\
    %     =& \ \frac{-2(\frac{n}{2})^2 (\sum_{i=k+1}^{\infty} p_i^{a} \lambda_i^{b})^2}{n \beta_k^2 p_{k+1}^{2a} \lambda_{k+1}^{2(b - 1)} \beta_k^2 (\sum_{i=k+1}^{\infty}{\lambda_i})^2} \\
    %     =& \ -\frac{n}{2\beta_k^2} \frac{(\sum_{i=k+1}^{\infty} p_i^{a} \lambda_i^{b})^2}{p_{k+1}^{2a} \lambda_{k+1}^{2(b - 1)} (\sum_{i=k+1}^{\infty} \lambda_i)^2} \\
    %     =& \ -\frac{n}{2\beta_k^2} (\frac{\sum_{i=k+1}^{\infty} p_i^{a} \lambda_i^{b}}{\sum_{i=k+1}^{\infty}{\lambda_i}} \cdot \frac{1}{p_{k+1}^{a} \lambda_{k+1}^{b-1}})^2\\
    %     \geq& \ -\frac{n}{2\beta_k^2} \left(\frac{\sum_{i=k+1}^{\infty}{\lambda_i \cdot (p_{k+1}^{a} \lambda_{k+1}^{b-1}})}{\sum_{i=k+1}^{\infty} \lambda_i} \cdot \frac{1}{p_{k+1}^{a} \lambda_{k+1}^{b-1}} \right)^2 \\
    %     =& \ -\frac{n}{2\beta_k^2} 
    % \end{align*}
    % Then $$\displaystyle 1 - 2\exp \left(\frac{-2t^2}{nM^2} \right) \leq 1 - 2\exp(\frac{1}{2\beta_k^2}n)$$
    % \honam{We probably want $\geq$ instead}
    % 1 -\frac{1}{2\beta_k^2}n
    
    % w.p. at least $\displaystyle 1 - 2\exp \left(\frac{-2t^2}{nM^2} \right)$, we have 
    % $$
    % \frac{1}{2} n\sum_{i>k} p_i^{a} \lambda_i^{b} \leq \trace(\hat{S}_n \psi^*_{>k} \Lambda^{>k}_{\mathcal{A}^{a} \Sigma^{b}} \psi_{>k} \hat{S}^*_n) \leq \frac{3}{2} n\sum_{i>k} p_i^{a} \lambda_i^{b}
    % $$
    % \yplu{your bound here looks right}
\end{proof}

Here we present the modified version of Lemma 2 in \cite{barzilai2023generalization}, we rewrite it to fit into our framework for completeness.
\begin{lemma}
    \label{lemma:bound_psi}
    For any $k \in [n]$ there exists some absolute constant $c', c_2 > 0$ s.t. the following hold simultaneously w.p. at least $1 - 2\exp(-\frac{c'}{\beta_k} \max \{ \frac{n}{k}, \log(k)\})$
    \begin{enumerate}
        \item $\mu_k( \underbrace{\psi_{\leq k} \hat{S}^*_n \hat{S}_n  \psi^*_{\leq k}}_{k \times k}) \geq \max \{ \sqrt{n} - \sqrt{\frac{1}{2} \max \{ n, \beta_k (1 + \frac{1}{c'} k \log(k))\}} , 0\}^2$;
        \item $\mu_1(\underbrace{\psi_{\leq k} \hat{S}^*_n \hat{S}_n  \psi^*_{\leq k}}_{k \times k}) \leq c_2 \max \{ n, \beta_k k \log(k) \}$.
    \end{enumerate}
    Moreover, there exists some $c > 0$ s.t. if $c\beta_k k \log(k) \leq n$ then w.p. at least $1 - 2\exp(-\frac{c'}{\beta_k} \frac{n}{k})$ and some absolute constant $c_1 > 0$ it holds that 
    $$
    c_1 n \leq \mu_k(\psi_{\leq k} \hat{S}^*_n \hat{S}_n  \psi^*_{\leq k}) \leq \mu_1(\psi_{\leq k} \hat{S}^*_n \hat{S}_n  \psi^*_{\leq k}) \leq c_2 n.
    $$
\end{lemma}
\begin{proof}
    We will bound the singular values $\sigma_i(\underbrace{\hat{S}_n \psi^*_{\leq k}}_{n \times k} )$ since $\sigma_i(A)^2 = \mu_i(A^T A)$ for any matrix $A$.
    %\honam{2024/04/29: We should modified version of Ostrowski's theorem to give a bound \ref{lemma:ostrowski}}

    We know rows of this matrix are independent isotropic random vectors in $\mathbb{R}^k$, where randomness is over the choice of $x$, where by the definition of $\beta_k$ \ref{assumption:well_behaveness_feature} the rows are heavy-tailed having norm bounded by
    \begin{equation*}
        \| \mbox{each row of~} \hat{S}_n \psi^\ast_{\le k} \|  \leq \sqrt{k\beta_k}.  
    \end{equation*}

    Here we can use \cite{vershynin2011introduction}[Theorem 5.41] which is applicable for heavy-tailed rows, there is some absolute constant $c' > 0$ s.t. for every $t \geq 0$, one has that w.p. at least $1 - 2k \exp(-2c't^2)$
    $$
    \sqrt{n} - t \sqrt{k \beta_k} \leq \sigma_k(\hat{S}_n \psi^*_{\leq k} ) \leq \sigma_1(\hat{S}_n \psi^*_{\leq k} )  \leq \sqrt{n} + t \sqrt{k \beta_k}. 
    $$
    We pick $t = \sqrt{\frac{1}{2\beta_k} \max \{ \frac{n}{k}, \log(k)\} + \frac{\log(k)}{2c'}}$, then w.p. at least $1 - 2\exp(\frac{-c'}{\beta_k} \max \{ \frac{n}{k}, \log(k)\})$ it holds that
    
    \begin{align*}
    \sigma_{1}\left(\hat{S}_n \psi^*_{\leq k}\right)^{2} & \leq\left(\sqrt{n}+\sqrt{\frac{1}{2} \max (n, k \log (k))+k \log (k) \frac{\beta_{k}}{2 c^{\prime}}}\right)^{2} \\
    & \leq\left(\sqrt{n}+\frac{1}{\sqrt{2}} \sqrt{n+\left(1+\frac{\beta_{k}}{c^{\prime}}\right) k \log (k)}\right)^{2} \\
    & \leq 3 n+\left(1+\frac{\beta_{k}}{c^{\prime}}\right) k \log (k),
    \end{align*}
    where the last inequality followed from the fact that $(a+b)^2 \leq 2(a^2 + b^2)$ for any $a, b \in \mathbb{R}$. Since $\beta_k \geq 1$ \ref{assumption:well_behaveness_feature}, we obtain $\sigma_{1}\left(\hat{S}_n \psi^*_{\leq k}\right)^{2} \leq c_2 \max \{n, \beta_k k \log (k) \}$ for a suitable $c_2 > 0$, proving (2).

    For the lower bound, we simultaneously have 
    \begin{align*}
\sigma_{k}\left(\hat{S}_n \psi^*_{\leq k}\right) & \geq \sqrt{n}-\frac{1}{\sqrt{2}} \sqrt{\frac{1}{2} \max (n, k \log (k))+k \log (k) \frac{\beta_{k}}{2 c^{\prime}}} \\
& \geq \sqrt{n}-\sqrt{\frac{1}{2} \max \left(n, \beta_{k}\left(1+\frac{1}{c^{\prime}}\right) k \log (k)\right)}.
\end{align*}
    Since the singular values are non-negative, the above implies 
    $$
        \sigma_{k}\left(\hat{S}_n \psi^*_{\leq k}\right) \geq \max \{ \sqrt{n}-\sqrt{\frac{1}{2} \max \left(n, \beta_{k}\left(1+\frac{1}{c^{\prime}}\right) k \log (k)\right)}, 0\}^2
    $$
    which proves (1).

    Next we move on to prove the moreover part, taking $c = (1 + \frac{1}{c'})$ we now have by assumption that $\frac{n}{k} \geq c\beta_k \log(k) \geq \log(k)$ (where we used the fact that $c \geq 1$ and $\beta_k \geq 1$), the probability that (1) and (2) hold is $1 - 2\exp(-\frac{c'}{\beta_k} \frac{n}{k})$.
    Furthermore, plugging $c \beta_k k \log(k) \leq n$ into the lower bound (1) obtains the following
    \begin{align*}
\mu_{k}\left(\psi_{\leq k} \hat{S}_n^* \hat{S}_n \psi_{\leq k}^* \right) & \geq \max \left(\sqrt{n}-\sqrt{\frac{1}{2} \max \left(n, c \beta_{k} k \log (k)\right)}, 0\right)^{2} \ \\
& \geq\left(\sqrt{n}-\sqrt{\frac{n}{2}}\right)^{2}=\left(1-\frac{1}{\sqrt{2}}\right)^{2} n .
\end{align*}
    Similarly since $\beta_k k \log(k) \leq n$, the upper bound (2) becomes 
    $$
    \mu_{1}\left(\psi_{\leq k} \hat{S}_n^* \hat{S}_n \psi_{\leq k}^* \right) \leq c_2 n.
    $$
    
\end{proof}

\begin{lemma}
\label{lemma:simultaneous_concentration}
   There exists some constant $c, c', c_1, c_2 > 0$ s.t. for any $k \in \mathbb{N}$ with $c\beta_k k \log(k) \leq n$, it holds w.p. at least $1 - 8\exp(-\frac{c'}{\beta_k^2} \frac{n}{k})$, the following hold simultaneously
   \begin{enumerate}
       \item $c_1 n \sum_{i>k} p_i^{-2} \lambda_i^{-\beta'} \leq \trace(\hat{S}_n \psi_{\leq k}^* \Lambda^{\leq k}_{\mathcal{A}^{-2} \Sigma^{-\beta'}} \psi_{\leq k} \hat{S}_n^*) \leq c_2 n \sum_{i>k}  p_i^{-2} \lambda_i^{-\beta'}$;
       \item $c_1 n \sum_{i>k} p_i^2 \lambda_i^{-\beta' + 2\beta} \trace(\hat{S}_n \psi_{\leq k}^* \Lambda^{\leq k}_{\mathcal{A}^{2} \Sigma^{-\beta' + 2\beta}} \psi_{\leq k} \hat{S}_n^*) \leq c_2 n \sum_{i>k} p_i^2 \lambda_i^{-\beta' + 2\beta}$;
       \item $\mu_k(\psi_{\leq k} \hat{S}^*_n \hat{S}_n  \psi^*_{\leq k}) \geq c_1 n$;
       \item $\mu_1(\psi_{\leq k} \hat{S}^*_n \hat{S}_n  \psi^*_{\leq k}) \leq c_2 n$.
   \end{enumerate}
\end{lemma}
\begin{proof}
    By Lemma \ref{lemma:concentration_1}, (1) and (2) each hold w.p. at least $1 - 2\exp(-\frac{1}{2\beta_k^2} n)$, so the probability of they both hold is at least $(1 - 2\exp(-\frac{1}{2\beta_k^2} n))^2$.
    And by Lemma \ref{lemma:bound_psi}, (3), (4) simultaneously holds with probability at least $1 - 2\exp(-\frac{c'}{\beta_k} \frac{n}{k})$. Therefore, the probability of all four statements hold is at least
    \begin{align*}
        &(1 - 2\exp(-\frac{1}{2\beta_k^2} n))^2 (1 - 2\exp(-\frac{c'}{\beta_k} \frac{n}{k})) \\
        \geq& \ 1 - 8\exp(-\min\{ \frac{1}{2\beta_k^2} n, \frac{c'}{\beta_k} \frac{n}{k}\}) \\
        \geq& \ 1 - 8\exp\{-\min(\frac{1}{2\beta_k^2}, \frac{c'}{\beta_k}\} \frac{n}{k}).
    \end{align*}
    Since we know $\beta_k \geq 1$ \ref{assumption:well_behaveness_feature}, then we replace $c'$ with $\min \{ \frac{1}{2}, c' \}$ results in the desired bound holding w.p. at least $1 - 8\exp(-\frac{c'}{\beta_k^2} \frac{n}{k})$.
    
\end{proof}

\begin{lemma}[Concentration bounds on $\| \hat{S}_n \mathcal{A}_{>k} f^*_{>k}\|^2$ in \ref{lemma:intermediate_ub_bias}]
    \label{lemma:SnAf}
    For any $k \in [n]$ and $\delta > 0$, it holds w.p. at least $1 - \delta$ that
    $$
    \| \hat{S}_n \mathcal{A}_{>k} f^*_{>k}\|^2 \leq \frac{1}{\delta} n \|\phi_{>k} \mathcal{A}_{>k} f^*_{>k} \|_{\Sigma_{>k}}^2.
    $$
\end{lemma}
\begin{proof}
    Let $v_j := \langle \mathcal{A}_{>k} f^*_{>k}, K_{x_j}  \rangle_{\mathcal{H}}^2$, then LHS is equal to $\sum_{j=1}^{n} v_j$. Since $x_j$ is independent, it holds that $v_j$ are independent random variables with mean
    \begin{align*}
        \mathbb{E}[v_j] &= \mathbb{E}[\langle \phi^*_{>k} \phi_{>k} \mathcal{A}_{>k} f^*_{>k}, \sum_{i=1}^{\infty}{\phi_i(x_j) \phi_i }\rangle_{\mathcal{H}}^2] \\
        &= \mathbb{E}[\langle \sum_{i=k+1}^{\infty}{[\phi_{>k} \mathcal{A}_{>k} f^*_{>k}]_i \phi_i }, \sum_{i=1}^{\infty} \phi_i(x_j) \phi_i \rangle_{\mathcal{H}}^2] \\
        &= \mathbb{E}[(\sum_{i = k+1}^{\infty}{[\phi_{>k} \mathcal{A}_{>k} f^*_{>k}]_i \phi_i(x_j)})^2] \\
        &= \sum_{i>k} \sum_{l>k} \sqrt{\lambda_i} \sqrt{\lambda_l} [\phi_{>k} \mathcal{A}_{>k} f^*_{>k}]_i [\phi_{>k} \mathcal{A}_{>k} f^*_{>k}]_l \underbrace{\mathbb{E}_{x_j}{\psi_i(x_j) \psi_l(x_j)}}_{=1 \text{ if } i = l; 0 \text{ otherwise}} \\
        &= \sum_{i>k} \lambda_i [\phi_{>k} \mathcal{A}_{>k} f^*_{>k}]_i^2 \\
        &= \|\phi_{>k} \mathcal{A}_{>k} f^*_{>k} \|^2_{\Lambda^{>k}_{\Sigma}}.
    \end{align*}
    Then we can apply Markov's inequality:
    $$
    \mathbb{P}(\sum_{j = 1}^{n} v_j \geq \frac{1}{\delta} n \|\phi_{>k} \mathcal{A}_{>k} f^*_{>k} \|_{\Sigma_{>k}}^2) \leq \delta.
    $$
\end{proof}

\section{Bounds on Eigenvalues}
\begin{theorem} 
    \label{theorem:eigenspectrum_proof}
 Suppose Assumption \ref{assumption:beta_regularity_condition} holds, and eigenvalues of $\tilde{\Sigma}$ are given in non-increasing order (i.e. $2p + \beta \lambda > 0$).
    There exists absolute constant $c, C,c_1,c_2>0$ s.t. for any  $k\leq k'\in[n]$ and $\delta>0$, it holds w.p. at least $1-\delta-4\frac{r_k}{k^4}\exp(-\frac{c}{\beta_k}\frac{n}{r_k})-2\exp(-\frac{c}{\beta_k}\max\left(\frac{n}{k},\log(k)\right))$  that
    $$\mu_{k}\left(\frac{1}{n} \tilde{K} \right) \leq c_1 \beta_{k}\left(\left(1+\frac{k \log (k)}{n}\right) \lambda_{k}^{\beta} p_k^2 +\log (k+1) \frac{\operatorname{tr}\left(\tilde{\Sigma}_{>k} \right)}{n}\right)$$
    $$ \mu_{k}\left(\frac{1}{n} \tilde{K} \right) \geq c_2 \mathbb{I}_{k, n} \lambda_{k}^{\beta} p_k^2 +\alpha_{k}\left(1-\frac{1}{\delta} \sqrt{\frac{n^{2}}{\trace(\tilde{\Sigma}_{>k'})^2/\trace(\tilde{\Sigma}^2_{>k'})}}\right) \frac{\operatorname{tr}\left( \tilde{\Sigma}_{>k'} \right)}{n}, $$
    where $\mu_k$ is the k-th largest eigenvalue of $\tilde{K}$, $\tilde{\Sigma} := \mathcal{A}^2 \Sigma^{\beta}$, $r_k := \trace(\tilde{\Sigma}_{>k})/(p_{k+1}^2\lambda_{k+1}^\beta)$, and $\mathbb{I}_{k,n}=\begin{cases}
			1, & \text{if } C\beta_kk\log(k)\leq n\\
            0, & \text{otherwise}
		 \end{cases}$.

\end{theorem}
\begin{proof}
    We hereby give the proof of Theorem \ref{theorem:eigenspectrum}. From Lemma \ref{lemma:symmetric_bound_1}, we have that
    $$
    \lambda^{\beta}_{i + k - \min (n,k)}  p^2_{i + k - \min (n,k)}  \mu_{\min(n, k)}(D_k) + \mu_n(\frac{1}{n} \tilde{K}_{>k}) \leq \mu_i(\frac{1}{n} \tilde{K}) \leq \lambda_i^{\beta} p_i^2    \mu_{1}(D_k) + \mu_1(\frac{1}{n} \tilde{K}_{>k}),
    $$
    where $D_k$ is as defined in the lemma.\\
    \indent We bound the two terms at the RHS seperately. From Lemma \ref{lemma:smallest_largest_eigen_kernel}, it holds w.p. at least $1 - 4\frac{r_k}{k^4}\exp(-\frac{c'}{\beta_k}\frac{n}{r_k})$ that for some absolute constants $c',c_1'>0$,
    $$
    \mu_1(\frac{1}{n}\tilde{K}_{>k})  \leq c_1'\left(p_{k+1}^2 \lambda_{k+1}^{\beta} + \beta_k \log(k+1) \frac{\trace(\tilde{\Sigma}_{>k})}{n}\right).
    $$
    For the other term, because $\mu_i(D_k)=\mu_i(\frac{1}{n} (\hat{S}_n \Sigma^{- 1/2}_{\leq k})(\hat{S}_n \Sigma^{- 1/2}_{\leq k})^T)= \mu_i(\frac{1}{n}\psi_{\leq k} \hat{S}^*_n \hat{S}_n  \psi^*_{\leq k})$, by \ref{lemma:bound_psi} ther exists some absolute constants $c'', c_1''>0$, s.t. w.p. at least $1 - 2\exp(-\frac{c''}{\beta_k} \max \{ \frac{n}{k}, \log(k)\})$
    $$
    \lambda_i^{\beta} p_i^2\mu_1(D_k)\leq c_1''\frac{1}{n} \max \{ n, \beta_k k \log(k) \}\lambda_i^{\beta} p_i^2 \leq c_1''\beta_k\left(1 + \frac{k\log(k)}{n}  \right)\lambda_i^{\beta} p_i^2,
    $$
    where the last inequality uses the fact that $\beta_k\geq 1$.\\
    \indent Therefore, by taking $c = \max(c',c'')$, both events hold w.p. at least $1-\delta-4\frac{r_k}{k^4}\exp(-\frac{c}{\beta_k}\frac{n}{r_k})-2\exp(-\frac{c}{\beta_k}\max\left(\frac{n}{k},\log(k)\right))$ and the upper bound of $\mu_i(\frac{1}{n} \tilde{K})$ now becomes
    $$
    \mu_{k}\left(\frac{1}{n} \tilde{K} \right) \leq c_1 \beta_{k}\left(\left(1+\frac{k \log (k)}{n}\right) \lambda_{k}^{\beta} p_k^2 +\log (k+1) \frac{\operatorname{tr}\left(\tilde{\Sigma}_{>k} \right)}{n}\right)
    $$
    for some suitable absolute constant $c_1 = \max(c_1',c_1'')>0$.\\
    \indent The other equation of this theorem is proved similarly as the "moreover" part in Lemma \ref{lemma:bound_psi}, which states that $\mu_k(D_k)\geq c_2$ if $C\beta_kk\log(k)\leq n$, and from the lower bound of Lemma \ref{lemma:smallest_largest_eigen_kernel}, it holds w.p. at least $1-\delta$.
\end{proof}

%\honam{This theorem will be very useful for pulling the diagonal matrix inside out}
%\honam{need to check though}
%\honam{need to check for instance: should we have $\frac{1}{n}$ here}
\begin{lemma}[Extension of Ostrowski's theorem]
    \label{lemma:ostrowski}
    We present the abstract matrix version here and we can obtain the bounds by substituting inside, let \( i, k \in \mathbb{N} \) satisfy \( 1 \leq i \leq \min(k, n) \) and a matrix $X_k\in\mathbb{R}^{n\times k}$. Let \( D_k := \frac{1}{n} X_k X_k^T \in \mathbb{R}^{n \times n} \). 
Suppose that the eigenvalues of \( \Sigma \) are given in non-increasing order \( \lambda_1 \geq \lambda_2 \geq \ldots \) then
    \[
    \lambda_{i+k-\min(n,k)} \mu_{\min(n,k)} (D_k) \leq \mu_i \left( \frac{1}{n} X_k\Sigma_{\leq k} X_k^\top \right) \leq \lambda_i \mu_1 (D_k).
    \]

    % We write an abstract matrix version not $\psi_{\le k}$ have a prove on this follows lemma 5 in \cite{barzilai2023generalization}
\end{lemma}
\begin{proof}
    We extends Ostrowski's theorem to the non-square case, where the proof is similar to Lemma 5 in \cite{barzilai2023generalization}.
    Let $\pi_1$ denote the number of positive eigenvalues of $\frac{1}{n} X_k \Sigma_{\leq k} X_k^T$, it follows from \cite{DANCIS1986141}[Theorem 1.5, Ostrowski's theorem] that for $1 \leq i \leq \pi_1$,
    $$
        \lambda_{i+k -\min(n,k)} \mu_{\min(n,k)}(D_k) \leq \mu_i(\frac{1}{n} X_k \Sigma_{\leq k} X_k^T) \leq \lambda_i \mu_1(D_k).
    $$
    Now we'll only have to consider the case where $\pi_i<i$. By definition of $\pi_1$ there are some orthonormal eigenvectors of $X_k \Sigma_{\leq k} X_k^T$, $v_{\pi_1+1},\dots,v_n $ with eigenvalues 0. Since $\Sigma \succeq 0$, for each such 0 eigenvector $v$,
    $$
    0 = (X_k^Tv)^T \Sigma_{\leq k} (X_k^Tv) \Rightarrow X_k^Tv = 0. 
    $$
    In particular, $D_k$ has $v_{\pi_1+1},\dots,v_n$ as 0 eigenvectors and since $D_k \succeq 0$, we have that $\mu_{\pi_1+1}(D_k),\dots,\mu_{n}(D_k)=0 $. So for $i>\pi_1$ we have
     $$
        \lambda_{i+k -\min(n,k)} \mu_{\min(n,k)}(D_k) \leq \mu_i(\frac{1}{n} X_k \Sigma_{\leq k} X_k^T) \leq \lambda_i \mu_1(D_k).
    $$
\end{proof}

\begin{lemma}[Symmetric Bound on eigenvalues of $\frac{1}{n}\tilde{K} $]
\label{lemma:symmetric_bound_1}
    Let $i, k \in \mathbb{N}$ satisfy $1 \leq i \leq n$ and $i \leq k$, let $D_k = \frac{1}{n} \hat{S}_n \Sigma^{- 1}_{\leq k} \hat{S}_n^*=\frac{1}{n} (\hat{S}_n \Sigma^{- 1/2}_{\leq k})(\hat{S}_n \Sigma^{- 1/2}_{\leq k})^T$ , and eigenvalues of $\tilde{\Sigma}$ is non-increasing i.e. $2p + \lambda \beta > 0$, then
    $$
    \lambda^{\beta}_{i + k - \min (n,k)}  p^2_{i + k - \min (n,k)}  \mu_{\min(n, k)}(D_k) + \mu_n(\frac{1}{n} \tilde{K}_{>k}) \leq \mu_i(\frac{1}{n} \tilde{K}) \leq \lambda_i^{\beta} p_i^2    \mu_{1}(D_k) + \mu_1(\frac{1}{n} \tilde{K}_{>k}).
    $$
    In particular
    $$
    \lambda^{\beta}_{i + k - \min (n,k)}  p^2_{i + k - \min (n,k)}  \mu_{\min(n, k)}(D_k) \leq \mu_i(\frac{1}{n} \tilde{K}) \leq \lambda_i^{\beta} p_i^2    \mu_{1}(D_k) + \mu_1(\frac{1}{n} \tilde{K}_{>k}).
    $$
\end{lemma}
\begin{proof} 
    We can decompose $\tilde{K}$ into the sum of two hermitian matrices $\tilde{K}_{\leq k}$ and $\tilde{K}_{> k}$. Then we can use Weyl's theorem \cite{Horn_Johnson_1985}[Corollary 4.3.15] to bound the eigenvalues of $\tilde{K}$ as
    $$
    \mu_i(\tilde{K}_{\leq k}) + \mu_n(\tilde{K}_{>k}) \leq \mu_i(\tilde{K}) \leq \mu_i(\tilde{K}_{\leq k}) + \mu_1(\tilde{K}_{>k}) .
    $$
    Then since $\tilde{K}_{\leq k} =  (\hat{S}_n \Sigma^{- 1/2}_{\leq k})\mathcal{A}^2 \Sigma^{\beta }(\hat{S}_n \Sigma^{- 1/2}_{\leq k})^T$, we use the extension of Ostrowski's theorem derived at Lemma \ref{lemma:ostrowski} to obtain the bound:
    $$
    \lambda^{\beta}_{i + k - \min (n,k)}  p^2_{i + k - \min (n,k)} \mu_{\min(n,k)}(D_k) \leq \mu_i(\frac{1}{n}\tilde{K}_{\leq k}) \leq \lambda_i \mu_1(D_k).
    $$
    Therefore, by combining the two results, it yields:
    $$
    \lambda^{\beta}_{i + k - \min (n,k)}  p^2_{i + k - \min (n,k)}  \mu_{\min(n, k)}(D_k) + \mu_n(\frac{1}{n} \tilde{K}_{>k}) \leq \mu_i(\frac{1}{n} \tilde{K}) \leq \lambda_i^{\beta} p_i^2    \mu_{1}(D_k) + \mu_1(\frac{1}{n} \tilde{K}_{>k}).
    $$
    The "in particular" part follows from $\mu_n(\frac{1}{n} \tilde{K}_{>k})\geq 0$.
\end{proof}

\begin{lemma}[Symmetric Bound on eigenvalues of $\frac{1}{n}\tilde{K}_{>k}$]
    \label{lemma:symmetric_bound_eigenvalue}
    For any $\delta>0$, it holds w.p. at least $1 - \delta$ that for all $i \in [n]$,
    $$
    \alpha_k \frac{1}{n} \trace(\tilde{\Sigma}_{>k}) \left(1 - \frac{1}{\delta} \sqrt{\frac{n^2}{\trace(\tilde{\Sigma}_{>k})^2/\trace(\tilde{\Sigma}^2_{>k})}}\right) \leq \mu_i(\frac{1}{n} \tilde{K}_{>k}) \leq \beta_k \frac{1}{n} \trace(\tilde{\Sigma}_{>k}) \left(1 + \frac{1}{\delta} \sqrt{\frac{n^2}{\trace(\tilde{\Sigma}_{>k})^2/\trace(\tilde{\Sigma}^2_{>k})}}\right) 
    $$
    % $$
    % \alpha_k \frac{1}{n}  \trace(\mathcal{A}_{>k}^2 \Sigma^{\beta}_{>k}) - \frac{1}{\delta}\sqrt{\trace(\mathcal{A}_{>k}^4 \Sigma^{2\beta}_{>k})} \leq \mu_i(\frac{1}{n} \tilde{K}_{>k}) \leq \beta_k \frac{1}{n}  \trace(\mathcal{A}_{>k}^2 \Sigma^{\beta}_{>k}) + \frac{1}{\delta}\sqrt{\trace(\mathcal{A}_{>k}^4 \Sigma^{2\beta}_{>k})} 
    % $$
    where $\tilde{\Sigma} := \mathcal{A}^2 \Sigma^{\beta}$.
\end{lemma}
%\fh{remember to ensure the lower bound is meaningful, i.e., $\geq 0$}

\begin{proof}
    We decompose the matrix into the diagonal component and non-diagonal component and bound them respectively, we denote diagonal component as $\text{diag}(\frac{1}{n} \tilde{K}_{>k})$ and $\Delta_{>k} := \frac{1}{n} \tilde{K}_{>k} - \text{diag}(\frac{1}{n} \tilde{K}^{\gamma}_{>k})$.

    Recall that $\tilde{K}_{>k} := \hat{S}_n \mathcal{A}_{> k}^2 \Sigma_{> k}^{\beta - 1} \hat{S}_n^*$, and for any $i \in [n]$, 
    \begin{align*}
        [\frac{1}{n}\tilde{K}_{>k}]_{ii} &= \frac{1}{n}\langle K_{x_i}, \mathcal{A}_{>k}^2 \Sigma_{>k}^{\beta - 1} K_{x_i} \rangle_{\mathcal{H}} \\
        &= \frac{1}{n}\langle \sum_{l=1}^{\infty} \phi_l(x_i) \phi_l, \sum_{l=k+1}^{\infty} p_l^2 \lambda_l^{\beta - 1} \phi_l(x_i) \phi_l \rangle_{\mathcal{H}} \\
        &= \frac{1}{n}\sum_{l=k+1}^{\infty} p_l^2 \lambda_l^{\beta} \psi_l(x_i)^2.
    \end{align*}
    Therefore, by definition of $\alpha_k$ and $\beta_k$, we have 
    $$\alpha_k \frac{1}{n} \trace(\mathcal{A}_{>k}^2 \Sigma_{>k}^{\beta}) \leq [\frac{1}{n}\tilde{K}_{>k}]_{ii} \leq \beta_k \frac{1}{n} \trace(\mathcal{A}^2_{>k} \Sigma^{\beta}_{>k}).$$
    %\fh{$>k$ is missing in LHS, e.g., $\mathcal{A}$ and $\Sigma$}
    Therefore, $$\alpha_k \frac{1}{n} \trace(\mathcal{A}^2_{>k} \Sigma^{\beta}_{>k}) I \preceq \text{diag}(\frac{1}{n} \tilde{K}_{>k}) \preceq \beta_k \frac{1}{n} \trace(\mathcal{A}^2_{>k} \Sigma^{\beta}_{>k}) I.$$
    Then by Weyl's theorem \cite{Horn_Johnson_1985}[Corollary 4.3.15], we can bound the eigenvalues of $\frac{1}{n} \tilde{K}_{>k}$ as
    $$
    \alpha_k \frac{1}{n} \trace(\mathcal{A}^2_{>k} \Sigma^{\beta}_{>k}) + \mu_n(\Delta_{>k}) \leq \mu_i(\frac{1}{n} \tilde{K}_{>k}) \leq \beta_k \frac{1}{n} \trace(\mathcal{A}^2_{>k} \Sigma^{\beta}_{>k}) + \mu_1(\Delta_{>k}).
    $$
    It remains to bound the eigenvalues of $\Delta_{>k}$, we first bound the expectation of the matrix norm using
    \begin{align*}
        \mathbb{E}[\|\Delta_{>k} \|] \leq& \  \mathbb{E}[\| \Delta_{>k}\|_{F}^2]^{1/2} = \sqrt{\sum_{i,j=1, i\neq j}^{n} \mathbb{E}[(\frac{1}{n} \sum_{l>k} p_l^2 \lambda_l^{\beta} \psi_l(x_i) \psi_l(x_j) )^2]} \\
        =& \sqrt{\frac{n(n-1)}{n^2} \trace(\mathcal{A}_{>k}^4 \Sigma^{2\beta}_{>k})} \leq \sqrt{ \trace(\mathcal{A}_{>k}^4 \Sigma^{2\beta}_{>k})} = \frac{1}{n} \trace(\tilde{\Sigma}_{>k}) \sqrt{\frac{n^2}{\trace(\tilde{\Sigma}_{>k})^2/\trace(\tilde{\Sigma}^2_{>k})}}.
    \end{align*}
    %The last equality is by defintion of $R_k$.
    
    By Markov's inequality,
    $$
    \mathbb{P}(\|\Delta_{>k}\|\geq \frac{1}{\delta} \mathbb{E} [\| \Delta_{>k}\|]) \leq \delta.
    $$
    So w.p. at least $1 - \delta$ it holds that
    $$\|\Delta_{>k}\|\leq\frac{1}{\delta} \mathbb{E}[\|\Delta_{>k} \|] \leq \frac{1}{n\delta} \trace(\tilde{\Sigma}_{>k}) \sqrt{\frac{n^2}{\trace(\tilde{\Sigma}_{>k})^2/\trace(\tilde{\Sigma}^2_{>k})}}.$$
    
\end{proof}

\begin{lemma}[Upper bound of largest eigenvalue]
    \label{lemma:upper_bound_largest_eigen}
     Suppose Assumption \ref{assumption:beta_regularity_condition} holds, and eigenvalues of $\tilde{\Sigma}$ are given in non-increasing order (i.e. $2p + \beta \lambda > 0$).
    There exists absolute constant $c,\ c'>0$ s.t. it holds w.p. at least $1-4\frac{r_k}{k^4}\exp(-\frac{c'}{\beta_k}\frac{n}{r_k})$ that
    $$
    \mu_1\left(\frac{1}{n} \hat{S}_n \mathcal{A}^2 \Sigma^{\beta - 1} \hat{S}_n^*\right) \leq c\left(p_{k+1}^2 \lambda_{k+1}^{\beta} + \beta_k \log(k+1) \frac{\trace(\tilde{\Sigma}_{>k})}{n}\right).
    $$
    where $\tilde{\Sigma} := \mathcal{A}^2 \Sigma^{\beta}$, $r_k := \frac{\trace(\tilde{\Sigma}_{>k})}{p_{k+1}^2\lambda_{k+1}^\beta}$.
\end{lemma}
\begin{proof}
    %\honam{@Wendao better check the notations inside}
    Let $m_k = \mu_1(\frac{1}{n}\tilde{K}_{>k}\ )$, $\tilde{K}_{k+1:p}=\hat{S}_n \mathcal{A}_{k+1:p}^2 \Sigma_{k+1:p}^{\beta - 1} \hat{S}_n^*$, the meaning of the footnote $k+1:p$ follows similar rule as the footnote $>k$, and let $\tilde{\Sigma} = \mathcal{A}^2\Sigma^\beta$, $\hat{\tilde{\Sigma}}_{>k} = \frac{1}{n}\mathcal{A}_{>k}\Sigma^{\frac{\beta - 1}{2}}_{>k}\hat{S_n^*}\hat{S_n}\Sigma^{\frac{\beta - 1}{2}}_{>k}\mathcal{A}_{>k} = \mathcal{A}_{>k}\Sigma^{\frac{\beta - 1}{2}}_{>k}\hat{\Sigma}\Sigma^{\frac{\beta - 1}{2}}_{>k}\mathcal{A}_{>k}$. Observe that $m_k = ||\hat{\tilde{\Sigma}}_{>k}||$, we would like to bound $||\hat{\tilde{\Sigma}}_{>k}||$ using the matrix Chernoff inequality with intrinsic dimension. \cite{tropp2015introduction}[Theorem 7.2.1]. However, this inequality was proved for finite matrices, so we'll approximate the infinite matrix with finite ones. $m_k$ can be bounded as:
     $$m_k = ||\frac{1}{n}\tilde{K}_{k+1:p'}+\frac{1}{n}\tilde{K}_{>p'}||\leq||\frac{1}{n}\tilde{K}_{k+1:p'}||+||\frac{1}{n}{\Tilde{K}}_{>p'}|| =  ||\hat{\tilde{\Sigma}}_{k+1:p'}|| + m_{p'}.$$
     %\honam{@Wendao Change $\hat{\tilde{\Sigma}}$ to $\tilde{\Sigma}$}
    Furthermore, $m_p'$ can be bounded as
    $$ m_{p'}\leq\frac{1}{n}\trace(\tilde{K}_{>p'}) = \frac{1}{n}\sum^n_{j=1}\sum_{i>p'}p_i^2\lambda_i^{\beta}\psi_i(x_j)^2\leq\beta_{p'}\sum_{i>p'}p_i^2\lambda_i^\beta\leq\beta_{p'}\trace(\tilde\Sigma_{>p'}). $$
    
    If $p$ is finite, we can take $p=p'$ and $m_p' = 0$. Otherwise, $p$ is infinite, and $m_{p'}\leq \beta_{p'}\trace(\Sigma_{>p'})$. By assumption \ref{assumption:beta_regularity_condition}:
    $$\forall \epsilon>0, \exists p'\in \mathbb{N} \text{ s.t. }    m_{p'}<\epsilon .$$
    
    We define $S^j_{k+1:p'}:=\frac{1}{n}\mathcal{A}_{k+1:p'}\Sigma^{\frac{\beta - 1}{2}}_{k+1:p'}\hat{S}^{j*}\hat{S}^j\Sigma^{\frac{\beta - 1}{2}}_{k+1:p'}\mathcal{A}_{k+1:p'}$, where $\hat{S}^jf=\left< f,K_{x_j}\right>_{\mathcal{H}}$ and $\hat{S}^{j*}\theta = \theta_jK_{x_j}$. Then we will have $\hat{\tilde{\Sigma}}_{k+1:p'}= \sum_{j=1}^nS^j_{k+1:p'}$. We need a bound on both $\mu_1(S^j_{k+1:p'})$ and $\mu_1(\mathbb{E}\hat{\tilde{\Sigma}}_{k+1:p'})$. For the first,
    $$\mu_1(S^j_{k+1:p'}) = \frac{1}{n}\sum_{i=k+1}^{p'}p_i^2\lambda_i^{\beta}\psi_i(x_j)^2\leq\frac{1}{n}\sum_{i=k+1}^{\infty}p_i^2\lambda_i^{\beta}\psi_i(x_j)^2\leq \frac{\beta_{k}}{n}\trace(\tilde\Sigma_{>k}). $$
    
    Let $L := \frac{\beta_{k}}{n}\trace(\tilde\Sigma_{>k})$ denoting the RHS. For the second item, $\mathbb{E}\hat{\tilde{\Sigma}}_{k+1:p'} =\tilde{\Sigma}_{k+1:p'}=\text{diag}(p_{k+1}^2\lambda_{k+1}^\beta,\dots,p_{p'}^2\lambda_{p'}^\beta) $. Thus, $\mathbb{E}\hat{\tilde{\Sigma}}_{k+1:p'}=p_{k+1}^2\lambda_{k+1}^\beta$.

    Now the conditions of \cite{tropp2015introduction}[Theorem 7.2.1] are satisfied. So, for $r_{k:p'}:=\frac{\trace(\tilde{\Sigma}_{k+1:p'})}{p_{k+1}^2\lambda_{k+1}^\beta}$ and any $t\geq 1+ \frac{L}{p_{k+1}^2\lambda_{k+1}^\beta}= 1+ \frac{\beta_kr_k}{n}$,
    $$
    \mathbb{P}(||\hat{\tilde{\Sigma}}_{k+1:p'}||\geq t p_{k+1}^2\lambda_{k+1}^\beta ) \leq 2r_{k:p'}\left( \frac{e^{t-1}}{t^t}  \right)^{p_{k+1}^2\lambda_{k+1}^\beta/L}.
    $$
    Using the fact that $p_{k+1}^2\lambda_{k+1}^\beta/L = n/\beta_kr_k$ and $e^{t-1}\leq e^t,\ r_{k:p'}\leq r_k$,
    $$
    \mathbb{P}(m_k-m_{p'}\geq t p_{k+1}^2\lambda_{k+1}^\beta)\leq\mathbb{P}(||\hat{\tilde{\Sigma}}_{k+1:p'}||\geq t p_{k+1}^2\lambda_{k+1}^\beta )\leq 2r_k\left( \frac{e}{t}  \right)^{nt/\beta_k r_k}.
    $$
    Now pick $t = e^3 + 2\frac{\beta_kr_k}{n}\ln(k+1)$, then
    $$
     \mathbb{P}(m_k-m_{p'}\geq t p_{k+1}^2\lambda_{k+1}^\beta) \leq 2 \frac{r_k}{(k+1)^4}\exp\left(-2\frac{e^3}{\beta_k}\frac{n}{r_k}\right). 
    $$
    As a result, we obtain that for $c' = 2 e^3,\ c = e^3$, the inequality holds w.p. at least $1-4\frac{r_k}{k^4}\exp(-\frac{c'}{\beta_k}\frac{n}{r_k})$ that
    $$
    m_k\leq c\left( p_{k+1}^2\lambda_{k+1}^\beta+\beta_k\log(k+1)\frac{\trace(\tilde\Sigma_{>k})}{n}\right)+m_{p'}. 
    $$
    As $p'$ tends to $\infty$ in some sequence determined by Assumption 1, $m_p'$ tends to 0. Therefore, we obtain the desired result.
\end{proof}

In the following we present an important lemma for bounding largest and smallest eigenvalues of unregularized spectrally transformed matrix. This lemma would be useful to bound concentration coefficient $\rho_{k,n}$ in the interpolation case.

\begin{lemma}[Bounds on $\mu_1(\frac{1}{n} \tilde{K}_{>k})$ and $\mu_n(\frac{1}{n} \tilde{K}_{>k'})$]
    Suppose Assumption \ref{assumption:beta_regularity_condition} holds, then there exists absolute constant $c, c' > 0$ s.t. it holds w.p. at least $1 - 4\frac{r_k}{k^4}\exp(-\frac{c'}{\beta_k}\frac{n}{r_k})$ that
    $$
    \mu_1(\frac{1}{n}\tilde{K}_{>k})  \leq c\left(p_{k+1}^2 \lambda_{k+1}^{\beta} + \beta_k \log(k+1) \frac{\trace(\tilde{\Sigma}_{>k})}{n}\right).
    $$
    And for any $k' \in \mathbb{N}$ with $k' > k$
, and any $\delta > 0$ it holds w.p. at least $1 - \delta - 4\frac{r_k}{k^4}\exp(-\frac{c'}{\beta_k}\frac{n}{r_k})$ that

    $$
    \alpha_{k'} \left(1 - \frac{1}{\delta} \sqrt{\frac{n^2}{\trace(\tilde{\Sigma}_{>k'})^2/\trace(\tilde{\Sigma}^2_{>k'})}}\right) \leq \mu_n(\frac{1}{n} \tilde{K}_{>k'}),
    $$
    where $\tilde{\Sigma} := \mathcal{A}^2 \Sigma^{\beta}$, $r_k := \frac{\trace(\tilde{\Sigma}_{>k})}{p_{k+1}^2\lambda_{k+1}^\beta}$.
    \label{lemma:smallest_largest_eigen_kernel}
\end{lemma}
\begin{proof}
    By Weyl's theorem \cite{Horn_Johnson_1985}[Corollary 4.3.15], for any $k' \geq k$ we have $\mu_n(\tilde K_{\geq k}) \geq \mu_n(\tilde K_{\geq k'}) + \mu_n(\tilde K_{k:k'}) \geq \mu_n(\tilde K_{\geq k'})$. So the lower bound comes from \ref{lemma:symmetric_bound_eigenvalue}(with $k'$) and the upper bound directly comes from \ref{lemma:upper_bound_largest_eigen}. 
    
\end{proof}

\section{Upper Bound for the Variance}
\label{appendix:ub_variance}
\begin{lemma}[Upper bound for variance]
    \label{lemma:intermediate_ub_variance}
    We define the variance of the noise be $\sigma_{\varepsilon}^2$ and evaluate variance in $\mathcal{H}^{\beta'}$ norm, If for some $k \in \mathbb{N}$, $\tilde{K}_{>k}^{\gamma}$ is positive-definite then
    % \begin{align*}
    %     V &\leq \sigma_{\varepsilon}^2 \cdot  (\frac{\mu_1((\tilde{K}_{>k}^{\gamma})^{-1})^2}{\mu_n((\tilde{K}_{>k}^{\gamma})^{-1})^2} \frac{\trace(\overbrace{\hat{S}_n \psi_{\leq k}^* \Lambda^{\leq k}_{\mathcal{A}^2 \Sigma^{\beta' + 2\beta - 2}} \psi_{\leq k} \hat{S}^{*}_{n}}^{n \times n})}{\mu_k(\underbrace{ \Lambda^{\leq k}_{\mathcal{A} \Sigma^{(\beta' + \beta - 1)/2}}\psi_{\leq k} \hat{S}_{n}^{*} \hat{S}_n \psi^*_{\leq k} \Lambda^{\leq k}_{\mathcal{A} \Sigma^{(\beta' + \beta - 1)/2}}}_{k \times k})^2}\\
    %     &+ \mu_1((\tilde{K}^\gamma_{>k})^{-1})^2 \trace(\underbrace{\hat{S}_n \phi_{>k}^* \Lambda^{>k}_{\mathcal{A}^2 \Sigma^{(-\beta'+2\beta-1)}} \phi_{>k} \hat{S}_n^*}_{n \times n}) \\
    % \end{align*}
    \begin{align*}
    V \leq& \sigma_\varepsilon^2 \cdot \bigg[
    \frac{(\mu_1(\tilde{K}_{>k}^\gamma)^{-1})^2}{(\mu_n(\tilde{K}_{>k}^\gamma)^{-1})^2}\frac{\trace(\hat{S}_n \psi^*_{\leq k} \Lambda^{\leq k}_{\mathcal{A}^{-2} \Sigma^{-\beta'}} \psi_{\leq k} \hat{S}^*_n)}{\mu_k(\psi_{\leq k} \hat{S}^*_n \hat{S}_n \psi^*_{\leq k})^2} \\
    +& (\mu_1(\tilde{K}_{>k}^\gamma)^{-1})^2  \trace(\hat{S}_n \psi^*_{>k} \Lambda^{>k}_{\mathcal{A}^2 \Sigma^{-\beta' + 2\beta}} \psi_{>k} \hat{S}^*_n)\bigg].
    \end{align*}
\end{lemma}
\begin{proof}
    Recall $V = \mathbb{E}_{\varepsilon} [\| \hat{f}(\varepsilon) \|^2_{\mathcal{H}^{\beta'}} ]$, we can split the variance into $\| \hat{f}(\varepsilon)_{\leq k} \|^2_{\mathcal{H}^{\beta'}}$ and $\| \hat{f}(\varepsilon)_{> k} \|^2_{\mathcal{H}^{\beta'}}$ according to Lemma~\ref{lemma:separate_smaller_k_large_eq_k}. To bound these, by Lemma~\ref{lemma:equiv_sobolev_matrix_norm} we could bound $\|\phi_{\leq k} \hat{f}(\varepsilon)_{\leq k}\|^2_{\Lambda_{\Sigma^{1 - \beta'}_{\leq k}}}$, $\|\phi_{> k} \hat{f}(\varepsilon)_{> k}\|^2_{\Lambda_{\Sigma^{1 - \beta'}_{> k}}}$ respectively using matrix inequalities.

    First we handle $\|\phi_{\leq k} \hat{f}(\varepsilon)_{\leq k}\|^2_{\Lambda_{\Sigma^{1 - \beta'}_{\leq k}}}$, using Lemma~\ref{lemma:separate_search} while substituting $y$ with $\varepsilon$, we have
    $$
    \phi_{\leq k}\hat{f}(\varepsilon)_{\leq k} + \phi_{\leq k} \mathcal{A}_{\leq k} \Sigma_{\leq k}^{\beta - 1} \hat{S}_n^* (\tilde{K}_{>k}^{\gamma})^{-1} \hat{S}_n \mathcal{A}_{\leq k} \hat{f}(\varepsilon)_{\leq k} = \phi_{\leq k}\mathcal{A}_{\leq k} \Sigma^{\beta - 1}_{\leq k} \hat{S}_n^* (\tilde{K}^{\gamma}_{>k})^{-1} \varepsilon.
    $$
    We multiply by $(\phi_{\leq k}\hat{f}(\varepsilon)_{\leq k})^T \Lambda^{\leq k}_{\mathcal{A}^{-2} \Sigma^{-\beta + (1 - \beta')}} \in \mathbb{R}^{1 \times k}$, on two sides respectively (note that the motivation of multiplying an additional diagonal matrix term here is to make the $\mu_k$ term only have $\mu_k(\psi_{\le k}\hat S_n^\ast \hat S_n \psi^\ast_{\le k})$), and this would not affect the polynomial bound.
    
    Then since $\|\phi_{\leq k}\hat{f}(\varepsilon)_{\leq k}\|^2_{\Lambda^{\leq k}_{\mathcal{A}^{-2} \Sigma^{-\beta + (1 - \beta')}}} \geq 0$, we have
    \begin{align*}
    &\underbrace{(\phi_{\leq k}\hat{f}(\varepsilon)_{\leq k})^T\Lambda^{\leq k}_{\mathcal{A}^{-2} \Sigma^{-\beta + (1 - \beta')}} \phi_{\leq k} \mathcal{A}_{\leq k} \Sigma_{\leq k}^{\beta - 1} \hat{S}_n^* (\tilde{K}_{>k}^{\gamma})^{-1} \hat{S}_n \mathcal{A}_{\leq k} \hat{f}(\varepsilon)_{\leq k}}_{\text{Quadratic term w.r.t. $\phi_{\leq k}\hat{f}(\varepsilon)_{\leq k}$}} \\
    &\leq \underbrace{(\phi_{\leq k}\hat{f}(\varepsilon)_{\leq k})^T\Lambda^{\leq k}_{\mathcal{A}^{-2} \Sigma^{-\beta + (1 - \beta')}} \phi_{\leq k}\mathcal{A}_{\leq k} \Sigma^{\beta - 1}_{\leq k} \hat{S}_n^* (\tilde{K}^{\gamma}_{>k})^{-1} \varepsilon}_{\text{Linear term w.r.t. $\phi_{\leq k}\hat{f}(\varepsilon)_{\leq k}$}}.     
    \end{align*}
    % $$\underbrace{\langle \hat{f}(y)_{\leq k},  \mathcal{A}_{\leq k} \Sigma_{\leq k}^{\beta - 1} \hat{S}_n^* (\tilde{K}_{>k}^{\gamma})^{-1} \hat{S}_n \mathcal{A}_{\leq k} \hat{f}(y)_{\leq k}  \rangle_{\mathcal{H}}}_{\text{Quadratic term}} \leq \underbrace{\langle \hat{f}(y)_{\leq k}, \mathcal{A}_{\leq k} \Sigma^{\beta - 1}_{\leq k} \hat{S}_n^* (\tilde{K}^{\gamma}_{>k})^{-1} y  \rangle_{\mathcal{H}}}_{\text{Linear term}}$$
    Then we lower bound the quadratic term and upper bound the linear term respectively, first we lower bound the quadratic term:
    \begin{align*}
        &(\phi_{\leq k}\hat{f}(\varepsilon)_{\leq k})^T \Lambda^{\leq k}_{\mathcal{A}^{-2} \Sigma^{-\beta + (1 - \beta')}} \phi_{\leq k} \mathcal{A}_{\leq k} \Sigma_{\leq k}^{\beta - 1} \hat{S}_n^* (\tilde{K}_{>k}^{\gamma})^{-1} \hat{S}_n \mathcal{A}_{\leq k} \hat{f}(\varepsilon)_{\leq k} \\
        &\text{ Diagonalize the operators,}\\
        =& \ (\phi_{\leq k}\hat{f}(\varepsilon)_{\leq k})^T \Lambda^{\leq k}_{\mathcal{A}^{-2} \Sigma^{-\beta + (1 - \beta')}} \phi_{\leq k} (\phi_{\leq k}^* \Lambda^{\leq k}_{\mathcal{A} \Sigma^{\beta - 1}} \phi_{\leq k}) \hat{S}_{n}^{*} (\tilde{K}_{>k}^{\gamma})^{-1} \hat{S}_n (\phi_{\leq k}^* \Lambda_{\mathcal{A}}^{\leq k} \phi_{\leq k}) \hat{f}(\varepsilon)_{\leq k}\\
        =& \ (\phi_{\leq k}\hat{f}(\varepsilon)_{\leq k})^T \Lambda^{\leq k}_{\mathcal{A}^{-1} \Sigma^{-\beta'}} \phi_{\leq k} \hat{S}_{n}^{*} (\tilde{K}_{>k}^{\gamma})^{-1} \hat{S}_n \phi_{\leq k}^* \Lambda_{\mathcal{A}}^{\leq k} (\phi_{\leq k} \hat{f}(\varepsilon)_{\leq k}) \text{ 
        ($\phi_{\leq k} \phi^{*}_{\leq k} = id_{\leq k}$)} \\
        &\text{  By $\phi_{\leq k} = \Lambda_{\Sigma^{1/2}}^{\leq k} \psi_{\leq k}$ and $\phi^{*}_{\leq k} = \psi^{*}_{\leq k} \Lambda_{\Sigma^{1/2}}^{\leq k}$,}\\
        =& \ \underbrace{(\phi_{\leq k}\hat{f}(\varepsilon)_{\leq k})^T}_{1 \times k} \underbrace{\Lambda^{\leq k}_{\mathcal{A}^{-1} \Sigma^{1/2-\beta'}}}_{k \times k}  \underbrace{\psi_{\leq k} \hat{S}_{n}^{*}}_{k \times n} \underbrace{(\tilde{K}_{>k}^{\gamma})^{-1}}_{n \times n} \underbrace{\hat{S}_n \psi_{\leq k}^*}_{n \times k}
        \underbrace{\Lambda_{\mathcal{A} \Sigma^{1/2}}^{\leq k} (\phi_{\leq k} \hat{f}(\varepsilon)_{\leq k})}_{k \times 1}  \\
        \geq& \ \mu_n((\tilde{K}_{>k}^{\gamma})^{-1}) \ \mu_k(  \psi_{\leq k} \hat{S}_{n}^{*} \hat{S}_n \psi^*_{\leq k} ) \ \ (\phi_{\leq k} \hat{f}\cite(\varepsilon)_{\leq k})^T \Lambda^{\leq k}_{\Sigma^{1 - \beta'}} (\phi_{\leq k} \hat{f}(\varepsilon)_{\leq k}). 
    \end{align*}

    \noindent The last inequality is because $\mu_k(AB) = \mu_k(BA)$ for $k\times k$ matrix $A, B$ by \cite[Theorem 1.3.20]{Horn_Johnson_1985}. 
        % \honam{This should be correct right? I did that for easier concentration, but seems now is not very necessary with what you wrote in \ref{lemma:Lambda_psi_Lambda}} \yplu{I don't get your point here}

    \noindent We continue to derive the bound
    \begin{align*}
        &\mu_n((\tilde{K}_{>k}^{\gamma})^{-1}) \ \mu_k( \psi_{\leq k} \hat{S}_{n}^{*} \hat{S}_n \psi^*_{\leq k} ) \ \ (\phi_{\leq k} \hat{f}(\varepsilon)_{\leq k})^T \Lambda^{\leq k}_{\Sigma^{1 - \beta'}} (\phi_{\leq k} \hat{f}(\varepsilon)_{\leq k}) \\
        =& \ \| \phi_{\leq k} \hat{f}(\varepsilon)_{\leq k} \|^2_{\Lambda^{\leq k}_{\Sigma^{1 - \beta'}}} \ \ \mu_n((\tilde{K}_{>k}^{\gamma})^{-1}) \ \mu_k( \psi_{\leq k} \hat{S}_{n}^{*} \hat{S}_n \psi^*_{\leq k}) \\
        =& \ \| \hat{f}(\varepsilon)_{\leq k} \|^2_{\mathcal{H}^{\beta'}} \ \ \mu_n((\tilde{K}_{>k}^{\gamma})^{-1}) \ \mu_k( \psi_{\leq k} \hat{S}_{n}^{*} \hat{S}_n \psi^*_{\leq k} ).
    \end{align*}

    \noindent This finishes lower bound of the quadratic term, we continue to upper bound the linear term     
    \begin{align*}
        &(\phi_{\leq k}\hat{f}(\varepsilon)_{\leq k})^T\Lambda^{\leq k}_{\mathcal{A}^{-2} \Sigma^{-\beta + (1 - \beta')}} \phi_{\leq k}\mathcal{A}_{\leq k} \Sigma^{\beta - 1}_{\leq k} \hat{S}_n^* (\tilde{K}^{\gamma}_{>k})^{-1} \varepsilon \\
        =& \ (\phi_{\leq k}\hat{f}(\varepsilon)_{\leq k})^T\Lambda^{\leq k}_{\mathcal{A}^{-2} \Sigma^{-\beta + (1 - \beta')}} \phi_{\leq k} \phi^{*}_{\leq k} \Lambda^{\leq k}_{\mathcal{A} \Sigma^{\beta - 1}}  \phi_{\leq k} \hat{S}_n^* (\tilde{K}^{\gamma}_{>k})^{-1}  \varepsilon\\
        =& \ \underbrace{(\phi_{\leq k}\hat{f}(\varepsilon)_{\leq k})^T \Lambda^{\leq k}_{\mathcal{A}^{-1} \Sigma^{1/2-\beta'}}}_{1 \times k} \underbrace{ \psi_{\leq k} \hat{S}_n^* (\tilde{K}^{\gamma}_{>k})^{-1}  \varepsilon}_{k \times 1} \text{  (By $\phi_{\leq k} \phi_{\leq k}^* = id_{\leq k}$ and $\phi_{\leq k} = \Lambda^{\leq k}_{\Sigma^{1/2}} \psi_{\leq k}$)} \\
        =& \ \underbrace{(\phi_{\leq k}\hat{f}(\varepsilon)_{\leq k})^T \Lambda^{\leq k}_{\Sigma^{(1 - \beta')/2}}}_{1 \times k} \underbrace{\Lambda^{\leq k}_{\mathcal{A}^{-1} \Sigma^{-\beta'/2}} \psi_{\leq k} \hat{S}_n^* (\tilde{K}^{\gamma}_{>k})^{-1}  \varepsilon}_{k \times 1} \\
        \leq& \ \| \phi_{\leq k} \hat{f}(\varepsilon)_{\leq k} \|_{\Lambda^{\leq k}_{\Sigma^{1 - \beta'}}} \| \Lambda^{\leq k}_{\mathcal{A}^{-1} \Sigma^{-\beta'/2}} \psi_{\leq k} \hat{S}_n^* (\tilde{K}^{\gamma}_{>k})^{-1}  \varepsilon \|\\
        =& \ \| \hat{f}(\varepsilon)_{\leq k} \|_{\mathcal{H}^{\beta'}} \|\Lambda^{\leq k}_{\mathcal{A}^{-1} \Sigma^{-\beta'/2}} \psi_{\leq k} \hat{S}_n^* (\tilde{K}^{\gamma}_{>k})^{-1}  \varepsilon \|.
    \end{align*}
    
    \noindent Therefore, we obtain 
    \begin{align*}
            &\| \hat{f}(\varepsilon)_{\leq k} \|^2_{\mathcal{H}^{\beta'}} \ \ \mu_n((\tilde{K}_{>k}^{\gamma})^{-1}) \ \mu_k( \psi_{\leq k} \hat{S}_{n}^{*} \hat{S}_n \psi^*_{\leq k} )  \leq \| \hat{f}(\varepsilon)_{\leq k} \|_{\mathcal{H}^{\beta'}} \|\Lambda^{\leq k}_{\mathcal{A}^{-1} \Sigma^{-\beta'/2}} \psi_{\leq k} \hat{S}_n^* (\tilde{K}^{\gamma}_{>k})^{-1}  \varepsilon \|.
    \end{align*}

    \noindent Therefore,
    $$\|\hat{f}(\varepsilon)_{\leq k}\|^2_{\mathcal{H}^{\beta'}} \leq \frac{\varepsilon^T (\tilde{K}_{>k}^{\gamma})^{-1} \hat{S}_n \psi_{\leq k}^* \Lambda^{\leq k}_{\mathcal{A}^{-2} \Sigma^{-\beta'}} \psi_{\leq k} \hat{S}^{*}_{n} (\tilde{K}_{>k}^{\gamma})^{-1} \varepsilon}{\mu_n((\tilde{K}_{>k}^{\gamma})^{-1})^2 \ \mu_k( \psi_{\leq k} \hat{S}_{n}^{*} \hat{S}_n \psi^*_{\leq k} )^2}.$$
    Then we take expectation w.r.t $\varepsilon$
    we have 
    \begin{align*}
        \mathbb{E}_{\varepsilon} \|\hat{f}(\varepsilon)_{\leq k}\|^2_{\mathcal{H}^{\beta'}} &\leq \sigma_{\varepsilon}^2 \cdot \frac{\trace( \overbrace{(\tilde{K}_{>k}^{\gamma})^{-1}}^{n \times n} \overbrace{ \hat{S}_n \psi_{\leq k}^* \Lambda^{\leq k}_{\mathcal{A}^{-2} \Sigma^{-\beta'}} \psi_{\leq k} \hat{S}^{*}_{n}}^{n \times n} \overbrace{(\tilde{K}_{>k}^{\gamma})^{-1})}^{n \times n}}{\mu_n((\tilde{K}_{>k}^{\gamma})^{-1})^2 \ \mu_k( \psi_{\leq k} \hat{S}_{n}^{*} \hat{S}_n \psi^*_{\leq k} )^2} \\
        &\leq \sigma_{\varepsilon}^2 \cdot  \frac{\mu_1((\tilde{K}_{>k}^{\gamma})^{-1})^2}{\mu_n((\tilde{K}_{>k}^{\gamma})^{-1})^2} \frac{\trace(\overbrace{\hat{S}_n \psi_{\leq k}^* \Lambda^{\leq k}_{\mathcal{A}^{-2} \Sigma^{-\beta'}} \psi_{\leq k} \hat{S}^{*}_{n}}^{n \times n})}{\mu_k(\underbrace{ \psi_{\leq k} \hat{S}_{n}^{*} \hat{S}_n \psi^*_{\leq k} }_{k \times k})^2},
    \end{align*}
    where the last inequality is by using the fact that $\trace(M M' M) \leq \mu_1(M)^2 \trace(M')$ for positive-definite matrix $M, M'$.
    
    Now we move on to bound the $>k$ components $\|\phi_{> k} \hat{f}(\varepsilon)_{> k}\|^2_{\Lambda^{> k}_{\Sigma^{1 - \beta'}}}$ 
    \begin{align*}
        &\| \phi_{>k} \hat{f}(\varepsilon)_{>k} \|^2_{\Lambda^{>k}_{\Sigma^{1 - \beta'}}} \\
        =& \ \| \phi_{>k} \mathcal{A}_{>k} \Sigma^{\beta - 1}_{>k}  \hat{S}_n^* (\tilde{K}^\gamma)^{-1} \varepsilon\|^2_{\Lambda^{>k}_{\Sigma^{1 - \beta'}}} \\
        =& \ \varepsilon^T (\tilde{K}^\gamma)^{-1} \hat{S}_n \Sigma_{>k}^{\beta - 1} \mathcal{A}_{>k} \phi^*_{>k} \Lambda^{>k}_{\Sigma^{1 - \beta'}} \phi_{>k} \mathcal{A}_{>k} \Sigma_{>k}^{\beta - 1} 
 \hat{S}_n^{*} (\tilde{K}^\gamma)^{-1}\varepsilon\\
        =& \ \varepsilon^T (\tilde{K}^\gamma)^{-1} \hat{S}_n \phi_{>k}^* \Lambda^{>k}_{\mathcal{A}^2 \Sigma^{(-\beta'+2\beta-1)}} \phi_{>k} \hat{S}_n^* (\tilde{K}^\gamma)^{-1}\varepsilon \text{  (By $2(\beta - 1) + (1 - \beta') = -\beta' + 2\beta - 1$)}.
    \end{align*}
    We take expectation over $\varepsilon$
    \begin{align*}
        \mathbb{E}_{\varepsilon} \| \phi_{>k} \hat{f}(\varepsilon)_{>k} \|^2_{\Lambda_{\mathcal{A}^2 \Sigma^{\beta}}} \leq& \sigma_{\varepsilon}^2 \mu_1((\tilde{K}^\gamma)^{-1})^2 \trace(\hat{S}_n \phi_{>k}^* \Lambda^{>k}_{\mathcal{A}^2 \Sigma^{(-\beta'+2\beta-1)}} \phi_{>k} \hat{S}_n^*)\\
        \leq& \sigma_{\varepsilon}^2 \mu_1((\tilde{K}^\gamma_{>k})^{-1})^2 \trace(\underbrace{\hat{S}_n \phi_{>k}^* \Lambda^{>k}_{\mathcal{A}^2 \Sigma^{(-\beta'+2\beta-1)}} \phi_{>k} \hat{S}_n^*}_{n \times n}) \\
        =& \sigma_{\varepsilon}^2 \mu_1((\tilde{K}^\gamma_{>k})^{-1})^2 \trace(\underbrace{\hat{S}_n \psi_{>k}^* \Lambda^{>k}_{\mathcal{A}^2 \Sigma^{(-\beta'+2\beta)}} \psi_{>k} \hat{S}_n^*}_{n \times n}),
    \end{align*}
    where the second last inequality is still using the fact that $\trace(M M' M) \leq \mu_1(M)^2 \trace(M')$ for positive-definite matrix $M, M'$, and the last inequality is using $\tilde{K}^{\gamma} \succeq \tilde{K}^{\gamma}_{>k}$ to infer $\mu_1((\tilde{K}^{\gamma} )^{-1}) \leq \mu_1((\tilde{K}^{\gamma}_{>k})^{-1})$.
\end{proof}

\begin{theorem}[Bound on Variance with concentration coefficient] Following previous Theorem \ref{lemma:intermediate_ub_variance}'s assumptions, we can express the bound of variance using concentration coefficient $\rho_{n,k}$
    \label{theorem:variance_proof}
\begin{align*}
        V &\leq \sigma_{\varepsilon}^2\rho_{k,n}^2 \cdot  {\Big(}\frac{{\trace(\hat{S}_n \psi^*_{\leq k} \Lambda^{\leq k}_{\mathcal{A}^{-2} \Sigma^{-\beta'}} \psi_{\leq k} \hat{S}^*_n)}}{{\mu_k(\psi_{\leq k} \hat{S}^*_n \hat{S}_n \psi^*_{\leq k})^2}}
    + \frac{\overbrace{\trace(\hat{S}_n \psi^*_{>k} \Lambda^{>k}_{\mathcal{A}^2 \Sigma^{-\beta' + 2\beta}} \psi_{>k} \hat{S}^*_n)}^{\text{effective rank}}}{n^2 {\| \tilde{\Sigma}_{>k} \|^2}} {\Big)}.
    \end{align*}
\end{theorem}
\begin{proof}
    By \ref{lemma:intermediate_ub_variance} we have
    \begin{align*}
    V \leq& \sigma_\varepsilon^2 \cdot  {\Big(}
    \frac{(\mu_1(\tilde{K}_{>k}^\gamma)^{-1})^2}{(\mu_n(\tilde{K}_{>k}^\gamma)^{-1})^2}\frac{\trace(\hat{S}_n \psi^*_{\leq k} \Lambda^{\leq k}_{\mathcal{A}^{-2} \Sigma^{-\beta'}} \psi_{\leq k} \hat{S}^*_n)}{\mu_k(\psi_{\leq k} \hat{S}^*_n \hat{S}_n \psi^*_{\leq k})^2} \\
    +& (\mu_1(\tilde{K}_{>k}^\gamma)^{-1})^2  \trace(\hat{S}_n \psi^*_{>k} \Lambda^{>k}_{\mathcal{A}^2 \Sigma^{-\beta' + 2\beta}} \psi_{>k} \hat{S}^*_n)  {\Big)}.
    \end{align*}
    Then by $\displaystyle \mu_1(\tilde{K}^{\gamma}_{>k})^{-1} = \frac{1}{n \mu_n(\frac{1}{n} \tilde{K}^{\gamma}_{>k})}, \mu_n(\tilde{K}^{\gamma}_{>k})^{-1} = \frac{1}{n \mu_1(\frac{1}{n} \tilde{K}^{\gamma}_{>k})}$, we have
    $$
\frac{(\mu_1(\tilde{K}_{>k}^\gamma)^{-1})^2}{(\mu_n(\tilde{K}_{>k}^\gamma)^{-1})^2} = \frac{\mu_1(\tilde{K}_{>k}^\gamma)^2}{\mu_n(\tilde{K}_{>k}^\gamma)^2} \leq \frac{(\mu_1(\tilde{K}_{>k}) + \gamma)^2}{(\mu_n(\tilde{K}_{>k}) + \gamma)^2} \leq \rho_{k,n}^2.
$$
And
\begin{align*}
    &(\mu_1(\tilde{K}^{\gamma}_{>k})^{-1})^2 \\ 
    \leq& \ \frac{1}{n^2} \frac{1}{\mu_n(\frac{1}{n} \tilde{K}^{\gamma}_{>k})^2} \\
    =& \ \frac{1}{n^2} \frac{\|\tilde{\Sigma}_{>k} \|^2}{\mu_n(\frac{1}{n} \tilde{K}^{\gamma}_{>k})^2} \frac{1}{\|\tilde{\Sigma}_{>k} \|^2} \\
    \leq& \ \frac{\rho_{k,n}^2}{n^2} \frac{1}{\|\tilde{\Sigma}_{>k} \|^2}. 
\end{align*}
    Therefore,
   \begin{align*}
    V \leq& \ \sigma_\varepsilon^2 \cdot  {\Big(}
    \frac{(\mu_1(\tilde{K}_{>k}^\gamma)^{-1})^2}{(\mu_n(\tilde{K}_{>k}^\gamma)^{-1})^2}\frac{\trace(\hat{S}_n \psi^*_{\leq k} \Lambda^{\leq k}_{\mathcal{A}^{-2} \Sigma^{-\beta'}} \psi_{\leq k} \hat{S}^*_n)}{\mu_k(\psi_{\leq k} \hat{S}^*_n \hat{S}_n \psi^*_{\leq k})^2} \\
    +& \ (\mu_1(\tilde{K}_{>k}^\gamma)^{-1})^2  \trace(\hat{S}_n \psi^*_{>k} \Lambda^{>k}_{\mathcal{A}^2 \Sigma^{-\beta' + 2\beta}} \psi_{>k} \hat{S}^*_n) {\Big)} \\
    \leq& \ \sigma_\varepsilon^2 \cdot  {\Big(}
    \rho_{k,n}^2 \frac{\trace(\hat{S}_n \psi^*_{\leq k} \Lambda^{\leq k}_{\mathcal{A}^{-2} \Sigma^{-\beta'}} \psi_{\leq k} \hat{S}^*_n)}{\mu_k(\psi_{\leq k} \hat{S}^*_n \hat{S}_n \psi^*_{\leq k})^2} \\
    +& \ \frac{\rho_{k,n}^2}{n^2} \frac{1}{\| \tilde{\Sigma}_{>k} \|^2}  \trace(\hat{S}_n \psi^*_{>k} \Lambda^{>k}_{\mathcal{A}^2 \Sigma^{-\beta' + 2\beta}} \psi_{>k} \hat{S}^*_n) {\Big)} \\
    \leq& \sigma_{\varepsilon}^2\rho_{k,n}^2 \cdot  {\Big(}\frac{{\trace(\hat{S}_n \psi^*_{\leq k} \Lambda^{\leq k}_{\mathcal{A}^{-2} \Sigma^{-\beta'}} \psi_{\leq k} \hat{S}^*_n)}}{{\mu_k(\psi_{\leq k} \hat{S}^*_n \hat{S}_n \psi^*_{\leq k})^2}}
    + \frac{\overbrace{\trace(\hat{S}_n \psi^*_{>k} \Lambda^{>k}_{\mathcal{A}^2 \Sigma^{-\beta' + 2\beta}} \psi_{>k} \hat{S}^*_n)}^{\text{effective rank}}}{n^2 {\| \tilde{\Sigma}_{>k} \|^2}} {\Big)}.
    \end{align*}
\end{proof}
\begin{lemma}[Simplified Upper bound for variance using concentration]
    \label{lemma:concentration_variance}
    There exists some absolute constant $c, c', C_1 > 0$ s.t. for any $k \in \mathbb{N}$ with $c \beta_k k \log(k) \leq n$, it holds w.p. at least $1 - 8 \exp(\frac{-c'}{\beta_k^2} \frac{n}{k})$, the variance can be upper bounded as:
    $$
    V \leq C_1 \sigma_{\varepsilon}^2 \rho_{k,n}^2  {\Big(}\frac{{  \sum_{i\leq k}  p_i^{-2} \lambda_i^{-\beta'}}}{{n}}
    + \frac{ \sum_{i>k} p_i^2 \lambda_i^{-\beta' + 2\beta}}{n {\| \tilde{\Sigma}_{>k} \|^2}} {\Big)}.
    $$
\end{lemma}
\begin{proof}
    By Theorem \ref{theorem:variance_proof}, we have
    \begin{align*}
        V &\leq \sigma_{\varepsilon}^2\rho_{k,n}^2 \cdot  {\Big(}\frac{{\trace(\hat{S}_n \psi^*_{\leq k} \Lambda^{\leq k}_{\mathcal{A}^{-2} \Sigma^{-\beta'}} \psi_{\leq k} \hat{S}^*_n)}}{{\mu_k(\psi_{\leq k} \hat{S}^*_n \hat{S}_n \psi^*_{\leq k})^2}}
    + \frac{\overbrace{\trace(\hat{S}_n \psi^*_{>k} \Lambda^{>k}_{\mathcal{A}^2 \Sigma^{-\beta' + 2\beta}} \psi_{>k} \hat{S}^*_n)}^{\text{effective rank}}}{n^2 {\| \tilde{\Sigma}_{>k} \|^2}} {\Big)}.
    \end{align*}
    Then we can apply concentration inequalities, by Lemma \ref{lemma:simultaneous_concentration}, it holds w.p. at least $1 - 8 \exp(\frac{-c'}{\beta_k^2} \frac{n}{k})$ that
    \begin{align*}
        V &\leq \sigma_{\varepsilon}^2\rho_{k,n}^2 \cdot  {\Big(}\frac{{c_2 n \sum_{i\leq k}  p_i^{-2} \lambda_i^{-\beta'}}}{{c_1^2 n^2}}
    + \frac{c_2 n \sum_{i>k} p_i^2 \lambda_i^{-\beta' + 2\beta}}{n^2 {\| \tilde{\Sigma}_{>k} \|^2}} {\Big)} \\
    &\leq \sigma_{\varepsilon}^2 \rho_{k,n}^2 \max \{ \frac{c_2}{c_1^2}, c_2\} {\Big(}\frac{{  \sum_{i\leq k}  p_i^{-2} \lambda_i^{-\beta'}}}{{n}}
    + \frac{ \sum_{i>k} p_i^2 \lambda_i^{-\beta' + 2\beta}}{n {\| \tilde{\Sigma}_{>k} \|^2}} {\Big)}.
    \end{align*}
    Then we take $C_1$ to be $\max \{ \frac{c_2}{c_1^2}, c_2\}$ to obtain the desired bound.
\end{proof}

\section{Upper bound for the Bias}
\label{appendix:ub_bias}
\begin{lemma}[Upper bound for bias]
    \label{lemma:intermediate_ub_bias}
    Suppose that for some $k < n$, the matrix $\tilde{K}_{>k}^{\gamma}$ is positive-definite, then
    % and $\mathcal{A}^2 \Sigma^{\beta'}$ has decaying eigenvalue i.e. $2p+\lambda \beta' > 0 $
    \begin{align*}
        B \leq& \ 3 \Bigg(\frac{\mu_1((\tilde{K}_{>k}^{\gamma})^{-1} )^2 }{\mu_n((\tilde{K}_{>k}^{\gamma})^{-1} )^2 } \frac{\mu_1(\psi_{\leq k } \hat{S}_n^* \hat{S}_n \psi^*_{\leq k} )}{\mu_k( \psi_{\leq k}\hat{S}_n^*   \hat{S}_n \psi^*_{\leq k} )^2 \mu_k( \Lambda^{\leq k}_{\mathcal{A}^2 \Sigma^{\beta'}} )} \|\hat{S}_n \mathcal{A}_{>k} f^{*}_{>k}\|^2 \\
    +& \ \frac{\| \phi_{\leq k} f^{*}_{\leq k}\|^2_{\Lambda^{\leq k}_{\mathcal{A}^{-2}\Sigma^{1 - 2\beta}}}}{\mu_n((\tilde{K}_{>k}^{\gamma})^{-1} )^2 \mu_k( \psi_{\leq k}\hat{S}_n^*   \hat{S}_n \psi^*_{\leq k} )^2 \mu_k( \Lambda^{\leq k}_{\mathcal{A}^2 \Sigma^{\beta'}} )} \\
    +& \ \|\phi_{>k} f^*_{>k}\|^2_{\Lambda^{>k}_{\Sigma^{1-\beta'}}} \\
    +& \ \| \Lambda^{>k}_{\Sigma^{1 - \beta'}} \| \ \mu_1[(\tilde{K}_{>k}^
\gamma)^{-1}]^2 \|\hat{S}_n \mathcal{A}_{>k} f_{>k}\|^2 \mu_1( \underbrace{\hat{S}_n \psi_{>k}^* \Lambda^{>k}_{\mathcal{A}^2 \Sigma^{2\beta - 1}} \psi_{>k} \hat{S}^*_n}_{n \times n}) \\
    +& \| \Lambda^{>k}_{\Sigma^{-\beta'+\beta}}\| \frac{\mu_1((\tilde{K}^{\gamma}_{>k})^{-1})}{\mu_n((\tilde{K}^{\gamma}_{>k})^{-1})^2} \frac{\mu_1(\psi_{\leq k} \hat{S}^*_n  \hat{S}_n \psi^*_{\leq k}) }{\mu_k(\psi_{\leq k} \hat{S}^*_n  \hat{S}_n \psi^*_{\leq k} )^2} \| \phi_{\leq k} f^*_{\leq k}\|_{\Lambda^{\leq k}_{\mathcal{A}^{-2} \Sigma^{1 - 2\beta}}}\Bigg).
    % +& \ \mu_1( \hat{S}_n \psi^*_{>k}  \Lambda^{>k}_{\mathcal{A}^2 \Sigma^{(-\beta' + 2\beta )}} \psi_{>k}\hat{S}_n^*  ) \ \mu_1((\tilde{K}^{\gamma}_{>k})^{-1})^2 \\ 
    % &\frac{\mu_1(\psi_{\leq k} \hat{S}_n^* \hat{S}_n \psi^*_{\leq k})}{ \mu_k( \Lambda^{\leq k}_{\mathcal{A}\Sigma^{\beta/2}}\psi_{\leq k} \hat{S}_n^* \hat{S}_n \psi^*_{\leq k} \Lambda^{\leq k}_{\mathcal{A}\Sigma^{\beta/2}})^2 \mu_n((\tilde{K}^{\gamma})^{-1})^2} \|\Lambda^{\leq k}_{\mathcal{A} \Sigma^{1/2}} \phi_{\leq k} f_{\leq k}\|^2)\\
    \end{align*}
\end{lemma}
\begin{proof}
    Similar as variance, by lemma \ref{lemma:separate_smaller_k_large_eq_k} we can bound $\leq k$ and $>k$ separately, for brevity we define the error vector $\xi := \phi (\hat{f}(\hat{S}_n \mathcal{A} f^*) - f^*) \in \mathbb{R}^{\infty}$, by lemma \ref{lemma:equiv_sobolev_matrix_norm} we can bound $\|\xi_{\leq k}\|_{\Sigma^{1 - \beta'}}$ and $\|\xi_{>k} \|_{\Sigma^{1 - \beta'}}$ separately.

    \noindent We first discuss $\|\xi_{\leq k}\|_{\Sigma^{1 - \beta'}}$, by lemma \ref{lemma:separate_search}, we have
    \begin{equation}
        \begin{aligned}
            \phi_{\leq k} \hat{f}(\hat{S}_n \mathcal{A} f^*) + \phi_{\leq k} \mathcal{A}_{\leq k} \Sigma_{\leq k}^{\beta - 1} \hat{S}_n^*(\tilde{K}_{>k}^{\gamma})^{-1}\hat{S}_n \mathcal{A}\hat{f}(\hat{S}_n \mathcal{A} f^*)_{\leq k} = \phi_{\leq k} \mathcal{A}_{\leq k} \Sigma_{\leq k}^{\beta - 1} \hat{S}_n^*(\tilde{K}_{>k}^{\gamma})^{-1}\hat{S}_n \mathcal{A}f^*.\label{eq:biasproof}
        \end{aligned}
    \end{equation}

    By definition of $\xi$, we have $\xi_{\leq k} 
 = \phi_{\leq k} (\hat{f} - f^*) = \phi_{\leq k} \hat{f}_{\leq k} - \phi_{\leq k} f^*_{\leq k}$, so we have $\phi_{\leq k} \hat{f} = \xi_{\leq k} + \phi_{\leq k} f^*_{\leq k}$.
    \begin{align*}
        \text{LHS of 
 (\ref{eq:biasproof})} =& \ \xi_{\leq k} + \phi_{\leq k} f^{*}_{\leq k} + \phi_{\leq k} \mathcal{A}_{\leq k} \Sigma_{\leq k}^{\beta - 1} \hat{S}_n^*(\tilde{K}_{>k}^{\gamma})^{-1}\hat{S}_n \phi^{*}_{\leq k} \Lambda_{\mathcal{A}}^{\leq k} (\xi_{\leq k} + \phi_{\leq k} f^{*}_{\leq k}) \\
%&\text{\color{gray} (Here we diagonalize the operator and convert $\phi_{\leq k} \hat{f}$ to the RHS of the identity)}\\
        =& \ \xi_{\leq k} + \phi_{\leq k} f^{*}_{\leq k} + \phi_{\leq k} \mathcal{A}_{\leq k} \Sigma_{\leq k}^{\beta - 1} \hat{S}_n^*(\tilde{K}_{>k}^{\gamma})^{-1}\hat{S}_n \phi^{*}_{\leq k} \Lambda_{\mathcal{A}}^{\leq k} \xi_{\leq k} \\
        +&  \underbrace{\phi_{\leq k} \mathcal{A}_{\leq k} \Sigma_{\leq k}^{\beta - 1} \hat{S}_n^*(\tilde{K}_{>k}^{\gamma})^{-1}\hat{S}_n \phi^{*}_{\leq k} \Lambda_{\mathcal{A}}^{\leq k} \phi_{\leq k} f^{*}_{\leq k}}_{\text{(*)}}.
    \end{align*}
    And
    \begin{align*}
        \text{RHS of 
 (\ref{eq:biasproof})} =& \ \phi_{\leq k} \mathcal{A}_{\leq k} \Sigma_{\leq k}^{\beta - 1} \hat{S}_n^*(\tilde{K}_{>k}^{\gamma})^{-1}\hat{S}_n (\phi_{\leq k}^* \Lambda^{\leq k}_{\mathcal{A}} \phi_{\leq k} f^{*}_{\leq k} + \phi_{> k}^* \Lambda^{>k}_{\mathcal{A}} \phi_{> k} f^{*}_{> k}) \\
        =& \ \underbrace{\phi_{\leq k} \mathcal{A}_{\leq k} \Sigma_{\leq k}^{\beta - 1} \hat{S}_n^*(\tilde{K}_{>k}^{\gamma})^{-1}\hat{S}_n \phi_{\leq k}^* \Lambda^{\leq k}_{\mathcal{A}} \phi_{\leq k} f^{*}_{\leq k}}_{\text{(*)}} \\
        +& \phi_{\leq k} \mathcal{A}_{\leq k} \Sigma_{\leq k}^{\beta - 1} \hat{S}_n^*(\tilde{K}_{>k}^{\gamma})^{-1}\hat{S}_n 
 \phi_{> k}^* \Lambda^{>k}_{\mathcal{A}} \phi_{> k} f^{*}_{> k} .
    \end{align*}
 %    \yplu{what does $\phi^\ast$ here means 
 % population or transpose? I think it means transpose, but on $f$ $\ast$ means population and empirical. This may lead to confusion.}
    
    The two (*) terms get cancelled out, therefore
    \begin{align*}
        \xi_{\leq k} &+ \phi_{\leq k} \mathcal{A}_{\leq k} \Sigma_{\leq k}^{\beta - 1} \hat{S}_n^*(\tilde{K}_{>k}^{\gamma})^{-1}\hat{S}_n \phi^{*}_{\leq k} \Lambda_{\mathcal{A}}^{\leq k} \xi_{\leq k} \\
        &= \phi_{\leq k} \mathcal{A}_{\leq k} \Sigma_{\leq k}^{\beta - 1} \hat{S}_n^*(\tilde{K}_{>k}^{\gamma})^{-1}\hat{S}_n \phi^{*}_{> k} \Lambda_{\mathcal{A}}^{> k} \phi_{>k} f^{*}_{>k} - \phi_{\leq k} f^{*}_{\leq k}.
    \end{align*}
    We multiply $\xi_{\leq k}^T\Lambda^{\leq k}_{\mathcal{A}^{-1}\Sigma^{1-\beta-\beta'/2}}$ in both sides and since $\|\xi_{\leq k}\|^2_{\Lambda^{\leq k}_{\mathcal{A}^{-1}\Sigma^{1-\beta-\beta'/2}}} \geq 0$,
    \begin{align*}
        &\xi_{\leq k}^T\Lambda^{\leq k}_{\mathcal{A}^{-1}\Sigma^{1-\beta-\beta'/2}} \phi_{\leq k} \mathcal{A}_{\leq k} \Sigma_{\leq k}^{\beta - 1} \hat{S}_n^*(\tilde{K}_{>k}^{\gamma})^{-1}\hat{S}_n \phi^{*}_{\leq k} \Lambda_{\mathcal{A}}^{\leq k} \xi_{\leq k} \\
        \leq& \ \xi_{\leq k}^T\Lambda^{\leq k}_{\mathcal{A}^{-1}\Sigma^{1-\beta-\beta'/2}} \phi_{\leq k} \mathcal{A}_{\leq k} \Sigma_{\leq k}^{\beta - 1} \hat{S}_n^*(\tilde{K}_{>k}^{\gamma})^{-1}\hat{S}_n \phi^{*}_{> k} \Lambda_{\mathcal{A}}^{> k} \phi_{>k} f^{*}_{>k} - \xi_{\leq k}^T\Lambda^{\leq k}_{\mathcal{A}^{-1}\Sigma^{1-\beta-\beta'/2}} \phi_{\leq k} f^{*}_{\leq k}.
    \end{align*}
    LHS is the quadratic term w.r.t. $\xi_{\leq k}$ and RHS is the linear term w.r.t. $\xi_{\leq k}$, similar to Variance case, we lower bound LHS and upper bound RHS respectively.

    \begin{align*}
\text{LHS} &= \overbrace{\xi_{\leq k}^T}^{1 \times k} \overbrace{\Lambda^{\leq k}_{\Sigma^{-\beta'/2}}}^{k \times k}\overbrace{ \phi_{\leq k} \hat{S}_n^*}^{k \times n} \overbrace{  (\tilde{K}_{>k}^{\gamma})^{-1}}^{n \times n}\overbrace{ \hat{S}_n \phi^{*}_{\leq k}}^{n \times k} \overbrace{\Lambda_{\mathcal{A}}^{\leq k}}^{k \times k} \overbrace{\xi_{\leq k}}^{k \times 1}\\
&= \xi_{\leq k}^T \Lambda^{\leq k}_{\Sigma^{(1-\beta')/2}} \psi_{\leq k}\hat{S}_n^* (\tilde{K}_{>k}^{\gamma})^{-1}  \hat{S}_n \psi^*_{\leq k} \Lambda^{\leq k}_{\mathcal{A} \Sigma^{1/2}} \xi_{\leq k}. \\
\end{align*}
Since $(1-\beta') + \beta'/2 = (1 - \beta')/2 + 1/2$, it can be lower bounded by 
\begin{align*}
& \mu_n((\tilde{K}_{>k}^{\gamma})^{-1} ) \ (\xi_{\leq k}^T \Lambda^{\leq k}_{\Sigma^{1 - \beta'}}\xi_{\leq k}) \mu_k\left( \psi_{\leq k}\hat{S}_n^*   \hat{S}_n \psi^*_{\leq k} \right) \mu_k\left( \Lambda^{\leq k}_{\mathcal{A} \Sigma^{\beta'/2}} \right) \\
&= \|\xi_{\leq k}\|_{\Lambda^{\leq k}_{\Sigma^{1 - \beta'}}}^2 \mu_n((\tilde{K}_{>k}^{\gamma})^{-1} ) \mu_k\left( \psi_{\leq k}\hat{S}_n^*   \hat{S}_n \psi^*_{\leq k} \right) \mu_k\left( \Lambda^{\leq k}_{\mathcal{A} \Sigma^{\beta'/2}} \right).
\end{align*}
Next we upper bound RHS, first we bound the first term in RHS
\begin{align*}
    \text{First term in RHS} &= \xi_{\leq k}^T\Lambda^{\leq k}_{\mathcal{A}^{-1}\Sigma^{1-\beta-\beta'/2}} \phi_{\leq k} \mathcal{A}_{\leq k} \Sigma_{\leq k}^{\beta - 1} \hat{S}_n^*(\tilde{K}_{>k}^{\gamma})^{-1}\hat{S}_n \phi^{*}_{> k} \Lambda_{\mathcal{A}}^{> k} \phi_{>k} f^{*}_{>k}\\
    &= \xi_{\leq k}^T \Lambda^{\leq k}_{\Sigma^{-\beta'/2}} \phi_{\leq k} \hat{S}_n^*(\tilde{K}_{>k}^{\gamma})^{-1} \hat{S}_n \phi^{*}_{> k} \Lambda_{\mathcal{A}}^{> k} \phi_{>k} f^{*}_{>k}.\\
\end{align*}
Since $(1 - \beta')/2 - 1/2  = -\beta'/2 $, it equals to
\begin{align*}
    &\xi_{\leq k}^T \Lambda^{\leq k}_{\Sigma^{(1-\beta')/2}} \Lambda^{\leq k}_{\Sigma^{-1/2}} \phi_{\leq k} \hat{S}_n^*(\tilde{K}_{>k}^{\gamma})^{-1} \hat{S}_n \mathcal{A}_{>k} f^{*}_{>k} \\
    &= \xi_{\leq k}^T \Lambda^{\leq k}_{\Sigma^{(1-\beta')/2}} \psi_{\leq k} \hat{S}_n^*(\tilde{K}_{>k}^{\gamma})^{-1} \hat{S}_n \mathcal{A}_{>k} f^{*}_{>k}\\
    &\leq \|\xi_{\leq k}\|_{\Lambda^{\leq k}_{\Sigma^{(1 - \beta')}}} \mu_1((\tilde{K}_{>k}^{\gamma})^{-1}) \sqrt{\mu_1(\underbrace{ \psi_{\leq k } \hat{S}_n^* \hat{S}_n \psi^*_{\leq k} }_{k \times k})} \| \hat{S}_n \mathcal{A}_{>k} f^{*}_{>k} \|.
\end{align*}
Then we bound the second term in RHS.
\begin{align*}
\text{Second term in RHS} &= \xi_{\leq k}^T\Lambda^{\leq k}_{\mathcal{A}^{-1}\Sigma^{1-\beta-\beta'/2}} \phi_{\leq k} f_{\leq k}^* = \xi_{\leq k}^T \Lambda^{\leq k}_{\Sigma^{(1 - \beta')/2}} \Lambda^{\leq k}_{\mathcal{A}^{-1}\Sigma^{1/2-\beta}} \phi_{\leq k} f^*_{\leq k} \\
&\leq \|\xi_{\leq k}\|_{\Lambda^{\leq k}_{\Sigma^{1 - \beta'}}} \| \phi_{\leq k} f^{*}_{\leq k}\|_{\Lambda^{\leq k}_{\mathcal{A}^{-2}\Sigma^{1-2\beta}}}.
\end{align*}
Therefore, gather the terms we have
\begin{align*}
    &\|\xi_{\leq k}\|_{\Lambda^{\leq k}_{\Sigma^{1 - \beta'}}}^2 \mu_n((\tilde{K}_{>k}^{\gamma})^{-1} ) \mu_k\left(\Lambda^{\leq k}_{\mathcal{A}^{1/2} \Sigma^{\beta'/4}} \psi_{\leq k}\hat{S}_n^* (\tilde{K}_{>k}^{\gamma})^{-1}  \hat{S}_n \psi^*_{\leq k} \Lambda^{\leq k}_{\mathcal{A}^{1/2} \Sigma^{\beta'/4}}\right)  \\
    \leq&  \ \ \|\xi_{\leq k}\|_{\Lambda^{\leq k}_{\Sigma^{(1 - \beta')}}} \mu_1((\tilde{K}_{>k}^{\gamma})^{-1}) \sqrt{\mu_1(\underbrace{ \psi_{\leq k } \hat{S}_n^* \hat{S}_n \psi^*_{\leq k} }_{k \times k})} \| \hat{S}_n \mathcal{A}_{>k} f^{*}_{>k} \|\\
    &+ \  \|\xi_{\leq k}\|_{\Lambda^{\leq k}_{\Sigma^{1 - \beta'}}} \| \phi_{\leq k} f^{*}_{\leq k}\|_{\Lambda^{\leq k}_{\mathcal{A}^{-2}\Sigma^{1-2\beta}}}.
\end{align*}
So
\begin{align*}
    \|\xi_{\leq k}\|_{\Lambda^{\leq k}_{\Sigma^{1 - \beta'}}} &\leq \frac{\mu_1((\tilde{K}_{>k}^{\gamma})^{-1} ) }{\mu_n((\tilde{K}_{>k}^{\gamma})^{-1} ) } \frac{\sqrt{\mu_1(\psi_{\leq k } \hat{S}_n^* \hat{S}_n \psi^*_{\leq k} )}}{\mu_k\left( \psi_{\leq k}\hat{S}_n^*   \hat{S}_n \psi^*_{\leq k} \right) \mu_k\left( \Lambda^{\leq k}_{\mathcal{A} \Sigma^{\beta'/2}} \right)} \|\hat{S}_n \mathcal{A}_{>k} f^{*}_{>k}\| \\
    &+ \frac{\| \phi_{\leq k} f^{*}_{\leq k}\|_{\Lambda^{\leq k}_{\mathcal{A}^{-2}\Sigma^{1-2\beta}}}}{\mu_n((\tilde{K}_{>k}^{\gamma})^{-1} ) \mu_k\left( \psi_{\leq k}\hat{S}_n^*   \hat{S}_n \psi^*_{\leq k} \right) \mu_k\left( \Lambda^{\leq k}_{\mathcal{A} \Sigma^{\beta'/2}} \right)}.
\end{align*}

\noindent By $\|a+b\|^2 \leq 2(\|a\|^2 + \|b\|^2)$, we can bound $\|\xi_{\leq k}\|^2_{\Sigma^{1 - \beta'}}$ by 
% \fh{2 is enough. why use 3?} \honam{In $>k$ part the constant is 3, so I use 3 here to make things look consistent, I think it would not affect the bound since it is just a constant?}\fh{yes, it is just a constant. but everyone will be confusing about this when reading here.}

\begin{align*}
    &2 \Bigg(\frac{\mu_1((\tilde{K}_{>k}^{\gamma})^{-1} )^2 }{\mu_n((\tilde{K}_{>k}^{\gamma})^{-1} )^2 } \frac{\mu_1(\overbrace{\psi_{\leq k } \hat{S}_n^* \hat{S}_n \psi^*_{\leq k} }^{k \times k})}{\mu_k( \psi_{\leq k}\hat{S}_n^*   \hat{S}_n \psi^*_{\leq k} )^2 \mu_k( \Lambda^{\leq k}_{\mathcal{A}^2 \Sigma^{\beta'}} )} \|\hat{S}_n \mathcal{A}_{>k} f^{*}_{>k}\|^2 \\
    &+ \frac{\| \phi_{\leq k} f^{*}_{\leq k}\|^2_{\Lambda^{\leq k}_{\mathcal{A}^{-2}\Sigma^{1-2\beta}}}}{\mu_k( \psi_{\leq k}\hat{S}_n^*   \hat{S}_n \psi^*_{\leq k} )^2 \mu_k( \Lambda^{\leq k}_{\mathcal{A}^2 \Sigma^{\beta'}} )} \Bigg).
\end{align*}

%\yplu{!}
\noindent Now we discuss the $>k$ case, which is more complicated, we bound it by three quantities by the fact that $(A+B+C)^2 \leq 3(A^2 + B^2 +C^2)$ and bound them respectively as follows

%\yplu{a small comment $\phi_{>k}f^*$ should be the same as $\phi_{>k}f^*_{>k}$?} \honam{Yes, write like this is just for readability, we maybe should mention this in the paper also to avoid confusing the reviewers}
\begin{align*}
&\| \phi_{>k} f^{*}_{>k} - \phi_{>k} \mathcal{A}_{>k} \Sigma_{>k}^{\beta - 1} \hat{S}_n^* (\tilde{K}^\gamma)^{-1} \hat{S}_n \mathcal{A} f^* \|_{\Lambda_{\Sigma^{1 - \beta'}}^{>k}}^2\\
\leq& 3 (\| \phi_{>k} f^{*}_{>k} \|_{\Lambda_{\Sigma^{1 - \beta'}}^{>k}}^2 + \|\phi_{>k} \mathcal{A}_{>k} \Sigma_{>k}^{\beta - 1} \hat{S}_n^* (\tilde{K}^\gamma)^{-1} \hat{S}_n \mathcal{A}_{>k} f^*_{>k} \|_{\Lambda_{\Sigma^{1 - \beta'}}^{>k}}^2 + \|\phi_{>k} \mathcal{A}_{>k} \Sigma_{>k}^{\beta - 1} \hat{S}_n^* (\tilde{K}^\gamma)^{-1} \hat{S}_n \mathcal{A}_{\leq k} f^*_{\leq k} \|_{\Lambda_{\Sigma^{1 - \beta'}}^{>k}}^2).
\end{align*} 

We first bound the second term
\begin{align*}
&\|\phi_{>k} \mathcal{A}_{>k} \Sigma_{>k}^{\beta - 1} \hat{S}_n^* (\tilde{K}^\gamma)^{-1} \hat{S}_n \mathcal{A}_{>k} f^*_{>k} \|_{\Lambda_{\Sigma^{1 - \beta'}}^{>k}}^2 \\
\leq& \ \|\Lambda^{>k}_{\Sigma^{1 - \beta'}}\| \  \|\phi_{>k} \mathcal{A}_{>k} \Sigma_{>k}^{\beta - 1} \hat{S}_n^* (\tilde{K}^\gamma)^{-1} \hat{S}_n \mathcal{A}_{>k} f^*_{>k} \|^2 \\
=& \ \| \Lambda^{>k}_{\Sigma^{1 - \beta'}} \| \ \| \Lambda^{>k}_{\mathcal{A} \Sigma^{\beta-1}} \phi_{>k} \hat{S}_n^* (\tilde{K}^\gamma)^{-1} \hat{S}_n \phi_{>k}^* \Lambda_{\mathcal{A}}^{>k} \phi_{>k} f^*_{>k}\|^2 \\
\leq& \ \| \Lambda^{>k}_{\Sigma^{1 - \beta'}} \| \ \mu_1[(\tilde{K}^
\gamma)^{-1}]^2 \|\hat{S}_n \mathcal{A}_{>k} f^*_{>k}\|^2 \mu_1( \underbrace{\hat{S}_n \phi_{>k}^* \Lambda^{>k}_{\mathcal{A}^2 \Sigma^{2(\beta - 1)}} \phi_{>k} \hat{S}^*_n}_{n \times n}) \\
\leq& \ \| \Lambda^{>k}_{\Sigma^{1 - \beta'}} \| \ \mu_1[(\tilde{K}_{>k}^
\gamma)^{-1}]^2 \|\hat{S}_n \mathcal{A}_{>k} f^*_{>k}\|^2 \mu_1( \underbrace{\hat{S}_n \phi_{>k}^* \Lambda^{>k}_{\mathcal{A}^2 \Sigma^{2(\beta - 1)}} \phi_{>k} \hat{S}^*_n}_{n \times n}).
% =& \ \| \Lambda^{>k}_{\Sigma^{1 - \beta'}} \| \ (\phi_{>k} f^*_{>k})^T \Lambda_{\mathcal{A}}^{>k} (\hat{S}_n \phi_{>k}^*)^T (\tilde{K}^\gamma)^{-1} (\phi_{>k} \hat{S}_n^*)^T \\
%  &\Lambda^{>k}_{\mathcal{A}^2 \Sigma^{2(\beta-1)}}\phi_{>k} \hat{S}_n^* (\tilde{K}^\gamma)^{-1} \hat{S}_n \phi_{>k}^*  \Lambda_{\mathcal{A}}^{>k}(\phi_{>k} f^*_{>k}) \\
% \leq& \ \| \Lambda^{>k}_{\Sigma^{1 - \beta'}} \| \ \|\hat{S}_n \mathcal{A}_{>k} f^*_{>k}\|^2 \mu_1[(\tilde{K}^\gamma)^{-1} \underbrace{ (\phi_{>k} \hat{S}_n^*)^T \Lambda^{>k}_{\mathcal{A}^2 \Sigma^{2(\beta-1)}}(\phi_{>k} \hat{S}_n^*)}_{n \times n} (\tilde{K}^\gamma)^{-1}] 
\end{align*}
The last inequality is by $\mu_1((\tilde{K}_{>k}^{\gamma})^{-1}) \geq \mu_1((\tilde{K}^{\gamma})^{-1})$.
%\yplu{does here you use the $AB$ and $BA$ has the same eigenvalue. \cite{Horn_Johnson_1985}[Theorem 1.3.20]} \honam{I think no here?}

%\yplu{some of the term you write using $f^\ast$, some of the term you use is $f$} \honam{Notation issue, need to fix later}
% \honam{How to bound this $\mu_1$ term? in \cite{barzilai2023generalization}, the $n \times n$ matrix above just "cancel out" with the $(\tilde{K}^{\gamma})^{-1}$, see bottom of Page 37 in \cite{barzilai2023generalization}}
% \fh{it can be upper bounded by $\frac{\mbox{see the comment on Lemma B.2}}{[\mu_n(\tilde{K}^\gamma)]^2}$, which will depend on the new defined $\beta_k$} \honam{I still don't get how to handle the $\tilde{K}^{\gamma}$ here, do you mean upper bound the $\mu_1$ term by $\mu_n(\tilde{K}^{\gamma})^2 \times \mu_1(\underbrace{ (\phi_{>k} \hat{S}_n^*)^T \Lambda^{>k}_{\mathcal{A}^2 \Sigma^{2(\beta-1)}}(\phi_{>k} \hat{S}_n^*)}_{n \times n})$}\fh{yes, but firstly it's $1/\mu_n(\tilde{K}^{\gamma})^2$. Then to bound the current $\mu_1$, it admits $\mu_1(A^{\top}BA) \leq \mu_1(A^{\top}A) \mu_1(B)$ for PSD matrices.}
% \\
% \\
% \honam{to-do: to finish bounding this $\mu_1$ term}

% \honam{Maybe we can use modification of Ostrowski's Theorem \ref{lemma:ostrowski} here and also at the end of the proof, check tomorrow }

Then we move on to bound the third term, that is, we want to bound
\begin{align*}
    &\| \phi_{>k} \mathcal{A}_{>k} \Sigma^{\beta - 1}_{>k} \hat{S}_n^* (\tilde{K}^\gamma)^{-1} \hat{S}_n \mathcal{A}_{\leq k} f^*_{\leq k}\|^2_{\Lambda^{>k}_{\Sigma^{1 - \beta'}}} \\
    =& \ \| \Lambda^{>k}_{\mathcal{A} \Sigma^{\beta - 1}} \phi_{>k} \hat{S}_n^* (\tilde{K}^\gamma)^{-1} \hat{S}_n \phi^*_{\leq k} \Lambda^{\leq k}_{\mathcal{A}} \phi_{\leq k} f^*_{\leq k}\|^2_{\Lambda^{>k}_{\Sigma^{1 - \beta'}}}.
\end{align*}

\noindent First we deal with $(\tilde{K}^\gamma)^{-1} (\hat{S}_n \phi_{\leq k}^*)$ first, we can write it as

$$
(\tilde{K}^\gamma)^{-1} (\hat{S}_n \phi_{\leq k}^*) = (\tilde{K}_{>k}^\gamma + (\hat{S}_n \phi_{\leq k}^*) \Lambda^{\leq k}_{\mathcal{A}^2 \Sigma^{\beta - 1}} (\phi_{\leq k} \hat{S}_n^*))^{-1} (\hat{S}_n \phi_{\leq k}^*),
$$
then apply \ref{lemma:extension_sherman_morrison_woodbury} with $A = \tilde{K}_{>k}^\gamma$, $U = \hat{S}_n \phi_{\leq k}^*$, $C = \Lambda^{\leq k}_{\mathcal{A}^2 \Sigma^{\beta - 1}}$, $V = \phi_{\leq k} \hat{S}_n^*$, we have it equal to 
$$
(\tilde{K}_{>k}^\gamma)^{-1} (\hat{S}_n \phi_{\leq k}^*) (I_k + \Lambda^{\leq k}_{\mathcal{A}^2 \Sigma^{\beta - 1}} (\phi_{\leq k} \hat{S}_n^*) (\tilde{K}_{>k}^\gamma)^{-1} (\hat{S}_n \phi_{\leq k}^*))^{-1} .
$$
Then we sub. the identity above to obtain
\begin{align*}
    &\| \Lambda^{>k}_{\mathcal{A} \Sigma^{\beta - 1}} \phi_{>k} \hat{S}_n^* (\tilde{K}^\gamma)^{-1} \hat{S}_n \phi^*_{\leq k} \Lambda^{\leq k}_{\mathcal{A}} \phi_{\leq k} f_{\leq k}\|^2_{\Lambda^{>k}_{\Sigma^{1 - \beta'}}} \\
    =& \ \| \Lambda^{>k}_{\Sigma^{(1 - \beta')/2}} \Lambda^{>k}_{\mathcal{A} \Sigma^{\beta - 1}} \phi_{>k} \hat{S}_n  (\tilde{K}^{\gamma}_{>k})^{-1}\hat{S}_n \phi_{\leq k}^* (I_k + \Lambda^{\leq k}_{\mathcal{A}^2 \Sigma^{\beta - 1}} \phi_{\leq k} \hat{S}^*_n (\tilde{K}^{\gamma}_{>k})^{-1} \hat{S}_n \phi^*_{\leq k})^{-1} \Lambda^{\leq k}_{\mathcal{A}} \phi_{\leq k} f^*_{\leq k}\|^2  \\
    =& \ \| \Lambda^{>k}_{\mathcal{A} \Sigma^{(-\beta' + 2\beta - 1)/2}} \phi_{>k} \hat{S}_n^*  (\tilde{K}^{\gamma}_{>k})^{-1}\hat{S}_n \phi_{\leq k}^* (\Lambda^{\leq k}_{\mathcal{A}^2 \Sigma^{\beta - 1/2}} (\Lambda^{\leq k}_{\mathcal{A}^{-2} \Sigma^{-\beta}} +  \psi_{\leq k} \hat{S}^*_n (\tilde{K}^{\gamma}_{>k})^{-1} \hat{S}_n \psi^*_{\leq k}) \Lambda^{\leq k}_{\Sigma^{1/2}})^{-1} \Lambda^{\leq k}_{\mathcal{A}} \phi_{\leq k} f^*_{\leq k}\|^2  \\
    =& \ \| \Lambda^{>k}_{\mathcal{A} \Sigma^{(-\beta' + 2\beta - 1)/2}} \phi_{>k} \hat{S}_n^*  (\tilde{K}^{\gamma}_{>k})^{-1}\hat{S}_n \phi_{\leq k}^* \Lambda^{\leq k}_{\Sigma^{-1/2}} (\Lambda^{\leq k}_{\mathcal{A}^{-2} \Sigma^{-\beta}} +  \psi_{\leq k} \hat{S}^*_n (\tilde{K}^{\gamma}_{>k})^{-1} \hat{S}_n \psi^*_{\leq k}) ^{-1} \Lambda^{\leq k}_{\mathcal{A}^{-2} \Sigma^{1/2 - \beta}} \Lambda^{\leq k}_{\mathcal{A}} \phi_{\leq k} f^*_{\leq k}\|^2  \\
    =& \ \| \underbrace{\Lambda^{>k}_{\mathcal{A} \Sigma^{(-\beta' + 2\beta)/2}} \psi_{>k} \hat{S}_n^*  (\tilde{K}^{\gamma}_{>k})^{-1/2}}_{(1)} \underbrace{(\tilde{K}^{\gamma}_{>k})^{-1/2}}_{(2)} \underbrace{\hat{S}_n \psi_{\leq k}^*}_{(3)}\underbrace{  (\Lambda^{\leq k}_{\mathcal{A}^{-2} \Sigma^{-\beta}} +  \psi_{\leq k} \hat{S}^*_n (\tilde{K}^{\gamma}_{>k})^{-1} \hat{S}_n \psi^*_{\leq k}) ^{-1}}_{(4)} \underbrace{\Lambda^{\leq k}_{\mathcal{A}^{-1} \Sigma^{1/2 - \beta}} \phi_{\leq k} f^*_{\leq k}}_{(5)}\|^2 . \\
\end{align*}
Above can be bounded by %\honam{(4) unsure, but i think should be correct and not affect the decay}
\begin{align*}
    &\underbrace{\| (\tilde{K}^{\gamma}_{>k})^{-1/2} \hat{S}_n \psi_{>k}^* \Lambda^{>k}_{\mathcal{A}^2 \Sigma^{-\beta'+2\beta}} \psi_{>k} \hat{S}_n^* (\tilde{K}^{\gamma}_{>k})^{-1/2}\|}_{(1)} \underbrace{\mu_1((\tilde{K}^{\gamma}_{>k})^{-1})}_{(2)} \\ &\underbrace{\mu_1(\psi_{\leq k} \hat{S}_n^* \hat{S}_n \psi_{\leq k}^*)}_{(3)} \underbrace{\mu_1((\psi_{\leq k} \hat{S}^*_n (\tilde{K}^{\gamma}_{>k})^{-1} \hat{S}_n \psi^*_{\leq k}) ^{-1})^2}_{(4)} \underbrace{\| \phi_{\leq k} f^*_{\leq k}\|_{\Lambda^{\leq k}_{\mathcal{A}^{-2} \Sigma^{1 - 2\beta}}}}_{(5)}.
\end{align*}
For (1) it can be upper bounded by
\begin{align*}
    &\| (\tilde{K}^{\gamma}_{>k})^{-1/2} \hat{S}_n \psi_{>k}^* \Lambda^{>k}_{\mathcal{A}^2 \Sigma^{-\beta'+2\beta}} \psi_{>k} \hat{S}_n^* (\tilde{K}^{\gamma}_{>k})^{-1/2}\| \\
    \leq& \| \Lambda^{>k}_{\Sigma^{-\beta'+\beta}}\| \| I_n - n\gamma_n (\tilde{K}^{\gamma}_{>k})^{-1} \| \\
    \leq& \| \Lambda^{>k}_{\Sigma^{-\beta'+\beta}}\|,
\end{align*}
where the last transition is by the fact that $I_n - n\gamma_n (\tilde{K}^{\gamma}_{>k})^{-1}  $ is PSD matrix with norm bounded by 1 for $\gamma_n \geq 0$.

For (4), it can be upper bounded by
\begin{align*}
    &\mu_1((\psi_{\leq k} \hat{S}^*_n (\tilde{K}^{\gamma}_{>k})^{-1} \hat{S}_n \psi^*_{\leq k}) ^{-1})^2 \\
    =& \frac{1}{\mu_k((\psi_{\leq k} \hat{S}^*_n (\tilde{K}^{\gamma}_{>k})^{-1} \hat{S}_n \psi^*_{\leq k}) )^2}\\
    \leq& \frac{1}{\mu_k((\psi_{\leq k} \hat{S}^*_n  \hat{S}_n \psi^*_{\leq k}) )^2 \mu_n((\tilde{K}^{\gamma}_{>k})^{-1})^2}.
\end{align*}
Therefore, the third term overall can be bounded by
$$
\| \Lambda^{>k}_{\Sigma^{-\beta'+\beta}}\| \frac{\mu_1((\tilde{K}^{\gamma}_{>k})^{-1})}{\mu_n((\tilde{K}^{\gamma}_{>k})^{-1})^2} \frac{\mu_1(\psi_{\leq k} \hat{S}^*_n  \hat{S}_n \psi^*_{\leq k}) }{\mu_k(\psi_{\leq k} \hat{S}^*_n  \hat{S}_n \psi^*_{\leq k} )^2} \| \phi_{\leq k} f^*_{\leq k}\|_{\Lambda^{\leq k}_{\mathcal{A}^{-2} \Sigma^{1 - 2\beta}}}.
$$

We gather all the terms then we get the desired bound.
\end{proof}
% \honam{TODO: also need to provide the lemma for bounding $\mu_1(\tilde{K}^{\gamma}_{>k})$ and $\mu_n(\tilde{K}^{\gamma}_{>k})$, refer to \ref{lemma:smallest_largest_eigen_kernel}}
% \honam{I think we have proved most of the necessary lemmas and theorems (Maybe we need more lemmas to derive bounds for $\tilde{K}_\leq k$, and some lemmas for deriving the polynomial bounds, probably do it later), 

\begin{lemma}[Simplified Upper bound for bias using concentration]
    \label{lemma:concentration_bias}
    There exists some absolute constant $c, c', C_2 > 0$ s.t. for any $k \in \mathbb{N}$ with $c \beta_k k \log(k) \leq n$, it holds w.p. at least $1 - \delta - 8\exp(-\frac{c'}{\beta_k^2} \frac{n}{k})$, the bias can be upper bounded as:
    \begin{align*}
        B \leq C_2 {\big(}&  \frac{\mu_1(\frac{1}{n}\tilde{K}_{>k}^{\gamma} )^2 }{\mu_n( \frac{1}{n} \tilde{K}_{>k}^{\gamma} )^2 } \frac{ 1}{ p_k^2 \lambda_k^{\beta'}} (\frac{1}{\delta}  \| \phi_{>k} \mathcal{A}_{>k} f_{>k}\|^2_{\Lambda^{>k}_{\Sigma}}) \\
        &+ \ \frac{\mu_1(\frac{1}{n}\tilde{K}^{\gamma}_{>k})^2 \|\phi_{\leq k} f_{\leq k}^*\|^2_{\Lambda^{\leq k}_{
        \mathcal{A}^{-2}\Sigma^{1-2\beta}}}}{ p_k^2 \lambda_k^{\beta'}} \\
        &+ \ \|\phi_{>k} f^*_{>k}\|^2_{\Lambda^{>k}_{\Sigma^{1-\beta'}}} \\
        &+  \ \| \Lambda^{>k}_{\Sigma^{1 - \beta'}} \| \ \frac{1}{ \mu_n(\frac{1}{n} \tilde{K}_{>k}^{\gamma})^2} (\frac{1}{\delta}  \| \phi_{>k} \mathcal{A}_{>k} f_{>k} \|^2_{\Lambda_{\Sigma}^{>k}}) ( p_{k+1}^2 \lambda_{k+1}^{2\beta - 1}) \\
        &+  \|\Lambda^{>k}_{\Sigma^{-\beta' + \beta}} \| \frac{ \mu_1(\frac{1}{n} \tilde{K}^{\gamma}_{>k})^2}{\mu_n(\frac{1}{n} \tilde{K}^{\gamma}_{>k})}  \|\phi_{\leq k} f_{\leq k}^*\|^2_{\Lambda^{\leq k}_{
        \mathcal{A}^{-2}\Sigma^{1-2\beta}}}
        {\big)}.
    \end{align*}
\end{lemma}
\begin{proof}
    Recall that from \ref{lemma:intermediate_ub_bias} we have
        \begin{align*}
        B \leq& \ 3 \Bigg(\frac{\mu_1((\tilde{K}_{>k}^{\gamma})^{-1} )^2 }{\mu_n((\tilde{K}_{>k}^{\gamma})^{-1} )^2 } \frac{\mu_1(\psi_{\leq k } \hat{S}_n^* \hat{S}_n \psi^*_{\leq k} )}{\mu_k( \psi_{\leq k}\hat{S}_n^*   \hat{S}_n \psi^*_{\leq k} )^2 \mu_k( \Lambda^{\leq k}_{\mathcal{A}^2 \Sigma^{\beta'}} )} \|\hat{S}_n \mathcal{A}_{>k} f^{*}_{>k}\|^2 \\
    +& \ \frac{\| \phi_{\leq k} f^{*}_{\leq k}\|^2_{\Lambda^{\leq k}_{\mathcal{A}^{-2}\Sigma^{1 - 2\beta}}}}{\mu_n((\tilde{K}_{>k}^{\gamma})^{-1} )^2 \mu_k( \psi_{\leq k}\hat{S}_n^*   \hat{S}_n \psi^*_{\leq k} )^2 \mu_k( \Lambda^{\leq k}_{\mathcal{A}^2 \Sigma^{\beta'}} )} \\
    +& \ \|\phi_{>k} f^*_{>k}\|^2_{\Lambda^{>k}_{\Sigma^{1-\beta'}}} \\
    +& \ \| \Lambda^{>k}_{\Sigma^{1 - \beta'}} \| \ \mu_1[(\tilde{K}_{>k}^
\gamma)^{-1}]^2 \|\hat{S}_n \mathcal{A}_{>k} f_{>k}\|^2 \mu_1( \underbrace{\hat{S}_n \psi_{>k}^* \Lambda^{>k}_{\mathcal{A}^2 \Sigma^{2\beta - 1}} \psi_{>k} \hat{S}^*_n}_{n \times n}) \\
    +& \| \Lambda^{>k}_{\Sigma^{-\beta'+\beta}}\| \frac{\mu_1((\tilde{K}^{\gamma}_{>k})^{-1})}{\mu_n((\tilde{K}^{\gamma}_{>k})^{-1})^2} \frac{\mu_1(\psi_{\leq k} \hat{S}^*_n  \hat{S}_n \psi^*_{\leq k}) }{\mu_k(\psi_{\leq k} \hat{S}^*_n  \hat{S}_n \psi^*_{\leq k} )^2} \| \phi_{\leq k} f^*_{\leq k}\|_{\Lambda^{\leq k}_{\mathcal{A}^{-2} \Sigma^{1 - 2\beta}}}\Bigg).
    % +& \ \mu_1( \hat{S}_n \psi^*_{>k}  \Lambda^{>k}_{\mathcal{A}^2 \Sigma^{(-\beta' + 2\beta )}} \psi_{>k}\hat{S}_n^*  ) \ \mu_1((\tilde{K}^{\gamma}_{>k})^{-1})^2 \\ 
    % &\frac{\mu_1(\psi_{\leq k} \hat{S}_n^* \hat{S}_n \psi^*_{\leq k})}{ \mu_k( \Lambda^{\leq k}_{\mathcal{A}\Sigma^{\beta/2}}\psi_{\leq k} \hat{S}_n^* \hat{S}_n \psi^*_{\leq k} \Lambda^{\leq k}_{\mathcal{A}\Sigma^{\beta/2}})^2 \mu_n((\tilde{K}^{\gamma})^{-1})^2} \|\Lambda^{\leq k}_{\mathcal{A} \Sigma^{1/2}} \phi_{\leq k} f_{\leq k}\|^2)\\
    \end{align*}
     We first apply $\mu_1((\tilde{K}_{>k}^{\gamma})^{-1}) = \frac{1}{n \mu_n(\frac{1}{n}\tilde{K}_{>k}^{\gamma})}$ and $\mu_n((\tilde{K}_{>k}^{\gamma})^{-1}) = \frac{1}{n \mu_1(\frac{1}{n}\tilde{K}_{>k}^{\gamma})}$ , also apply concentration inequalities using Lemma \ref{lemma:simultaneous_concentration}, Lemma \ref{lemma:SnAf} and Lemma \ref{lemma:ostrowski} , 
     %(\honam{Here needs to check whether the current form of Ostrowski's thm satisfies the requirement, since there is $\frac{1}{n}$})
     then w.p. at least $1 - \delta - 8\exp(-\frac{c}{\beta_k^2} \frac{n}{k})$, we can obtain bound like this
    \begin{align*}
        {\big(}&  \frac{\mu_1(\frac{1}{n}\tilde{K}_{>k}^{\gamma} )^2 }{\mu_n( \frac{1}{n} \tilde{K}_{>k}^{\gamma} )^2 } \frac{ c_1 n}{c_2^2 n^2 p_k^2 \lambda_k^{\beta'}} (\frac{1}{\delta} n \| \phi_{>k} \mathcal{A}_{>k} f_{>k}\|^2_{\Lambda^{>k}_{\Sigma}}) \\
        &+ \ \frac{\mu_1(\frac{1}{n}\tilde{K}^{\gamma}_{>k})^2 \|\phi_{\leq k} f_{\leq k}^*\|^2_{\Lambda^{\leq k}_{
        \mathcal{A}^{-2}\Sigma^{1-2\beta}}}}{c_1^2 p_k^2 \lambda_k^{\beta'}} \\
        &+ \ \|\phi_{>k} f^*_{>k}\|^2_{\Lambda^{>k}_{\Sigma^{1-\beta'}}} \\
        &+  \ \| \Lambda^{>k}_{\Sigma^{1 - \beta'}} \| \ \frac{1}{n^2 \mu_n(\frac{1}{n} \tilde{K}_{>k}^{\gamma})^2} (\frac{1}{\delta} n \| \phi_{>k} \mathcal{A}_{>k} f_{>k} \|^2_{\Lambda_{\Sigma}^{>k}}) (n p_{k+1}^2 \lambda_{k+1}^{2\beta - 1}) \\
        &+  \|\Lambda^{>k}_{-\beta' + \beta} \| \frac{n^2 \mu_1(\frac{1}{n} \tilde{K}^{\gamma}_{>k})^2}{n \mu_n(\frac{1}{n} \tilde{K}^{\gamma}_{>k})} \frac{c_2 n}{c_1^2 n^2} \|\phi_{\leq k} f_{\leq k}^*\|^2_{\Lambda^{\leq k}_{
        \mathcal{A}^{-2}\Sigma^{1-2\beta}}}
        {\big)}.
    \end{align*}
    This can be upper bounded by 
    \begin{align*}
        C_2 {\big(}&  \frac{\mu_1(\frac{1}{n}\tilde{K}_{>k}^{\gamma} )^2 }{\mu_n( \frac{1}{n} \tilde{K}_{>k}^{\gamma} )^2 } \frac{ 1}{ p_k^2 \lambda_k^{\beta'}} (\frac{1}{\delta}  \| \phi_{>k} \mathcal{A}_{>k} f_{>k}\|^2_{\Lambda^{>k}_{\Sigma}}) \\
        &+ \ \frac{\mu_1(\frac{1}{n}\tilde{K}^{\gamma}_{>k})^2 \|\phi_{\leq k} f_{\leq k}^*\|^2_{\Lambda^{\leq k}_{
        \mathcal{A}^{-2}\Sigma^{1-2\beta}}}}{ p_k^2 \lambda_k^{\beta'}} \\
        &+ \ \|\phi_{>k} f^*_{>k}\|^2_{\Lambda^{>k}_{\Sigma^{1-\beta'}}} \\
        &+  \ \| \Lambda^{>k}_{\Sigma^{1 - \beta'}} \| \ \frac{1}{ \mu_n(\frac{1}{n} \tilde{K}_{>k}^{\gamma})^2} (\frac{1}{\delta}  \| \phi_{>k} \mathcal{A}_{>k} f_{>k} \|^2_{\Lambda_{\Sigma}^{>k}}) ( p_{k+1}^2 \lambda_{k+1}^{2\beta - 1}) \\
        &+  \|\Lambda^{>k}_{-\beta' + \beta} \| \frac{ \mu_1(\frac{1}{n} \tilde{K}^{\gamma}_{>k})^2}{\mu_n(\frac{1}{n} \tilde{K}^{\gamma}_{>k})}  \|\phi_{\leq k} f_{\leq k}^*\|^2_{\Lambda^{\leq k}_{
        \mathcal{A}^{-2}\Sigma^{1-2\beta}}}
        {\big)}
    \end{align*}
    where $C_2 > 0$ is some constant only depends on $c_1, c_2$.
\end{proof}

\begin{theorem}[Bound on bias] There exists some absolute constant $C_2, c, c' > 0$ s.t. for any $k \in \mathbb{N}$ with $c \beta_k k \log(k) \leq n$, it holds w.p. at least $1 - \delta - 8\exp(-\frac{c'}{\beta_k^2} \frac{n}{k})$, the bias can be further bounded as 
\begin{align*}
        B \leq& C_2 \frac{\rho_{k,n}^3}{\delta} (\| \phi_{>k} \mathcal{A}_{>k} f_{>k}\|^2_{\Lambda^{>k}_{\Sigma}} \frac{1}{p_k^2 \lambda_k^{\beta'}} 
        + \ \|\phi_{\leq k} f_{\leq k}^*\|^2_{\Lambda^{\leq k}_{
        \mathcal{A}^{-2}\Sigma^{1-2\beta}}}  (\gamma_n + \frac{\beta_k \trace(\tilde{\Sigma}_{>k})}{n})^2 \frac{1}{p_k^2 \lambda_k^{\beta'}}  \\
        +& \ \| \phi_{>k} f^*_{>k}\|^2_{\Lambda^{>k}_{\Sigma^{1 - \beta'}}}).
    \end{align*}
    
    \label{theorem:bias_proof}
\end{theorem}
\begin{proof}
    We refer result from previous lemma \ref{lemma:concentration_bias}.
    \begin{align*}
        B \leq C_2 {\big(}&  \frac{\mu_1(\frac{1}{n}\tilde{K}_{>k}^{\gamma} )^2 }{\mu_n( \frac{1}{n} \tilde{K}_{>k}^{\gamma} )^2 } \frac{ 1}{ p_k^2 \lambda_k^{\beta'}} (\frac{1}{\delta}  \| \phi_{>k} \mathcal{A}_{>k} f_{>k}\|^2_{\Lambda^{>k}_{\Sigma}}) \\
        &+ \ \frac{\mu_1(\frac{1}{n}\tilde{K}^{\gamma}_{>k})^2 \|\phi_{\leq k} f_{\leq k}^*\|^2_{\Lambda^{\leq k}_{
        \mathcal{A}^{-2}\Sigma^{1-2\beta}}}}{ p_k^2 \lambda_k^{\beta'}} \\
        &+ \ \|\phi_{>k} f^*_{>k}\|^2_{\Lambda^{>k}_{\Sigma^{1-\beta'}}} \\
        &+  \ \| \Lambda^{>k}_{\Sigma^{1 - \beta'}} \| \ \frac{1}{ \mu_n(\frac{1}{n} \tilde{K}_{>k}^{\gamma})^2} (\frac{1}{\delta}  \| \phi_{>k} \mathcal{A}_{>k} f_{>k} \|^2_{\Lambda_{\Sigma}^{>k}}) ( p_{k+1}^2 \lambda_{k+1}^{2\beta - 1}) \\
        &+  \|\Lambda^{>k}_{\Sigma^{-\beta' + \beta}} \| \frac{ \mu_1(\frac{1}{n} \tilde{K}^{\gamma}_{>k})^2}{\mu_n(\frac{1}{n} \tilde{K}^{\gamma}_{>k})}  \|\phi_{\leq k} f_{\leq k}^*\|^2_{\Lambda^{\leq k}_{
        \mathcal{A}^{-2}\Sigma^{1-2\beta}}}
        {\big)}.
    \end{align*}
    Note that by definition of $\rho_{k,n}$ (refer to Definition~\ref{def:concerntration}), we have a following estimations:
    $$
     \frac{\mu_1(\frac{1}{n}\tilde{K}_{>k}^{\gamma} )^2 }{\mu_n( \frac{1}{n} \tilde{K}_{>k}^{\gamma} )^2 } = \frac{(\mu_1(\frac{1}{n}\tilde{K}_{>k}) + \gamma_n) ^2 }{(\mu_n( \frac{1}{n} \tilde{K}_{>k}) + \gamma_n) ^2 } \leq \rho_{k,n}^2,
    $$
    \begin{align*}
        \mu_1( \frac{1}{n} \tilde{K}_{>k}^{\gamma})^2 =& \frac{\mu_1(\frac{1}{n}\tilde{K}_{>k}^{\gamma} )^2 }{\mu_n( \frac{1}{n} \tilde{K}_{>k}^{\gamma} )^2 } \mu_n( \frac{1}{n} \tilde{K}_{>k}^{\gamma} )^2 \\
        \leq& \rho_{k,n}^2 (\frac{1}{n} \trace(\frac{1}{n} \tilde{K}^{\gamma}_{>k}))^2 \leq \rho_{k,n}^2 (\gamma_n + \frac{1}{n} \sum_{j=1}^{n} \sum_{i>k} \lambda_i^\beta p_i^2 \psi_i(x_j)^2 )^2 \\
        \leq& \rho_{k,n}^2 (\gamma_n + \frac{\beta_k \trace(\tilde{\Sigma}_{>k})}{n})^2,
    \end{align*}
    $$
    \frac{ \| \Lambda^{>k}_{\mathcal{A}^2 \Sigma^{\beta}}\| }{\mu_n(\frac{1}{n} \tilde{K}_{>k})} \leq \rho_{k,n}
    $$
    and
    \begin{align*}
        &\|\Lambda^{>k}_{\Sigma^{-\beta' + \beta}} \| \frac{ \mu_1(\frac{1}{n} \tilde{K}_{>k}^{\gamma})^2}{\mu_n(\frac{1}{n} \tilde{K}_{>k}^{\gamma})} =  \frac{ \| \Lambda^{>k}_{\mathcal{A}^2 \Sigma^{\beta}}\| }{\mu_n(\frac{1}{n} \tilde{K}_{>k})} \|\Lambda^{>k}_{\mathcal{A}^{-2} \Sigma^{-\beta'}} \| \mu_1(\frac{1}{n} \tilde{K}_{>k}^{\gamma})^2 \\
        \leq& \rho_{k,n}^3 (\gamma_n + \frac{\beta_k \trace(\tilde{\Sigma}_{>k})}{n})^2\| \Lambda^{>k}_{\mathcal{A}^{-2} \Sigma^{-\beta'}}\|.
    \end{align*}
    We bound first and forth term first
    \begin{align*}
         &\frac{\mu_1(\frac{1}{n}\tilde{K}_{>k}^{\gamma} )^2 }{\mu_n( \frac{1}{n} \tilde{K}_{>k}^{\gamma} )^2 } \frac{ 1}{ p_k^2 \lambda_k^{\beta'}} (\frac{1}{\delta}  \| \phi_{>k} \mathcal{A}_{>k} f_{>k}\|^2_{\Lambda^{>k}_{\Sigma}})  + \| \Lambda^{>k}_{\Sigma^{1 - \beta'}} \| \ \frac{1}{ \mu_n(\frac{1}{n} \tilde{K}_{>k}^{\gamma})^2} (\frac{1}{\delta}  \| \phi_{>k} \mathcal{A}_{>k} f_{>k} \|^2_{\Lambda_{\Sigma}^{>k}}) ( p_{k+1}^2 \lambda_{k+1}^{2\beta - 1}) \\
        \leq& (\frac{1}{\delta}  \| \phi_{>k} \mathcal{A}_{>k} f_{>k}\|^2_{\Lambda^{>k}_{\Sigma}}) (\rho_{k,n}^2 \frac{1}{p_k^2 \lambda_k^{\beta'}} + \frac{\|\Lambda^{>k}_{\mathcal{A}^4 \Sigma^{2\beta}} \|}{\mu_n(\frac{1}{n} \tilde{K}^{\gamma}_{>k})^2} p_{k+1}^2 \lambda_{k+1}^{2\beta - 1} \| \Lambda^{>k}_{\mathcal{A}^{-4} \Sigma^{1 - \beta' - 2\beta}}\|) \\
        \leq& \rho_{k,n}^2 (\frac{1}{\delta}  \| \phi_{>k} \mathcal{A}_{>k} f_{>k}\|^2_{\Lambda^{>k}_{\Sigma}}) (\frac{1}{p_k^2 \lambda_k^{\beta'}} + p_{k+1}^2 \lambda_{k+1}^{2\beta - 1} \| \Lambda^{>k}_{\mathcal{A}^{-4} \Sigma^{1 - \beta' - 2\beta}}\|).
    \end{align*}
    Since two terms here have the same order, we can just bound it by
    $$
    c_1 \rho_{k,n}^2 (\frac{1}{\delta}  \| \phi_{>k} \mathcal{A}_{>k} f_{>k}\|^2_{\Lambda^{>k}_{\Sigma}}) \frac{1}{p_k^2 \lambda_k^{\beta'}}
    $$
    where $c_1$ is some constant.
    
    Next we bound the second and fifth term
    \begin{align*}
         &\frac{\mu_1(\frac{1}{n}\tilde{K}^{\gamma}_{>k})^2 \|\phi_{\leq k} f_{\leq k}^*\|^2_{\Lambda^{\leq k}_{
        \mathcal{A}^{-2}\Sigma^{1-2\beta}}}}{ p_k^2 \lambda_k^{\beta'}} +  \|\Lambda^{>k}_{\Sigma^{-\beta' + \beta}} \| \frac{ \mu_1(\frac{1}{n} \tilde{K}^{\gamma}_{>k})^2}{\mu_n(\frac{1}{n} \tilde{K}^{\gamma}_{>k})}  \|\phi_{\leq k} f_{\leq k}^*\|^2_{\Lambda^{\leq k}_{
        \mathcal{A}^{-2}\Sigma^{1-2\beta}}} \\
        \leq& \|\phi_{\leq k} f_{\leq k}^*\|^2_{\Lambda^{\leq k}_{
        \mathcal{A}^{-2}\Sigma^{1-2\beta}}} (\frac{1}{p_k^2 \lambda_k^{\beta'}} \rho_{k,n}^2 (\gamma_n + \frac{\beta_k \trace(\tilde{\Sigma}_{>k})}{n})^2 +  \rho_{k,n}^3 (\gamma_n + \frac{\beta_k \trace(\tilde{\Sigma}_{>k})}{n})^2\| \Lambda^{>k}_{\mathcal{A}^{-2} \Sigma^{-\beta'}}\|).
    \end{align*}
    We know $\frac{1}{p_k^2 \lambda_k^{\beta'}}$ and $\| \Lambda^{>k}_{\mathcal{A}^{-2} \Sigma^{-\beta'}} \|$ are of the same order, and $\rho_{k,n} \geq 1$ by its definition, therefore, the second term would be dominated by the fifth term.
    So we can bound it by
    $$
    c_2 \rho_{k,n}^3 \|\phi_{\leq k} f_{\leq k}^*\|^2_{\Lambda^{\leq k}_{
        \mathcal{A}^{-2}\Sigma^{1-2\beta}}}  (\gamma_n + \frac{\beta_k \trace(\tilde{\Sigma}_{>k})}{n})^2 \frac{1}{p_k^2 \lambda_k^{\beta'}}.
    $$
    Therefore, the final bound becomes 
    \begin{align*}
        &C_2(c_1 \rho_{k,n}^2 (\frac{1}{\delta}  \| \phi_{>k} \mathcal{A}_{>k} f_{>k}\|^2_{\Lambda^{>k}_{\Sigma}}) \frac{1}{p_k^2 \lambda_k^{\beta'}} + c_2 \rho_{k,n}^3 \|\phi_{\leq k} f_{\leq k}^*\|^2_{\Lambda^{\leq k}_{
        \mathcal{A}^{-2}\Sigma^{1-2\beta}}}  (\gamma_n + \frac{\beta_k \trace(\tilde{\Sigma}_{>k})}{n})^2 \frac{1}{p_k^2 \lambda_k^{\beta'}} \\
        +& \| \phi_{>k} f^*_{>k}\|^2_{\Lambda^{>k}_{\Sigma^{1 - \beta'}}}) \\
        \leq& C_2' \frac{\rho_{k,n}^3}{\delta} (\| \phi_{>k} \mathcal{A}_{>k} f_{>k}\|^2_{\Lambda^{>k}_{\Sigma}} \frac{1}{p_k^2 \lambda_k^{\beta'}} + \|\phi_{\leq k} f_{\leq k}^*\|^2_{\Lambda^{\leq k}_{
        \mathcal{A}^{-2}\Sigma^{1-2\beta}}}  (\gamma_n + \frac{\beta_k \trace(\tilde{\Sigma}_{>k})}{n})^2 \frac{1}{p_k^2 \lambda_k^{\beta'}}  \\
        +& \| \phi_{>k} f^*_{>k}\|^2_{\Lambda^{>k}_{\Sigma^{1 - \beta'}}}),
    \end{align*}
    $C_2'$ is w.r.t. $C_2, c_1, c_2$, and we finally just take $C_2 = C_2'$ to finish the proof.
\end{proof}

\section{Applications}
\subsection{Regularized Case}
\begin{theorem}[Regularized case, Proof of Theorem \ref{theorem:bias_variance_regularized}]
    \label{theorem:bias_variance_regularized_proof}
   Let the kernel and target function satisfies Assumption \ref{assumption:kernel}, $\gamma_n=\Theta(n^{-\gamma})$, and {   $\gamma < 2p + \beta \lambda$, $2p + \lambda r > 0$ and $r > \beta'$} then for any $\delta > 0$, it holds w.p. $1 - \delta - O(\frac{1}{\log(n)})$ that
   %with high probability\yplu{???}\fh{The probability is the same as Ohad, $1 - \delta - 16 xxxx$}\fh{also under well-behaved features? The condition $\beta_k \log k xxx < n$ is needed.}}

   %V \leq\sigma_{\varepsilon}^2 O(n^{\frac{\gamma (1 + 2p + \lambda \beta')}{2p + \lambda \beta} - 1})
    
    \begin{align*}
        V = \sigma_{\varepsilon}^2 O(n^{\max \{ \frac{  \gamma (1 + 2p + \lambda \beta')}{2p + \lambda \beta} , 0 \} - 1}),B \leq \frac{1}{\delta} \cdot \tilde{O}_n(n^{\frac{\gamma}{2p + \beta \lambda}(\max \{ {\lambda (\beta'-r)},{ -2p+\lambda (\beta' 
- 2\beta)} \})}).\\
    \end{align*}
\end{theorem}
\begin{proof}
    We use the two lemmas \ref{lemma:concentration_variance}, \ref{theorem:bias_proof} for upper bounding bias and variance in this proof, there exists some absolute constants $c, c' > 0$, first we need to pick $k$ s.t. $c \beta_k k \log(k) \leq n$, then the two lemmas will simultaneously hold w.p. at least $1 - \delta - 16 \exp(-\frac{c'}{\beta_k^2} \frac{n}{k})$. With regularization, we can pick $k$ large enough s.t. the concentration coefficient $\rho_{k,n} = o(1)$, to achieve so, we want $\mu_1(\frac{1}{n} \tilde{K}_{>k})= O(\gamma_n)$.
By Lemma~\ref{lemma:mu_1_actual_value}, we can show w.p. at least $1 - 4\frac{r_k}{k^4} \exp(-\frac{c'}{\beta_k} \frac{n}{r_k})$
\begin{equation}
    \label{eq:mu_1}
    \mu_1(\frac{1}{n}\tilde{K}_{>k}) = O_n(p_{k+1}^2 \lambda_{k+1}^{\beta}) = O_n(k^{-2p - \beta \lambda}) = O_n(\gamma_n) = O_n(n^{-\gamma}).
\end{equation}

This can be achieved by setting $k(n) = \lceil n^{\frac{\gamma}{2p + \beta \lambda}} \rceil$, note that we have $\frac{\gamma}{2p + \beta \lambda} < 1$, therefore, $k(n) = O(\frac{n}{\log(n)})$ and the lemmas can be used for sufficient large $n$.

% \honam{need $\frac{\gamma}{2p + \beta \lambda} < 1$ to make $k(n) = o_n(\frac{n}{\log(n)})$, this is required to apply the Thm. 2}.

We combine the probability of both \ref{lemma:concentration_variance}, \ref{theorem:bias_proof} and \ref{eq:mu_1} hold:
$$
1 - \delta - 16\exp {\big(} -\frac{c'}{\beta_k^2} \frac{n}{k} \big{)} - O(\frac{1}{k^3}) \exp(-\Omega(\frac{n}{k})) = 1 - \delta - O(\frac{1}{n})
$$
where we use the fact that $\frac{c'}{\beta_k^2} \frac{n}{k} = \Omega(\log(n))$ since $k(n) = O(\frac{n}{\log(n)})$.

Then now we can assume \ref{lemma:concentration_variance}, \ref{theorem:bias_proof} and \ref{eq:mu_1} hold, and we provide the bound on variance and bias respectively. 

By Theorem~\ref{lemma:concentration_variance} and we sub. $p_i = \Theta(i^{-p})$, $\lambda_i = \Theta(i^{-\lambda})$, $\|\Sigma_{>k} \| = p_{k+1}^2 \lambda_{k+1}^{\beta} = \Theta ((k+1)^{-\beta \lambda - 2p}) = \Theta (k^{-\beta \lambda - 2p})$, 
\begin{align*}
    V \leq& C_1 \sigma_{\varepsilon}^2 \rho_{k,n}^2  {\Big(}\frac{{  \sum_{i\leq k}  p_i^{-2} \lambda_i^{-\beta'}}}{{n}}
    + \frac{ \sum_{i>k} p_i^2 \lambda_i^{-\beta' + 2\beta}}{n {\| \tilde{\Sigma}_{>k} \|^2}} {\Big)} \\
    =& \sigma_{\varepsilon}^2 O(1) O(  \frac{ \max \{ k^{1 + 2p + \lambda \beta'} , 1\}}{n},  \frac{  k^{1 - 2p + \lambda (\beta' - 2\beta)}}{n k^{-2\beta \lambda -4p}}) = \sigma_{\varepsilon}^2 \tilde{O}(\frac{ \max \{ k^{1+2p+\lambda \beta'}, 1\}}{n}).
\end{align*}

We substitute $k$ with $\lceil n^{\frac{\gamma}{2p + \beta \lambda}} \rceil$ to obtain the final bound
$$
V = \sigma_{\varepsilon}^2 O(n^{\max \{ \frac{  \gamma (1 + 2p + \lambda \beta')}{2p + \lambda \beta} , 0 \} - 1}).
$$
% = \sigma_{\varepsilon}^2 O(n^{\frac{\gamma (1 + 2p + \lambda \beta')}{2p + \lambda \beta} - 1})

For bias, recall that by Theorem~\ref{theorem:bias_proof}, we have
\begin{align*}
        B \leq& C_2 \frac{\rho_{k,n}^3}{\delta} (\| \phi_{>k} \mathcal{A}_{>k} f_{>k}\|^2_{\Lambda^{>k}_{\Sigma}} \frac{1}{p_k^2 \lambda_k^{\beta'}} \\
        +& \ \|\phi_{\leq k} f_{\leq k}^*\|^2_{\Lambda^{\leq k}_{
        \mathcal{A}^{-2}\Sigma^{1-2\beta}}}  (\gamma_n + \frac{\beta_k \trace(\tilde{\Sigma}_{>k})}{n})^2 \frac{1}{p_k^2 \lambda_k^{\beta'}}  \\
        +& \ \| \phi_{>k} f^*_{>k}\|^2_{\Lambda^{>k}_{\Sigma^{1 - \beta'}}}).
    \end{align*}
    
By $\trace(\tilde{\Sigma}_{>k}) = \sum_{i>k} p_i^2 \lambda_i^{\beta} =  O(k \lambda_k^{\beta} p_k^2) = O(k \gamma_n)$, then 
$$
(\gamma_n + \frac{\beta_k \trace(\tilde{\Sigma}_{>k})}{n})^2 = O((\gamma_n + \frac{n}{k} \gamma_n)^2) = O(\gamma_n^2) = O(k^{-4p -2\lambda\beta})
$$
Recall that
$$
\frac{ \|\phi_{\leq k} f_{\leq k}^*\|^2_{\Lambda^{\leq k}_{\mathcal{A}^{-2} \Sigma^{1 - 2\beta}}}}{p_k^2 \lambda_k^{\beta'}} = \tilde{O}(k^{\max \{ 1 + 4p - \lambda (1 -  \beta' - 2\beta) - 2r', 2p + \lambda \beta'\}}).
$$
Therefore, the second term's bound is
$$
O(k^{\max \{ 1 - 2r - \lambda (1 - \beta'), -2p + \lambda (\beta' - 2\beta) \}}).
$$
Since $2p + \lambda r > 0$ and $r > \beta'$, we have $2p + 2r' + \lambda > 1$, and $2r' + (1 - \beta') \lambda > 1$, We can quote Lemma~\ref{lemma:sub_polynomial_bound} for the remaining terms, so the third term's bound is
$$
O(k^{1 - 2r' - (1 - \beta') \lambda}).
$$
First term's bound is the same as the second
$$
O(k^{\max \{ 1 - 2r' - \lambda (1 - \beta'), -2p + \lambda (\beta' - 2\beta) \}}).
$$
%\fh{remember $p$ is smaller than zero}

So we sub. $k = \lceil n^{\frac{\gamma}{2p + \beta \lambda}} \rceil$ to obtain 
$$
B \leq \frac{1}{\delta} \cdot \tilde{O}_n(n^{\frac{\gamma}{2p + \beta \lambda}(\max \{ 1 - 2r' - \lambda (1 - \beta'), -2p+\lambda (\beta' 
- 2\beta) \})}).
$$
And we substitute $r' = \frac{1 - \lambda (1 - r)}{2}$ to obtain the final bound
$$
B = O(n^{\frac{\gamma}{2p + \beta \lambda}(\max \{ {\lambda (\beta'-r)},{ -2p+\lambda (\beta' 
- 2\beta)} \})}).
$$

\end{proof}

\subsection{Interpolation Case}
%\honam{TODO: change the wordings and maybe need to reorganize a bit}
\begin{theorem}[Interpolation case, proof of Theorem \ref{theorem:polynomial_bias_variance_interpolate}]
    \label{theorem:app_interpolation}

    Let the kernel and target function satisfies Assumption \ref{assumption:kernel}, $2p + \beta \lambda > 0$, $2p + \lambda r > 0$ and $r > \beta'$, then for any $\delta > 0$ it holds w.p. at least $1 - \delta - O(\frac{1}{\log(n)})$ that
    %\fh{assumptions and probability}
    \begin{align*}
        V \leq\sigma_{\varepsilon}^2 \rho_{k,n}^2 \tilde{O}( n^{ \max \{ 2p + \lambda \beta' , -1 \}}), B \leq  \frac{\rho_{k,n}^3}{\delta}\tilde{O}(n^{\max  \{ {\lambda (\beta'-r)}, {-2p + \lambda(\beta' - 2\beta)} \}\}}),
    \end{align*}
    where $\rho_{k,n} = \tilde{O}(n^{2p + \beta \lambda - 1})$, when features are well-behaved i.e. subGaussian it can be improved to $\rho_{k,n} = o(1)$.
\end{theorem}
\begin{proof}
    Same as regularized case, we use the two theorems \ref{lemma:concentration_variance}, \ref{theorem:bias_proof} for upper bounding bias and variance in this proof, there exists some absolute constants $c, c' > 0$, first we need to pick $k$ s.t. $c \beta_k k \log(k) \leq n$, then the two lemmas will simultaneously hold w.p. at least $1 - \delta - 16 \exp(-\frac{c'}{\beta_k^2} \frac{n}{k})$. Since $\beta_k = o(1)$ we know it can be upper bounded by $C_0$ for some $C_0 > 0$. Similar to \cite{barzilai2023generalization}, we let $k:= k(n) := \frac{n}{\max \{ cC0, 1 \} \log n} $ and we also let $k' := k'(n) = n^2 \log^4 (n)$. So the probability of those theorems hold become $1 - \delta - O(\frac{1}{n})$.

    In this case, $\rho_{k,n}$ cannot be regularized to $o(1)$ if the features are not well-behaved, we compute its bound first, which requires bounding $\mu_1(\frac{1}{n} \tilde{K}_{>k})$ and $\mu_n(\frac{1}{n} \tilde{K}_{>k})$ respectively.

We apply Lemma~\ref{lemma:smallest_largest_eigen_kernel} by setting $\delta = \log n$, then w.p. $1 - \frac{1}{\log(n)}$ we have %\honam{The logarithmic term of $\delta$ and $k'$ is carefully picked to make $\frac{1}{\delta} \sqrt{\frac{n^2}{R_k'}}$ won't explode to infinity }
\begin{align*}
    \mu_n(\frac{1}{n} \tilde{K}_{>k}) \geq& \alpha_k (1 - \frac{1}{\log n} \sqrt{\frac{n^2}{ \trace(\tilde{\Sigma}_{>k'})^2/\trace(\tilde{\Sigma}^2_{>k'})}}) \frac{\trace(\tilde{\Sigma}_{>k'})}{n} \\
=& \ \Omega((1 - \log n \sqrt{\frac{1}{\log^4 n}}) \frac{\trace(\tilde{\Sigma}_{>k'})}{n}) \\
=& \ \Omega(\frac{(k')^{1 - 2p - \beta \lambda}}{n}) \\
=& \ \Omega(\frac{(n^2 \log^4 n)^{1 - 2p - \beta \lambda}}{n}) \\
=& \ \tilde{\Omega}(n^{1 - 4p - 2\beta \lambda}). 
\end{align*}
Note that the first equality is because we have $\trace(\tilde{\Sigma}_{>k'})^2 / \trace(\tilde{\Sigma}^2_{>k'}) = \frac{(\sum_{i>k'} p_i^2 \lambda_i^{\beta} )^2 }{\sum_{i>k'} p_i^4 \lambda_i^{2\beta}} =  \frac{{k'}^{2 - 2p - \lambda \beta}}{{k'}^{1 - 2p - \lambda \beta}} = k' = n^2 \log^4(n)$, $\tilde{\Omega}$ means we neglect logarithmic terms.

For $\mu_1(\frac{1}{n} \tilde{K}_{>k})$ term by Lemma~\ref{lemma:mu_1_actual_value}, we have w.p. $1 - O(\frac{1}{k^3}) \exp(-\Omega(\frac{n}{k}))$
\begin{equation}
    \label{eq:mu_1_O}
    \mu_1(\frac{1}{n} \tilde{K}_{>k}) = O(p_{k+1}^2 \lambda_{k+1}^{\beta}) = O(k^{-2p - \beta \lambda}) = \tilde{O}(n^{-2p - \beta \lambda}).
\end{equation}

So we have Eq. \ref{eq:mu_1_O}, Lemma~\ref{lemma:smallest_largest_eigen_kernel}, Theorem~\ref{lemma:concentration_variance}, \ref{theorem:bias_proof} all hold simultaneously hold with probability $1 - \delta - O(\frac{1}{\log(n)})$.

Therefore, we have $\rho_{k,n} = \tilde{O}(n^{2p + \beta \lambda - 1})$.% \honam{Note that $2p + \beta \lambda - 1$ should $ \geq 0$, since by definition $\rho_{k,n}$ should > 1}

Recall from Lemma~\ref{lemma:concentration_variance} that

\begin{align*}
    V \leq& \  C_1 \sigma_{\varepsilon}^2 \rho_{k,n}^2  {\Big(}\frac{{  \sum_{i\leq k}  p_i^{-2} \lambda_i^{-\beta'}}}{{n}}
    + \frac{ \sum_{i>k} p_i^2 \lambda_i^{-\beta' + 2\beta}}{n {\| \tilde{\Sigma}_{>k} \|^2}} {\Big)} \\
    =& \ \sigma_{\varepsilon}^2 \rho_{k,n}^2 O(\frac{\max \{ k^{1 + 2p + \lambda \beta'} , 1\} }{n} + \frac{  k^{1 - 2p + \lambda (\beta' - 2\beta)}}{n k^{-2\beta \lambda -4p}}) \\
    =& \ \sigma_{\varepsilon}^2 \rho_{k,n}^2 \tilde{O}(\frac{\max \{ k^{1 + 2p + \lambda \beta'} , 1\}}{n}).
\end{align*}

So we sub. $k = \tilde{\Theta}(n)$ and the final bound of variance is
$$
V \leq \sigma_{\varepsilon}^2 \rho_{k,n}^2 \tilde{O}( n^{ \max \{ 2p + \lambda \beta' , -1 \}}).
$$
For bias, similar to the regularized case, the bound is
$$\frac{1}{\delta} \rho_{k,n}^3 O(k^{\max \{ 1 - 2r' - \lambda (1 - \beta'), -2p + \lambda (\beta' - 2\beta) \}}).$$
The main difference is the choice of $k$, since $k = \tilde{\Theta}(n)$, the final bound is
$$\frac{1}{\delta} \rho_{k,n}^3 O(n^{\max \{ 1 - 2r' - \lambda (1 - \beta'), -2p + \lambda (\beta' - 2\beta) \}}).$$
Note that if the features are well-behaved, then $\rho_{k,n}$ can be improved to $o(1)$.
\end{proof}

\subsection{Lemmas for substituting polynomial decay}
\begin{lemma}
    \label{lemma:sub_smaller_than_k}
    Let $a \in \mathbb{R}$, $1 < k \in \mathbb{N}$, then
    $$
    \sum_{i \leq k} i^{-a} \leq \begin{cases}
        1 + \frac{1}{1-a}k^{1-a} \ \ &a < 1 \\
        1 + \log(k) \ \ &a = 1 \\
        1 + \frac{1}{a-1} \ \ &a > 1.
    \end{cases}
    $$
    Therefore, $\sum_{i \leq k} i^{-a} = \tilde{O}(\max \{ k^{-a + 1}, 1 \})$
\end{lemma}
\begin{proof}
     We know that, for $a<1$
    $$
     \sum_{i\leq k} i^{-a}\leq 1+ \int_{1}^{k} x^{-a} \ dx  =1+ \frac{1}{1-a}(k^{1-a}-1)\leq 1+ \frac{1}{1-a}k^{1-a}.
    $$
    For $a = 1$
    $$
     \sum_{i\leq k} i^{-a}\leq 1+ \int_{1}^{k} x^{-a} \ dx  =1+ \log (k).
    $$
    For $a>1$
    $$
     \sum_{i\leq k} i^{-a}\leq 1+ \int_{1}^{\infty} x^{-a} \ dx  =1+ \frac{1}{a-1}.
    $$
\end{proof}

\begin{lemma}
\label{lemma:sub_larger_than_k}
    Let $a \in \mathbb{R}$, $1 < k \in \mathbb{N}$, then
    $$
    \sum_{i>k} i^{-a} \in \begin{cases}
        \infty \ \ &a \leq 1 \\
        [\frac{1}{a-1} (k+1)^{-a+1}, (k+1)^{-a} + \frac{1}{a-1} (k+1)^{-a+1}] \ \ &a > 1 .
    \end{cases}
    $$
    Therefore, $\sum_{i>k}{i^{-a}}$ is $O(k^{-a+1})$ if $a > 1$, otherwise it diverges to infinity
\end{lemma}
\begin{proof}
    We know that,
    $$
    \int_{k+1}^{\infty} x^{-a} \ dx \leq \sum_{i>k} i^{-a} \leq (k+1)^{-a} + \int_{k+1}^{\infty} x^{-a} \ dx .
    $$
    If $a < 1$ then $\int_{k+1}^{\infty} x^{-a} = \infty$ which implies the series diverge, otherwise, $\int_{k+1}^{\infty} x^{-a} = \frac{1}{a+1}(k+1)^{-a+1}$
    
\end{proof}
% \honam{Instead of using this lemma, \cite{barzilai2023generalization} bound $r_k$ using Lemma 18 there, we don't have defined $r_k$ currently though}

\begin{lemma} Assume $[\phi f^*]_i = \Theta(i^{-r'})$, $\Sigma$'s polynomial decaying eigenvalues satisfy $\lambda_i = \Theta(i^{-\lambda})$ ($\lambda$ > 0), and $\mathcal{A}$'s eigenvalue is $\Theta(i^{-p})$ ($p < 0$), then 
    \label{lemma:sub_polynomial_bound}
    \begin{align*}
        \| \phi_{>k} \mathcal{A}_{>k} f^*_{>k}\|^2_{\Lambda^{>k}_{\Sigma}} &= \Theta\left(\frac{1}{k^{2p+2r'+\lambda - 1}}\right) \text{ if } 2p+2r'+\lambda > 1 ;\\
        \|\phi_{\leq k} f_{\leq k}^*\|^2_{\Lambda^{\leq k}_{ \mathcal{A}^{-2} \Sigma^{1 - 2\beta}}} &= \tilde{O}( \max \{ k^{1 + 2p - \lambda(1 - 2\beta) - 2r'}, 1\}); \\
        \|\phi_{>k} f^*_{>k}\|^2_{\Lambda^{>k}_{\Sigma^{1-\beta'}}} &= \Theta\left(\frac{1}{k^{2r'+(1 - \beta') \lambda - 1}}\right) \text{ if } 2r'+(1 - \beta') \lambda > 1,
    \end{align*}
    where $r' = \frac{1 - \lambda (1 - r)}{2}$.
    % \begin{enumerate}
    %     \item $$
    % \| \phi_{>k} \mathcal{A}_{>k} f^*_{>k}\|^2_{\Lambda^{>k}_{\Sigma}} = \Theta(\frac{1}{k^{2p+2r+\lambda - 1}}) \text{ if } 2p+2r+\lambda > 1
    % $$
    %     \item $$
    % \|\phi_{\leq k} f_{\leq k}^*\|^2_{\Lambda^{\leq k}_{ \mathcal{A}^{-2} \Sigma^{1 - 2\beta}}} = \tilde{O}( \max \{ k^{1 + 2p - \lambda(1 - 2\beta) - 2r}, 1\})
    % $$
    % \item $$
    %  \|\phi_{>k} f^*_{>k}\|^2_{\Lambda^{>k}_{\Sigma^{1-\beta'}}} = \Theta(\frac{1}{k^{2r+(1 - \beta') \lambda - 1}}) \text{ if } 2r+(1 - \beta') \lambda > 1
    % $$
    % \end{enumerate}
\end{lemma}
\begin{proof}
    We know from \ref{lemma:sub_larger_than_k} that, $$  \| \phi_{>k} \mathcal{A}_{>k} f^*_{>k}\|^2_{\Lambda^{>k}_{\Sigma}} = \sum_{i>k} [\phi f^*]_i^2\cdot p_i^2\lambda_i = \sum_{i>k} \Theta\left( \frac{1}{i^{2p+2r'+\lambda}}    \right) = \Theta\left(\frac{1}{k^{2p+2r'+\lambda - 1}}\right) \text{ if } 2p+2r'+\lambda > 1.$$
    Similarly, using \ref{lemma:sub_smaller_than_k}
    $$
    \|\phi_{\leq k} f_{\leq k}^*\|^2_{\Lambda^{\leq k}_{ \mathcal{A}^{-2} \Sigma^{1 - 2\beta}}} = \sum_{i\leq k} [\phi f^*]_i^2\cdot p_i^{2}\lambda_i^{1-2\beta} = \sum_{i\leq k} \Theta\left( \frac{1}{i^{2r'-2p+\lambda(2\beta -1)}}    \right) = \tilde{O}( \max \{ k^{1 + 2p - \lambda(1 - 2\beta) - 2r'}, 1\}).
    $$
    Using \ref{lemma:sub_larger_than_k} again, we'll have
    $$
     \|\phi_{>k} f^*_{>k}\|^2_{\Lambda^{>k}_{\Sigma^{1-\beta'}}} = \sum_{i>k} [\phi f^*]_i^2\cdot \lambda_i^{\beta'-1} = \sum_{i>k} \Theta\left( \frac{1}{i^{2r'+(1 - \beta')\lambda}}    \right) = \Theta\left(\frac{1}{k^{2r'+(1-\beta')\lambda - 1}}\right) \text{ if } 2r'+(1-\beta')\lambda > 1.
    $$
\end{proof}

\begin{lemma}
    \label{lemma:mu_1_actual_value}
    Assume $\Sigma$'s polynomial decaying eigenvalues satisfy $\lambda_i = \Theta(i^{-\lambda})$ ($\lambda$ > 0), and $\mathcal{A}$'s eigenvalue is $\Theta(i^{-p})$. And we suppose $\frac{\beta_k k \log(k)}{n} = o(1), \beta_k = o(1)$.
    
    Then it holds w.p. at least $1 - O(\frac{1}{k^3}) \exp(-\Omega(\frac{n}{k}))$ that
    $$
    \mu_1(\frac{1}{n} \tilde{K}_{>k}) = O(\lambda_{k+1}^{\beta} p_{k+1}^2) = O(k^{- 2p - \beta \lambda }).
    $$
\end{lemma}
\begin{proof}
    We use \ref{lemma:smallest_largest_eigen_kernel} then there exists absolute constant $c, c' > 0$ s.t. it holds w.p. at least $1 - 4\frac{r_k}{k^4} \exp(-\frac{c'}{\beta_k} \frac{n}{r_k})$ that
    \begin{align*}
        \mu_1(\frac{1}{n}\tilde{K}_{>k}) \leq& c(\lambda_{k+1}^\beta p_{k+1}^2 + \beta_k \log(k+1) \frac{\trace(\tilde{\Sigma}_{>k})}{n}) \\
        =& \ O(\lambda_{k+1}^\beta p_{k+1}^2 (1 + \beta_k \log(k+1) \frac{k}{n})) \\
        =& O(\lambda_{k+1}^{\beta} p_{k+1}^2),
    \end{align*}
    where $\tilde{\Sigma} := \mathcal{A}^2 \Sigma^{\beta}$, $r_k := \frac{\trace(\tilde{\Sigma}_{>k})}{p_{k+1}^2\lambda_{k+1}^\beta}$.\\
    The last inequality is because $\frac{k \log(k+1)}{n} = o(1)$.

    Now we bound the probability of this holds, we can derive $r_k =  \frac{k^{1 - 2p - \lambda \beta}}{(k+1)^{-2p - \lambda \beta}}= \Theta(k)$,  $1 - 4\frac{r_k}{k^4} \exp(\frac{-c'}{\beta_k} \frac{n}{r_k}) = 1 - O(\frac{1}{k^3}) \exp(-\Omega(\frac{n}{k}))$.
\end{proof}

% \honam{This lemma is obselete}
% \yplu{wendao, change the following lemma to $\Sigma^\alpha$}
% \begin{lemma}
%     \[
% \|\theta^*\|_{>k}^2 \leq \mathcal{O}_{k,n} \left( \frac{1}{k^{2r+a}} \right),
% \]

% \[
% \|\theta^*\|_{\leq k}^2 \leq \begin{cases} 
% \mathcal{O}_{k,n} (k^{-2r+2+a}) & \text{if } 2r < 2 + a, \\
% \mathcal{O}_{k,n} (\log(k)) & \text{if } 2r = 2 + a, \\
% \mathcal{O}_{k,n} (1) & \text{if } 2r > 2 + a.
% \end{cases}
% \]
% \end{lemma}
% \begin{proof}
%     The condition that $f^* \in L^2(\chi, \mu)$ implies $\sum_{i=1}^\infty (\theta^*_i)^2 < \infty$. The $\theta^* > k$ part can be bounded using Lemma 15 as\fh{what is lemma 15 and 17? and $\|\theta^*_{>k}\|^2$ is in $L^2(\chi, \mu)$ right? why $\lambda_i$ and $1/\lambda_i$ appear?}\fh{Besides, the dependence on $n$ and $k$ could be clearly indicated.}
% \[
% \|\theta^*_{>k}\|^2 = \sum_{i>k} (\theta^*_i)^2 \lambda_i = \mathcal{O}_{k,n} \left( \sum_{i>k} i^{-2r-1-a} \right) \leq \mathcal{O}_{k,n} \left( \frac{1}{k^{2r+a}} \right).
% \]
% The $\leq k$ part can be bounded using Lemma 17 (with $2r - 1 - a$) as
% \[
% \|\theta^*_{\leq k}\|^2 = \sum_{i \leq k} (\theta^*_i)^2 \frac{1}{\lambda_i} = \mathcal{O}_{k,n} \left( \sum_{i \leq k} i^{-2r+1+a} \right) \leq
% \begin{cases}
% \mathcal{O}_{k,n} (k^{-2r+2+a}) & \text{if } 2r < 2 + a, \\
% \mathcal{O}_{k,n} (\log(k)) & \text{if } 2r = 2 + a, \\
% \mathcal{O}_{k,n} (1) & \text{if } 2r > 2 + a.
% \end{cases}
% \]
% \end{proof}
\section{Implementation Details of Experiments}
\label{appendix:exp} 

\paragraph{Experimental Setup} In all experiments, we follow a similar setup as in \cite{mallinar2022benign}. We use Adam optimizer with learning rate 5e-3 for regression problems and 1e-4 for PINN problems, where both are optimally tuned. Following \cite{mallinar2022benign}, we do not apply weight decay, dropout or other regularization techniques, to encourage interpolation. We train with full-batch gradient descent, where the learning rate schedule is StepLR with step size 3000 and gamma 0.8. In both experiments, we train for 100000 iterations to allow convergence. All models considered are sufficiently overparametrized.
 
\paragraph{Poisson Equation} we consider the Poisson equation $u = \Delta f$ on $\Omega = [0,2]^2$ with Dirichlet boundary condition on $\partial \Omega$, where the ground truth $f(x_1,x_2) = \sin(\pi x_1) \sin(\pi x_2)$, where the training data points $\{(x_i, y_i)\}_{i=1}^{n}$ are sampled uniformly from $\Omega$, and $y_i = \Delta f(x_i) + \varepsilon$ with $\varepsilon \sim \mathcal{N}(0, \sigma^2)$. The training loss function is $\min_{\theta} \hat{L}(\theta) := \frac{1}{n} \sum_{i = 1}^{n}{(\Delta 
 \hat{f}(x_i; \theta) - y_i)^2}$. To satisfy the boundary condition, we enforce $\hat{f}(x) = x_1 (x_1 - 2) x_2 (x_2 - 2) f_{\text{NN}}(x)$, where $f_{\text{NN}}$ is the neural network  \cite{liang2021reproducingactivationfunctiondeep}. For clean test loss, we use $\frac{1}{n} \sum_{i = 1}^{n}(\hat{f}(x_i, \theta) - f(x_i))^2$ to match the definition of excess risk, where $\{ (x_i, y_i) \}$ is re-sampled from $\Omega$. We perform Kaiming Initialization on neural networks.

% In Figure \ref{fig:convergence_rate}(Left), the used learning model here is a one-layer wide neural network with hidden size 10,000 and different activation functions $\text{ReLU}$,$\text{ReLU}^2$,$\text{ReLU}^3$ and $\text{ReLU}^4$ \citep{weinan2018deep} with noise level $\sigma^2 = 0.1$. In Figure \ref{fig:convergence_rate}(Middle)(Right),  the used learning model here is a two-layer wide neural network with hidden size 1024, with sample size 500, all models considered are sufficiently over-parametrized.

For the experiment verifying the effect of smoothness of the inductive bias, we uses the one-layer wide neural network with width 10000 (we choose one-layer here to avoid explosion of output due to $\text{ReLU}^4$), and vary different activation functions $\text{ReLU}$,$\text{ReLU}^2$,$\text{ReLU}^3$ and $\text{ReLU}^4$. Noise level $\sigma^2$ is set as 0.1. We vary sample size 50, 100, 500, 1000 and plot the convergence rate using different activation functions.

For the experiment verifying benign over-fitting of Physics-Informed interpolator, we train sufficient iterations to ensure interpolation into the noise.
The used learning model here is a two-layer wide neural network with hidden size 1024, with sample size 500, using ReLU as activation function. We vary noise variance 1e-1, 3e-1, 5e-1, 1e+0, 3e+0, 5e+0, and plot the clean test loss against noise variance.

For the figure of visualizing landscape, we use a two-layer wide neural network with hidden size 1024, with sample size 500, using ReLU as activation function and with noise variance 5 and train it until it interpolates into the noise. We using the 100x100 grid on $[0,2]^2$ to display landscape of ground truth $f$ and model output $\hat{f}$, also we display $\Delta f$ and $\Delta \hat{f}$, where red dots are the training set points.

% \subsection{Biharmonic Equation}
% We consider solving the fourth-order equation to demonstrate the necessity of using smoother activation function when encountering high-order equations. While \cite{song2024how} demonstrates this through optimization analysis, our theoretical result proves this via generalization analysis.

\paragraph{Verifying the Benign Overfitting Beyond Co-diagonalization Assumption}  We provide additional experiments on the PDE 
\[
-\nabla \cdot (|x| \nabla u) = f \quad \text{for } x \in \Omega \text{ and } u = 0 \text{ for } x \in \partial \Omega
\]
where the commutative assumption no longer holds. Our result demonstrates that it still verifies our two findings. Here we consider solving a solution \( u(x) = \sin(2\pi(1 - |x|)) \) defined on 
\(\Omega = \{ x : |x| < 1 \}\). $\hat{u}(x; \theta) = (1 - |x|)\ u_{\text{NN}}(x; \theta) $ to automatically satisfy the boundary condition, where \( u_{\text{NN}} \) is the neural network. We maintain the same configurations as previous experiments.

\begin{figure}
    \centering
    %\vspace{-0.2in}
    \includegraphics[width=0.5\linewidth]{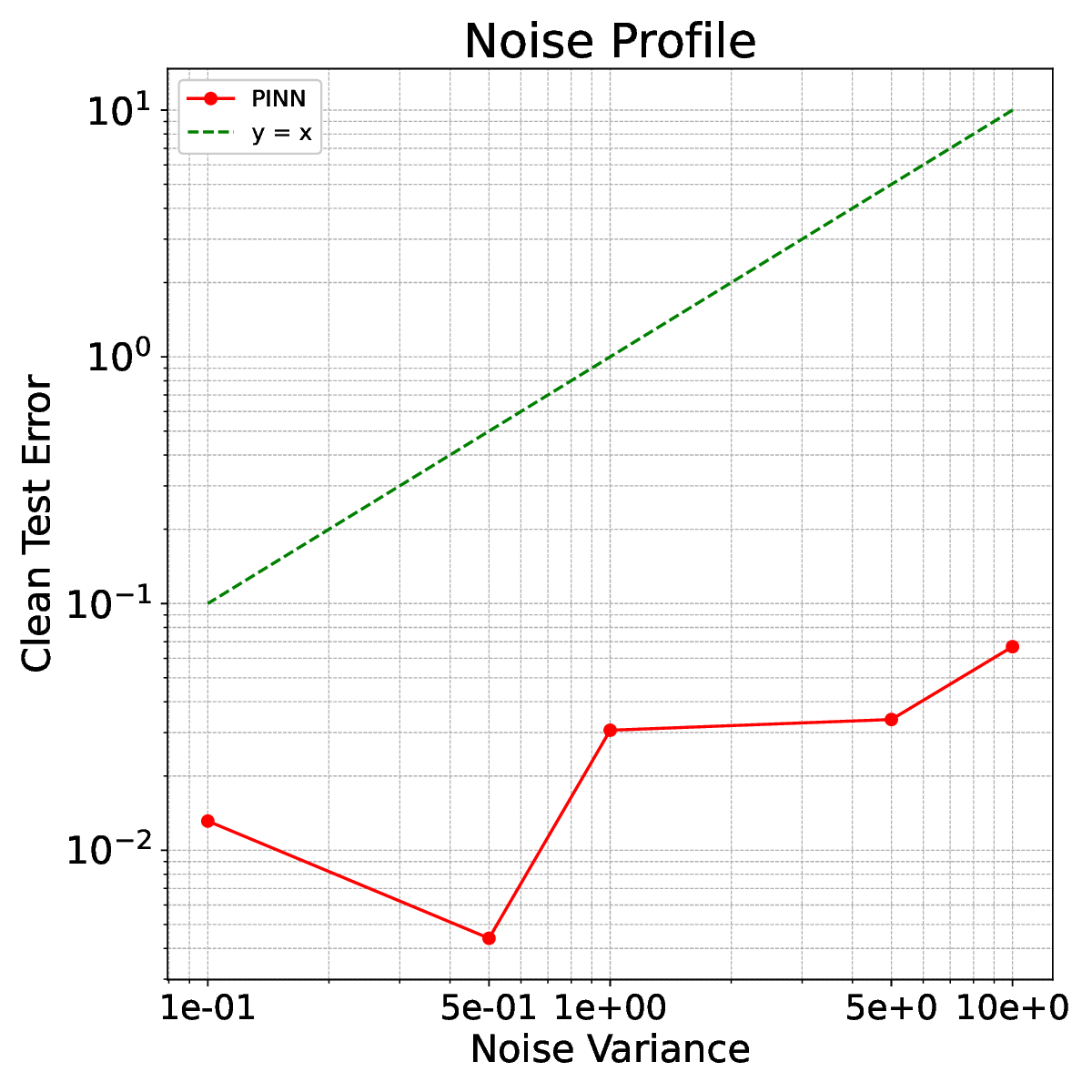}
    %\vspace{-0.1in}
    \caption{  We again verified our findings using PDE with solution of low regularity at the origin. The noise profile of Physics-informed interpolator exhibits benign overfitting, unlike the regression interpolator. }%\textbf{(Left)} We examine the impact of smooth inductive bias on convergence. Our findings demonstrate that when the activation function is sufficiently smooth, the inductive bias has a limited effect on improving convergence, which aligns with our theoretical predictions. \textbf{(Right)}   }
    %\vspace{-0.25in}\label{fig:low_regularity_figure}
\end{figure}

\end{document}